%% file: main.tex
\PassOptionsToPackage{table,x11names}{xcolor}

\documentclass[10pt]{article} 
\usepackage[accepted]{tmlr}

\input{math_commands.tex}

\usepackage{hyperref}
\usepackage{url}
\usepackage{microtype}
\usepackage{graphicx}
\usepackage{subfigure}
\usepackage{booktabs} 

\usepackage{amsmath}
\usepackage{amssymb}
\usepackage{mathtools}
\usepackage{amsthm}

\usepackage[capitalize,noabbrev]{cleveref}

\theoremstyle{plain}
\newtheorem{theorem}{Theorem}[section]
\newtheorem{proposition}[theorem]{Proposition}
\newtheorem{lemma}[theorem]{Lemma}

\theoremstyle{definition}

\theoremstyle{remark}

\usepackage[textsize=tiny]{todonotes}

\usepackage[utf8]{inputenc} 
\usepackage[T1]{fontenc}    
\usepackage{booktabs}       
\usepackage{amsfonts}       
\usepackage{nicefrac}       
\usepackage{microtype}      
\usepackage{graphicx}
\usepackage{amsmath}
\usepackage{amssymb}
\usepackage{amsthm}
\usepackage{lipsum}  
\usepackage{algorithm}
\usepackage{algorithmic}
\usepackage{xspace}

\usepackage{overpic}
\usepackage{subcaption}
\usepackage{enumitem}

\usepackage{multirow}
\usepackage{xcolor}
\usepackage{colortbl}

\newcommand{\RR}{\mathbb{R}}

\usepackage{rotating} 

\usepackage{hyperref}
\hypersetup{
    colorlinks,
    linkcolor={red!50!black},
    citecolor={magenta},
    urlcolor={blue!90!black}
}
\let\AND\relax

\title{Elucidating the Design Choice of Probability Paths 
in Flow Matching for Forecasting}

\author{\name Soon Hoe Lim \email shlim@kth.se \\
\addr Department of Mathematics, KTH Royal Institute of Technology \\
Nordita, KTH Royal Institute of Technology and Stockholm University  \\
\AND \\
\name  Yijin Wang \email yijin.wang@berkeley.edu \\
\addr International Computer Science Institute  \\
\AND \\
\name Annan Yu \email ay262@cornell.edu \\
\addr Center for Applied Mathematics, Cornell University \\
\AND \\
\name Emma Hart \email emma.hart@emory.edu \\
\addr Department of Mathematics, Emory University \\
\AND \\
\name Michael W. Mahoney \email mmahoney@stat.berkeley.edu  \\
\addr  Department of Statistics, University of California at Berkeley  \\
International Computer Science Institute \\
Lawrence Berkeley National Laboratory  \\
\AND  \\
\name Xiaoye S. Li \email xsli@lbl.gov \\
\addr Lawrence Berkeley National Laboratory  \\
\AND \\
\name N. Benjamin Erichson \email erichson@icsi.berkeley.edu \\
\addr International Computer Science Institute  \\ 
Lawrence Berkeley National Laboratory 
}
	
\def\month{07}  
\def\year{2025} 
\def\openreview{\url{https://openreview.net/forum?id=JApMDLwbLR}} 

\begin{document}

\maketitle

\begin{abstract}
Flow matching has recently emerged as a powerful paradigm for generative modeling and has been extended to probabilistic time series forecasting. However, the impact of the specific choice of probability path model on forecasting performance,  particularly for high-dimensional spatio-temporal dynamics,  remains under-explored. 
In this work, we demonstrate that forecasting spatio-temporal data with flow matching is highly sensitive to the selection of the probability path model. 
Motivated by this insight, we propose a novel probability path model designed to improve forecasting performance. 
Our empirical results across various dynamical system benchmarks show that our model achieves faster convergence during training and improved predictive performance compared to existing probability path models. 
Importantly, our approach is efficient during inference, requiring only a few sampling steps. 
This makes our proposed model practical for real-world applications and opens new avenues for probabilistic forecasting.
\end{abstract}

\section{Introduction}

Generative modeling has achieved remarkable success in recent years, especially for generating high-dimensional objects by learning mappings from simple, easily-sampled reference distributions, \(\pi_0\), to complex target distributions, \(\pi_1\). 
In particular, diffusion models have pushed the capabilities of generating realistic samples across various data modalities, including images~\citep{ho2020denoising, song2020score, karras2022}, videos~\citep{ho2022video, blattmann2023stable, gupta2023photorealistic}, and spatio-temporal scientific data like climate and weather patterns \citep{pathak2024kilometer, kohl2024benchmarking}. 
Despite their impressive performance, diffusion models often come with high computational costs during training and inference. 
Additionally, they typically assume a Gaussian reference distribution, which may not be optimal for all data types and can limit modeling flexibility.

One promising alternative is flow matching, where the mappings are learned via a stochastic process that transforms \(\pi_0\) into \(\pi_1\) through an ordinary differential equation (ODE), approximating its marginal vector flow \citep{lipman2022flow, albergo2023stochastic, liu2022flow, tong2023improving, pooladian2023multisample, lipman2024flow}. 
While score-based models \citep{song2019generative, song2020denoising, song2020score, ho2020denoising} are specific instances of flow matching, with Gaussian transition densities, the general framework allows for a broader class of interpolating  paths. 
This flexibility can lead to deterministic sampling schemes that are faster and require fewer steps \citep{zhang2022fast}.
Recent work has demonstrated the remarkable capabilities of flow matching models for generating high-dimensional images~\citep{esser2024scaling} and discrete data~\citep{gat2024discrete}.

Building on this, flow matching in latent space has recently been applied to forecasting spatio-temporal data~\citep{davtyan2023efficient} (predicting future frames in videos). 
This approach leverages latent representations to capture the complex dynamics inherent in temporal data. 
However, spatio-temporal forecasting, especially for videos and dynamical systems data, presents unique challenges. 
A video prediction model capable of generalizing to new, unseen scenarios must implicitly ``understand'' the scene: detecting and classifying objects, learning how they move and interact, estimating their 3D shapes and positions, and modeling the physical laws governing the environment \citep{battaglia2016interaction}. 
Similarly, accurate weather forecasting requires capturing intricate physical processes and interactions across multiple scales \citep{dueben2018challenges, schultz2021can}.

We observe that, in the context of spatio-temporal forecasting, the performance of flow matching is highly sensitive to the choice of the probability path model, an important topic which has not been widely explored within a unified framework. 
Different probability paths can significantly impact the accuracy and convergence of forecasting models, particularly when dealing with complex dynamical systems characterized by partial differential equations (PDEs) and chaotic behaviors.
%
Motivated by this, we propose a novel probability path model  tailored for probabilistic forecasting of dynamical systems, with a particular focus on PDE datasets relevant to scientific applications. 
Our model leverages the continuous dynamics intrinsic to spatio-temporal data by interpolating between consecutive sequential samples. This approach ensures better alignment with the constructed flow, leading to improved predictive performance, more stable training, and greater inference efficiency. Existing probability path models often fail to fully capture the continuous nature of spatio-temporal data, resulting in a misalignment with flow-based methods and suboptimal outcomes. Our proposed model addresses these limitations directly.

Building on previous approaches, we provide a theoretical framework and efficient algorithms tailored to probabilistic forecasting of high-dimensional spatio-temporal dynamics using flow matching in latent space. Within this framework, we demonstrate that our probability path model outperforms existing models across several forecasting tasks involving PDEs, achieving faster convergence during training and requiring fewer sampling steps during inference.
These advances enhance the practicality of flow matching approaches for real-world applications, particularly in scenarios where computational resources and time are critical constraints. 
Our main contributions are the following. 


\begin{itemize}
    \item 
    \textbf{Theoretical Framework and Efficient Algorithms:} 
    We present a theoretical framework and efficient algorithms for applying flow matching in latent space to the probabilistic forecasting of dynamical systems (see Algorithms \ref{alg:training}-\ref{alg:sampling}), extending the approach of \citep{lipman2022flow} and \citep{davtyan2023efficient}. 
    Our approach is specifically tailored for time series data, enabling effective modeling of complex temporal dependencies inherent in dynamical systems.     \vspace{-0.1cm}

    \item 
    \textbf{Novel Probability Path Model:} Using dynamical optimal transport and the Schr{\"o}dinger's bridge perspective (see Theorem \ref{thm_optimal_path}), 
    we motivate and propose a new probability path model (see Section \ref{subsect_newpath}) specifically designed for modeling dynamical systems data. 
    We provide theoretical insights to show that the variance of the vector field (VF) generating our proposed path can be lower than that of the optimal transport (OT) VF proposed by \citep{lipman2022flow} for sufficiently correlated spatio-temporal samples  (see Theorem \ref{thm_variance}).  We further provide intuitions to understand why our model leads to smoother training loss curve and faster convergence compared to other models.     \vspace{-0.1cm}

    \item  
    \textbf{Empirical Validation:} 
    We provide extensive empirical results to demonstrate that our proposed probability path model can outperform other  models on several forecasting tasks involving PDEs  (see Section \ref{sect_exp}).
    Our results demonstrate that the proposed probability path model outperforms existing flow matching models, achieving faster convergence during training and improved predictive performance across several evaluation metrics, while requiring fewer sampling steps. 
    \end{itemize}
    \vspace{-0.3cm}

\section{Flow Matching for Probabilistic Forecasting} 
\label{sect_framework}

In this section, we first introduce the objective of probabilistic forecasting, and then we discuss how flow matching can be used for learning  conditional distributions in latent space.

\textbf{Probabilistic forecasting framework.} 
Suppose that we are given a training set of $n$ trajectories, with each trajectory of length $m$,  $S_n = \{(\mathbf{x}^{1:m})^{(i)} \}_{i=1,\dots,n}$,  where $(\mathbf{x}^{1:m})^{(i)} = ((x^1)^{(i)}, \dots, (x^m)^{(i)})$,  (with the $(x^l)^{(i)} \in \mathbb{R}^d$), coming from an underlying continuous-time dynamical system.  
For simplicity, we denote the trajectories as $\mathbf{x}^{1:m} = (x^1, \dots, x^m)$ unless there is a need to specify the corresponding $n$. 
The trajectories are observed at arbitrary time points $\mathbf{t}_{1:m} = (t_1, \dots, t_m)$ such that $x^i := x(t_i) \in \mathbb{R}^d$ and $(x(t))_{t \in [t_1, t_m]}$ are the observed states of the ground truth system.  
In practice, we may have access to only few trajectories, i.e., $n$ is small or even $n=1$, and the trajectories themselves may be observed at different time stamps. 
%

The goal of probabilistic forecasting is to predict the distribution   of the upcoming $l$ elements given the first $k$ elements, where $m = l+k$, i.e.: $q(x^{k+1}, \dots, x^{k+l} | x^1, \dots, x^k) =  \prod_{i=1}^{l} q(x^{k+i} | x^1, \dots, x^{k+i-1})$.
We propose to model each one-step predictive conditional distribution via a probability density path. 
Instead of using score-based diffusion models to specify the path, we choose flow matching, a simpler method to train generative models. 
With flow matching, we directly work with probability paths, and we can simply avoid reasoning about (forward) diffusion processes altogether. Moreover, we shall work in a latent space.



\textbf{\bf Flow matching in latent space.}
Let $z^\tau = \mathcal{E}(x^\tau)$ for $\tau = 1,\dots, m$, where $\mathcal{E}$ denotes a pre-trained encoder that maps from the data space to a lower dimensional latent space. Working in the latent space, our goal is to approximate the ground truth distribution $ q(z^{\tau} | x^1, \dots, x^{\tau-1})$ by the parametric distribution $p(z^\tau | z^{\tau-1})$, which can then be decoded as $x^\tau = \mathcal{D}(z^\tau)$. The latent dynamics can be modeled by an ODE: 
\begin{equation}
    \dot{Z}_t = u_t(Z_t), 
\end{equation}
where $u_t$ is the (velocity) VF describing the instantenous rate of change of the state at time $t$. Learning the dynamics of the system is equivalent to approximating the VF $u_t$. 
This can be done by regressing a neural network using the mean squared error (MSE) loss.

Following the idea of flow matching, we infer the dynamics of the system generating ${\bf z}$ from the collection of latent observables by learning a time-dependent VF $v_t: [0,1] \times \mathbb{R}^d  \to \mathbb{R}^d$, $t \in [0,1]$, such that the ODE
\begin{equation} \label{eq_flow}
    \dot{\phi}_t(Z) = v_t(\phi_t(Z)), \ \ \phi_0(Z) = Z, 
\end{equation}
defines a time-dependent diffeomorphic map (called a flow), $\phi_t(Z): [0,1] \times \mathbb{R}^d \to \mathbb{R}^d$, that pushes a reference distribution $p_0(Z)$ towards the distribution $p_1(Z) \approx q(Z)$ along some probability density path $p_t(Z)$ and the corresponding VF  $u_t(Z)$. 
In other words, $p_t = [\phi_t]_* p_0$, where $[\cdot]_*$ denotes the push-forward operation. 
Here, $q$ is the ground truth  distribution, $p$ denotes a probability density path, i.e., $p: [0,1] \times \mathbb{R}^d \to \mathbb{R}_{>0}$, and $\int p_t(Z) dZ = 1$. 
We also write $Z_t = \phi(Z)$; and thus the ODE can be written as $\dot{Z}_t = v_t(Z_t), Z_0 = Z$. 
Typically the reference distribution $p_0$ is chosen to be the standard Gaussian \citep{lipman2022flow, liu2022flow}.

In other words, the main goal of flow matching is to learn a deterministic coupling between $p$ and $q$ by learning a VF $v_t$ such that the solution to the ODE (\ref{eq_flow}) satisfies $Z_0 \sim p$ and $Z_1 \sim q$. When ${\bf Z} = (Z_t)_{t \in [0,1]}$ solves Eq. (\ref{eq_flow}) for a given function $v_t$, we say that ${\bf Z}$ is a flow with the VF $v_t$. 
If we have such a VF, then $(Z_0, Z_1)$ is a coupling of $(p, q)$. 
If we can sample from $p$, then we can generate approximate samples from the coupling by sampling $Z_0 \sim p$ and numerically integrating Eq. (\ref{eq_flow}). This can be viewed as a continuous normalizing flow \citep{chen2018neural}.

If one were given a pre-defined probability path $p_t(Z)$ and the corresponding VF $u_t(Z)$ that generates the path, then one could parametrize $v_t(Z)$ with a neural network $v^\theta_t(Z)$, with $\theta$ the learnable parameter, and solve the least square regression by minimizing the flow matching  loss: 
\begin{equation} \label{original_opt}
    \min_{\theta} \mathcal{L}_{fm}(\theta) :=  \mathbb{E}_{t, p_t(Z)} \ 
\omega(t) \| v^\theta_t(Z) - u_t(Z) \|^2, 
\end{equation}
where $t \in \mathcal{U}[0,1]$, $Z \sim p_t(Z)$ and $\omega(t) > 0$ is a weighting function. The weighting function weights the importance of the $L^2$ loss at different times $t$ (noise level),  balancing the importance of high frequency and low frequency components. Following the standard framework \citep{lipman2022flow}, we take $\omega(t) = 1$. Note that we choose to include $\omega(t)$ in Eq. (\ref{original_opt}) for generality as different choice of weighting function corresponds to different parametrization of the network output (see \citep{anonymous2025diffusionmodelsand} for more details).

However, we do not have prior knowledge for choosing $p_t$ and $u_t$, and there are many choices that can satisfy $p_1 \approx q$. Moreover, we do not have access to a closed form $u_t$ that generates the desired $p_t$. We shall follow the approach of \citep{lipman2022flow} and construct a target probability path by mixing simpler conditional probability paths.  This probability path is the marginal probability path:
\begin{equation} \label{eq_marginalpath}
    p_t(Z) = \int p_t(Z | \tilde{Z}) q(\tilde{Z}) d\tilde{Z},
\end{equation}
obtained by marginalizing the conditional probability density paths $p_t(Z|\tilde{Z})$ over observed latent trajectories $\tilde{Z}$, with $p_0(Z|\tilde{Z}) = p(Z)$ and $p_1(Z|\tilde{Z}) = \mathcal{N}(Z | \tilde{Z}, \epsilon^2 I)$ for a small $\epsilon > 0$. Doing so gives us a marginal probability $p_1$ which is a mixture distribution that closely approximates $q$. Then, assuming that $p_t(Z) > 0$ for all $Z$ and $t$, we can also define a marginal VF as: 
\begin{equation} \label{eq_marginalvf}
    u_t(Z) = \int u_t(Z| \tilde{Z}) \frac{p_t(Z|\tilde{Z}) q(\tilde{Z})}{p_t(Z)} d\tilde{Z},  
\end{equation}
where $u_t(Z|\tilde{Z})$ is a conditional VF (conditioned on the latent trajectory $\tilde{Z}$). 
It turns out that this way of mixing the conditional VFs leads to the correct VF for generating the marginal probability path  (\ref{eq_marginalpath}). We can then break down the intractable marginal VF into simpler conditional VFs which depends on a single sample.

To deal with the intractable integrals in Eq. (\ref{eq_marginalpath})-(\ref{eq_marginalvf}) which complicates computation of an unbiased estimator of $\mathcal{L}_{fm}$, we shall minimize the conditional loss proposed by \citep{lipman2022flow}:
\begin{equation} \label{conditional_loss}
    \min_{\theta} \mathcal{L}_{cfm}(\theta) :=  \mathbb{E}_{t, p_t(Z|\tilde{Z}), q(\tilde{Z}) } \ \omega(t) \| v^\theta_t(Z) - u_t(Z | \tilde{Z}) \|^2,
\end{equation}
where $t \in \mathcal{U}[0,1]$, $\tilde{Z} \sim q(\tilde{Z})$, $Z \sim p_t(Z|\tilde{Z})$, and $u_t(Z| \tilde{Z})$ is the VF defined {\it per sample} $\tilde{Z}$ that generates the conditional probability path $p_t(Z|\tilde{Z})$. Importantly, one can show that the solution of  (\ref{conditional_loss}) is guaranteed to converge to the same result in  (\ref{original_opt}); see Theorem \ref{thm_1} in App. \ref{app_proof}. Therefore, the conditional flow matching loss can match the pre-defined target probability path, constructing the flow that pushes $p_0$ towards $p_1$. Since both the probability path and VF are defined per sample, we can sample unbiased estimates of the conditional loss efficiently, particularly so with suitable choices of conditional probability paths and VFs.


\section{Probability Path Models for Probabilistic Forecasting}   \label{sect_model}

In this section, we describe the family of probability paths that we consider for flow matching, and we propose an improved model tailored for probabilistic forecasting of spatio-temporal data. 

\subsection{Common Probability Path Models} 

The family of Gaussian conditional probability paths gives us tractable choices to work with since  the relevant quantities in Eq. (\ref{conditional_loss}) and thus the conditional flow can be defined explicitly. Therefore, we will work with Gaussian probability paths. Moreover, we  are going to solve  (\ref{conditional_loss}) over the dataset of all transition pairs $\mathcal{D}_{pair} = \{(z^{\tau-1}, z^{\tau})\}_{\tau=2,\dots,m}$, and use a pair of points for $\tilde{Z}$,  setting $\tilde{Z} = (Z_0, Z_1) \in \mathcal{D}_{pair}$. 
In particular, we consider the following class of models for the probability paths: 
\begin{equation} \label{eq_generalclass}
p_t(Z| \tilde{Z} := (Z_0, Z_1)) = \mathcal{N}(Z | a_t Z_0 + b_t Z_1, c^2_t I),
\end{equation}
where $a_t$, $b_t$ and $c_t$ are differentiable time-dependent functions on $[0,1]$, and $I$ denotes the identity. 

Table \ref{tab_choices} provides five different choices of probability paths, including our proposed choice,  that we study here. 
The optimal transport (OT) VF model  was initially proposed by \citep{lipman2022flow}, and setting $\epsilon_{min}= 0$  gives us the rectified flow model  of \citep{liu2022flow}, which proposed connecting data and noise on a straight line. 
The stochastic interpolant (SI) model  in Table \ref{tab_choices}  is the one considered by \citep{chen2024probabilistic}. 
The VE- and VP-diffusion conditional VFs (derived with Theorem \ref{thm_gaussianpaths})  coincide with the VFs governing the Probability Flow ODE for the VE and VP diffusion paths proposed in \citep{song2020score}.  
It has been shown that combining diffusion conditional VFs  with the flow matching objective leads to a training alternative that is more stable and robust when compared to existing score matching approaches \citep{lipman2022flow}.

\begin{table}[!t]
	\caption{Choices of probability density paths that we study in this paper. For VE-diffusion, $\sigma_t$ is increasing in $t$, $\sigma_0 = 0$. For VP-diffusion, $\beta =$ noise scale.}
	\label{tab_choices}
	\centering
	\scalebox{0.85}{
		\begin{tabular}{l c c c}
			\toprule
			Model &  $a_t$ &  $b_t$ & $c_t^2$ \\	
			\midrule 
			VE-diff.  \citep{song2020score, lipman2022flow}  & 1 & 0 & $\sigma_{1-t}^2$ \\
			VP-diff. \citep{song2020score, lipman2022flow}  & $e^{-\frac{1}{2} T(1-t) }$ & 0 & $ 1- e^{-T(1-t)}, T(t) = \int_0^t \beta(s) ds$ \\
            OT-VF \citep{liu2022flow}  &  $t$ & 0 & $(1-(1-\epsilon_{min}) t)^2 $, $\epsilon_{min} \geq 0$\\
            SI \citep{chen2024probabilistic} & $1-t$ & $t$ or $t^2$ & $\epsilon^2 t(1-t)^2$, $\epsilon > 0$  \\
                \textbf{Ours}          &  $1-t$ & $t$ & $\sigma_{min}^2 +\sigma^2 t(1-t)$, $\sigma_{min}, \sigma \geq 0$   \\
			\bottomrule \\
	\end{tabular}}
\end{table} 

As remarked in \citep{lipman2022flow}, there are many choices of VFs that generate a given probability path. We shall use the simplest VF that generates flow whose map is affine linear.  Let $p_t(Z|\tilde{Z})$ be the Gaussian probability path defined in Eq. (\ref{eq_generalclass}) and consider the flow map $\psi_t$ defined as $\psi_t(Z) := a_t Z_0 + b_t Z_1 + c_t Z$ with $c_t >0$. Then the unique VF that defines $\psi_t$ is (see Theorem \ref{thm_gaussianpaths} and the proof in App. \ref{app_proof}):
    \begin{equation} \label{gen_vf_formula}
    u_t(Z|\tilde{Z}) = \frac{c_t'}{c_t} (Z- (a_t Z_0 + b_t Z_1)) + a_t' Z_0 + b_t' Z_1,
    \end{equation}
    where prime denotes derivative with respect to $t$, and $u_t(Z|\tilde{Z})$ generates the Gaussian path $p_t(Z|\tilde{Z})$.


In view of this, minimizing the conditional loss becomes:
\begin{align} \label{conditional_opt}
    &\min_{\theta} \mathcal{L}_{cfm}(\theta) :=  \mathbb{E}_{t, z^\tau, z^{\tau-1}, Z } \ \omega(t) \bigg\| v^\theta_t(Z) -   
 \frac{c_t'}{c_t} (Z- (a_t z^\tau + b_t z^{\tau-1} )) - a_t' z^\tau - b_t' z^{\tau-1} \bigg\|^2,
\end{align}
where $t \sim \mathcal{U}[0,1]$, $Z \sim p_t(Z| z^\tau, z^{\tau-1})$ and ${\bf z} \sim q({\bf z})$.  
We refer to this as the Flow Matching loss parametrization and shall work with this parametrization. There are other parametrizations: most popular ones are the Score Matching loss, Score Flow loss and DDPM loss. See App. \ref{app_lossparam} for a comparison of different loss parametrizations and App. \ref{app_connect} for connections to SDE based generative models.

In practice, it may be beneficial to learn the  vector field $v_t^\theta$  using additional context information, in which case  (\ref{conditional_opt}) becomes:
\begin{align} \label{conditional_opt_practice}
    &\min_{\theta} \mathcal{L}_{cfm}(\theta) :=  \mathbb{E}_{t, z^\tau, z^{\tau-1}, Z, C } \  \omega(t) \bigg\| v^\theta_t(Z | C) -   
 \frac{c_t'}{c_t} (Z- (a_t z^\tau + b_t z^{\tau-1} )) - a_t' z^\tau - b_t' z^{\tau-1} \bigg\|^2,
\end{align}
where $C$ represents the context information. The choice of $C$ is task-dependent. For spatio-temporal tasks, we adopt the sparse conditioning scheme of \cite{davtyan2023efficient}, choosing $C$ to be past  frames as follows. Given a sampled frame $z^\tau$, we sample another index $c$ uniformly from $\{1, \cdots , \tau - 2\}$ and use $z^c$, which we call context frame, together with the previous frame $z^{\tau-1}$, as the two conditioning frames. Thus, in this case $C = (z^{\tau-1}, z^{c}, \tau - c)$ (see Algorithm \ref{alg:training}). While conditioning on as many past frames as possible is desirable to improve prediction accuracy, sparse conditioning is sufficient to achieve favorable trade-off between computational efficiency and accuracy \citep{davtyan2023efficient}. Datasets with more challenging dynamics may require the use of more context frames.

\subsection{Optimal Probability Paths}
We consider the problem of selecting the optimal probability paths within  the class of the Gaussian probability paths. We shall make use of the Schr{\"o}dinger bridge \citep{leonard2013survey, chen2021stochastic} perspective, and seek to find the {\it optimal} stochastic processes that evolve a given measure into another, subject to marginal constraints and based on a {\it prior belief}. This optimal process describes a novel probability path model that we propose for probabilistic forecasting of spatio-temporal data. 

To be more precise, let $\nu, \nu'$ be two given probability measures and let $\mathbb{Q}$ be the path measure of an arbitrary stochastic process. The Schr{\"o}dinger bridge (SB) is the solution to the following constrained minimization problem over all path measures $\mathbb{P}$ (which are absolutely continuous with respect to $
\mathbb{Q}$) of stochastic processes  on the finite time interval $[0,1]$:
\begin{equation}
\min_{\mathbb{P}_0 = \nu, \  \mathbb{P}_1 = \nu'} D_{KL}(\mathbb{P} \| \mathbb{Q}),    
\end{equation}
where $D_{KL}$ denotes the Kullback–Leibler divergence and $\mathbb{P}_t$ denotes the time marginal of $\mathbb{P}$ at time $t$.
Typically $\nu$ and $\nu'$ are the (empirical) marginal distributions of an unknown continuous-time dynamics observed at the initial and terminal times, and $\mathbb{Q}$ is the path measure of a prior (or reference) process that represents our belief of the dynamics before observing the data. The solution $\mathbb{P}^*$ to the problem can then be interpreted as the optimal dynamics that conforms to the prior belief $\mathbb{Q}$ while respecting the data marginals $\mathbb{P}_0^* = 
\nu$, $\mathbb{P}_1^* = \nu'$. In other words, the SB is the (path measure of the) finite-time diffusion which admits as initial
and terminal distributions the two distributions of interest and is the closest in KL
divergence to (the path measure of) a reference diffusion. Recent work \citep{shi2024diffusion, pooladian2024plug} focuses on developing improved algorithms to solve SB problems for general classes of reference diffusions. SBs have also been used to formulate generative models by interpolating distributions on a finite time interval \citep{wang2021deep, de2021diffusion, chen2021likelihood, peluchetti2023diffusion, tong2023simulation, liu2023generalized, gottwald2024localized}.

We will consider Gaussian prior processes for $\mathbb{Q}$.
The Gaussian probability paths that we consider can be formulated via the differential equation: $dZ_t = (\dot{a}_t Z_0 + \dot{b}_t Z_1 + \dot{c}_t \xi) dt$, where $
\xi \sim \mathcal{N}(0,I)$. Note that $
\mathbb{E}[\dot{Z}_t | Z_0, Z_1] = \dot{a}_t Z_0 + \dot{b}_t Z_1 =: \alpha_t$. We shall take $\mathbb{Q}$ to be the path measure of the linear SDE with  drift $\alpha_t$:
\begin{equation}
    dY_t = \alpha_t  dt + \omega dW_t, \ t \in [0,1],
\end{equation}
where $\omega > 0$ is a regularity parameter and $W_t$ is the standard Wiener process. While many choices for $\mathbb{Q}$ exists, the rationale behind our choice is that the Gaussian process $Y_t$ admits a minimal representation $Y_t = \mathbb{E}[Z_t | Z_0, Z_1] + \omega W_t$ that incorporates Brownian motion as a reference noise to characterize the stochasticity surrounding the conditional estimation of $Z_t$. Importantly, such choice allows interpretation of the SBs  as generalized dynamical optimal transport (DOT) \citep{chewi2024statistical} between two (not necessarily Gaussian) measures. In particular, adapting Theorem 1 in \citep{bunne2023Schrodinger} to our setting, we have:

\begin{proposition}
    Consider the SB problem with $Y_t$ as the reference process:
    \begin{equation} \label{gen_SB}
        \min_{\mathbb{P}_0 = \nu, \  \mathbb{P}_1 = \nu'} D_{KL}(\mathbb{P} \| \mathbb{Y}),     
    \end{equation}
    where $\mathbb{Y}$ is the path measure induced by $(Y_t)_{t \in [0,1]}$.
    Then, (\ref{gen_SB}) is equivalent to $\inf_{(\rho_t, v_t)} \mathbb{E} \left[ \int_0^1  C(\rho_t, v_t) \ dt \right]$, with
    \begin{equation} \label{GDOT}
        C(\rho_t, v_t) := \frac{\|v_t\|^2}{2 \omega^2} + \frac{\omega^2}{8} \|\nabla \log \rho_t \|^2 - \frac{1}{2} \langle \alpha_t, \nabla \log \rho_t \rangle,
    \end{equation}
    where the infimum is taken over all pairs $(\rho_t, v_t)$ such that $\rho_0 = \nu$, $\rho_1 = \nu'$, $\rho_t$ absolutely continuous, and satisfies  $\partial_t \rho_t = - \nabla \cdot ( (v_t + \alpha_t) \rho_t)$.
\end{proposition}

The DOT problem (\ref{GDOT}) can be seen as a generalized version of DOT with quadratic cost (minimizing a "kinetic energy"); see  Eq. (4.41a)-(4.41c) in \citep{lipman2024flow}. The GDOT (\ref{GDOT}) not only minimizes the quadratic cost, but also minimizes the Fisher information and maximizes the expected alignment of the score function with the drift $\alpha_t$.

In our case, given a pair $(Z_0, Z_1)$ of data points, we are interested in constructing   SBs where the marginal constraints are Gaussians centered around $Z_0$ and $Z_1$, which we denote as $\xi_0 := \mathcal{N}(Z_0, 
\sigma_{min}^2 I)$ and $\xi_1 := \mathcal{N}(Z_1, 
\sigma_{min}^2 I)$ respectively, for some given $\sigma_{min} \geq 0$. These are the Gaussian SBs: 
\begin{equation} \label{Gaussian_SB}
        \min_{\mathbb{P}_0 = \xi_0, \  \mathbb{P}_1 = \xi_1} D_{KL}(\mathbb{P} \| \mathbb{Y}).   
\end{equation}
It turns out that the solution of these Gaussian SBs admits a closed-form expression \citep{bunne2023Schrodinger}.

\begin{theorem} \label{thm_optimal_path}
    The solution $\mathbb{P}^*$ to the Gaussian SB (\ref{Gaussian_SB}) is the path measure of a Markov Gaussian process with the marginal variable $X_t \sim \mathcal{N}(\mu_t, \sigma^2_t I)$, where 
    \begin{align}
        \mu_t &= (1-t + a_t-a_0 - t (a_1 - a_0)) Z_0 + (t  + b_t - b_0  - t(b_1 - b_0)) Z_1,    
    \\
    \sigma_t^2 &= \sigma_{min}^2 + \sigma^2 t(1-t), 
    \end{align}
    with $\sigma^2 = \sqrt{4 \sigma_{min}^4 + \omega^4} - 2 \sigma_{min}^2 > 0$.
\end{theorem}

See App. \ref{app_GSB} for the proof. In particular, imposing the boundary constraints $a_0 = b_1 =  1, a_1 = b_0 = 0$ on the functions $a_t$ and $b_t$, we have $\mu_t = a_t  Z_0 +  b_t Z_1$ in Theorem \ref{thm_optimal_path}. Note that $\sigma_t^2$ is independent of $a_t$ and $b_t$.

\subsection{A Novel Probability Path Model}  \label{subsect_newpath}

Our proposed probability path model is a stochastic extension of the straight-line trajectory connecting consecutive latent samples $(Z_0, Z_1)$, i.e., $\mu_t = (1-t)Z_0 + tZ_1$. 
This choice of $a_t$ and $b_t$ can be motivated by the following optimality principle. Given a pair of data points $(Z_0, Z_1) \in \RR^d \times \RR^d$, the linear interpolation path $\mu_t$ arises naturally as the solution to the following variational problem: 
\begin{equation}
    \min_{\mu: [0,1] \to \RR^d} \int_0^1 \|\dot{\mu}_t\|^2 dt \ \text{ subject to } \mu_0 = Z_0, \mu_1 = Z_1.
\end{equation}
This is a classical energy minimization problem, where the goal is to find the smoothest  curve connecting the endpoints $Z_0$ and $Z_1$. It can be solved using the Euler-Lagrange equations under the given  boundary conditions. This formulation corresponds to minimizing the kinetic energy of the curve among all smooth interpolations between $Z_0$ and $Z_1$. The solution, which is constant-velocity motion along the straight line, has a natural physical interpretation: it is the least effort way to move from $Z_0$ to $Z_1$ in Euclidean space.

More generally, when the boundary conditions are probability distributions (e.g., Gaussians centered at $Z_0$ and $Z_1$), this linear interpolation  corresponds to the displacement interpolation in the 2-Wasserstein space. In the case where the marginals are Gaussian and the cost function is quadratic, the interpolating path of distributions is also Gaussian with linearly interpolated means, minimizing an upper bound on the expected kinetic energy among all stochastic paths with prescribed marginals at $t=0$ and $t=1$; see Section 4.7 in \citep{lipman2024flow} for more details. Alternative interpolants are possible, but the linear form admits analytical tractability and natural interpretation in terms of velocity fields induced by diffusion processes. Importantly, this leads to trajectories that are generally easier to sample with ODE solvers, leading to ODE simulations with smaller errors \citep{liu2022flow}. 

For the stochastic extension, we build on the deterministic backbone and relax the boundary constraints to Gaussians centered around $Z_0$ and $Z_1$. We shall look for the "most likely" probability paths connecting the boundary distributions as follows. First, we consider the noise perturbed paths $Z_t = \mu_t + \omega W_t$, where $W_t$ is the standard Wiener process, as our reference process.  Then, we take the solution to the Schr{\"o}dinger bridge (\ref{Gaussian_SB}) with respect to the reference process as our proposed probability path model.

More precisely, in view of Theorem \ref{thm_optimal_path}, we propose to choose  $a_t = 1-t$, $b_t = t$, $c^2_t = \sigma_{min}^2 + \sigma^2 t(1-t)$,  in which case we have the probability path described by: 
\begin{equation} \label{eq_pttopt}
p_t(Z| \tilde{Z}) = \mathcal{N}(Z | t Z_1 + (1-t) Z_0, (\sigma_{min}^2 + \sigma^2 t(1-t) ) I),
\end{equation}
which transports a Gaussian distribution centered around $Z_0$ with variance $\sigma_{min}^2$ at $t=0$ to a Gaussian distribution  centered around $Z_1$ at $t=1$ with variance $\sigma_{min}^2$.  
Here $ \sigma_{min}, \sigma \geq  0 
$ are tunable parameters. 
In the case when $\sigma_{min} = 0$, it describes a Brownian bridge that interpolates between $Z_0$ and $Z_1$ \citep{gasbarra2007gaussian}. To ensure numerical stability when sampling $t \sim \mathcal{U}[0,1]$, it is beneficial to use a small $\sigma_{min} > 0$.
Note that $\sigma^2$ is a scale factor determining the magnitude of fluctuations around the path interpolating between $Z_0$ and $Z_1$. The variance $c_t^2$ is minimum with the value of $\sigma_{min}^2$ at the endpoints $t=0$ and $t=1$, and the maximum variance is $\sigma_{min}^2 + \sigma^2/4$ which occurs in the middle of the path at $t=1/2$.

The variance schedule is designed to balance exploration and stability.  Low variance at the start ensures stable initialization, preventing the trajectory from deviating too far from the initial distribution.
High variance in the middle allows the model to explore diverse paths in the latent space, avoiding mode collapse and enhancing diversity in the generated trajectories.  Low variance at the end sharpens the trajectory, ensuring accurate reconstruction of the desired output. This strategy is inspired by findings in diffusion models that utilize a forward noising process and a backward denoising process \citep{ho2020denoising, song2020score}, where such variance patterns have been shown to effectively manage the trade-off between exploration and refinement. 

The corresponding VF that defines the flow is then given by (applying Theorem \ref{thm_gaussianpaths}):
\begin{equation} \label{eq_targetVF}
u_t(Z|\tilde{Z}) = Z_1 - Z_0 + \frac{\sigma^2}{2} \frac{1-2t}{\sigma_{min}^2 + \sigma^2 t(1-t ) } \epsilon,
\end{equation}
where $\epsilon := Z - (t Z_1 + (1-t) Z_0)$.
Under reasonable assumptions, the variance of this VF is lower than the variance of  the OT-VF of \citep{lipman2022flow} with $\epsilon_{min} := 0$ (rectified flow). Here the variance is taken with respect to the randomness in the samples $z^\tau$ and the Gaussian samples drawn during  gradient descent updates.

More precisely, denote the VF that corresponds to our model as $u_t$ and the VF that corresponds to the rectified flow model as $\tilde{u}_t$. For a given $\tau$, they generate the probability path $Z_t = t z^\tau  + (1-t) z^{\tau-1} + c_t \xi$ and $\tilde{Z}_t = t z^{\tau-1} + (1-t) \eta$ respectively, where $\xi, \eta \sim \mathcal{N}(0,I)$, $c_t = \sqrt{\sigma_{min}^2 + \sigma^2 t(1-t)}$ and $t \in [0,1]$.
Applying Eq. (\ref{gen_vf_formula}), we have:
\begin{align}
    u_t(Z_t|z^{\tau-1}, z^\tau) &=  z^\tau - z^{\tau-1} + c_t' \xi,  \\ 
     \tilde{u}_t(\tilde{Z}_t|z^{\tau-1}) &= z^{\tau-1} - \eta.
\end{align}

\begin{theorem} \label{thm_variance}
    Suppose that $(z^\tau)_{\tau=1,\dots,m}$, with the $z^\tau \in \RR^d$, is a discrete-time stochastic process with nonzero correlation in time and let $t \in [0,1]$ be given.
    If $Cov(z^{\tau-1}, z^\tau) \geq \frac{1}{2}\left( \left( \frac{\sigma^4}{ 4 \sigma_{min}^2} - 1 \right) I + Var(z^\tau)  \right)$, then 
    $Var(\tilde{u}_t(\tilde{Z}_t|z^{\tau-1})) \geq Var(u_t(Z_t|z^{\tau-1}, z^\tau)).$
\end{theorem}
See App. \ref{app_explain} for a proof and related discussions. Theorem \ref{thm_variance} implies that if the consecutive latent variables $z^\tau$, $z^{\tau-1}$ are sufficiently correlated and $\sigma$ is chosen small enough relative to $\sigma_{min}$, then the variance of the VF that corresponds to our probability path model is lower than that corresponds to the rectified flow model.



\section{An Efficient Probabilistic Forecasting Algorithm}

In this section, we present efficient algorithms for training and inferencing the flow matching model. 

Recently \citep{davtyan2023efficient} proposed an efficient algorithm for latent flow matching for the task of video prediction, using the probability path generated by the OT-VF of \citep{lipman2022flow}.  To enable efficient training, we shall follow \citep{davtyan2023efficient} and leverage the
iterative nature of sampling from the learned flow and use a single random conditioning element from the past at each iteration.  However, our method differs from \citep{davtyan2023efficient} as we use different probability paths and target VFs.  

\noindent {\bf Training.} 
Recall that $\tilde{Z} = (Z_0, Z_1)$ denote training samples from $\mathcal{D}_{pair} = \{(z^{\tau-1}, z^{\tau})\}_{\tau=2,\dots,m}$. In other words, 
we set $Z_1$ to be the target element and $Z_0$ to be the reference element chosen to be the previous element before the target element. In this way, our target probability path model maps a distribution centered around a previous state to the distribution of the current states, which is more natural from the point of view of probabilistic forecasting whose goal is to obtain an ensemble of forecasts. Note that this differs from the approach of \citet{davtyan2023efficient}, where $\tilde{Z} = Z_1$ (i.e., they do not use a reference element to define their probability path, whereas we use a pair of elements $(Z_0, Z_1)$). 
Algorithm \ref{alg:training} summarizes the training procedure of our method. 


Both the autoencoder and the VF neural network  can also be jointly trained in an end-to-end manner, but our empirical results show that separating the training can lead to improved performance. Moreover, doing so allows us to better assess the impact of using different probability paths. 


\begin{small}
\begin{algorithm}[!h]
    \caption{Flow matching for spatio-temporal data}\label{alg:training}
    \begin{algorithmic}
        \STATE {\bf Input:} Dataset of sequences $D$, number of iterations $M$
        \FOR{$i$ in range(1, $M$)}
        \STATE Sample a sequence $\mathbf{x}$ from the dataset $D$
        \STATE Encode it with a pre-trained encoder to obtain ${\bf z}$
        \STATE Choose a random target element $z^{\tau}, \tau \in \{3, \dots |\mathbf{z}|\}$, from ${\bf z}$
        \STATE Sample a step $t \sim U[0, 1]$
        \STATE Sample a noisy observation $Z \sim p_t(Z \,|\, z^{\tau}, z^{\tau-1})$, where $p_t$ is given by  Eq. (\ref{eq_generalclass})
        \STATE Compute $u_t(Z \,|\, z^{\tau}, z^{\tau-1})$ according to (\ref{eq_targetVF})
        \STATE Sample a condition frame $z^{c}, c \in \{1, \dots \tau - 2\}$
        \STATE Update the parameters $\theta$  via gradient descent 
        \begin{align*}
            \nabla_{\theta}  \| v^\theta_{t}(Z \,|\, z^{\tau-1}, z^{c}, \tau - c ) - u_t(Z \,|\, z^{\tau}, z^{\tau-1}) \|^2
        \end{align*}
        \ENDFOR
        \STATE {\bf Return:} A learned VF, $v_t^{\theta^*}$
    \end{algorithmic}
\end{algorithm}
\end{small}

\vspace{0.3cm}
\noindent {\bf Inference.}
We use an ODE sampler during inference to generate forecasts. 
The ODE sampler is described as follows.
Let $(Y_i^{\tau})_{i=1,\dots, N-1}$ denote  the generation process, where $N-1$ is the number of integration steps and the superscript $\tau$ denotes the time index for which the generation/forecast is intended for.  Given the previous elements $(x^1, \dots, x^{T-1})$ of a time series sample, in order to generate the next element (i.e., the $T$-th element), we start with sampling the initial condition $Y_0^T$ from $\mathcal{N}(z^{T-1}, \sigma_{sam}^2 I)$ for some small $\sigma_{sam} \geq 0$, where $z^{T-1} = \mathcal{E}(x^{T-1})$. This is in contrast to the procedure of \citep{davtyan2023efficient}, which simply uses a mean-zero Gaussian sample instead.  We then use an ODE solver to integrate the learned VF along the time interval $[0, 1]$ to obtain $Y_{N-1}^T$. We use $\mathcal{D}(Y_{N-1}^T)$ as an estimate of $x^{T}$, and forecasting is done autoregressively.  

Algorithm \ref{alg:sampling} summarizes this procedure when the sampling is done using the forward Euler scheme. Note that we can also use computationally more expensive numerical schemes such as the Runge-Kutta (RK) schemes. 

\begin{algorithm}[!h]
    \caption{One-step ahead forecasting with forward Euler}\label{alg:sampling}
    \begin{algorithmic}
        \STATE {\bf Input:}  A sequence $(x^1, \dots, x^{T-1})$ containing the previous elements, number of integration steps $N-1$, grid $s_0 = 0 < s_1 < \dots < s_{N-1} = 1$, a learnt VF $v_{s}^{\theta^*}$ for $s \in [0,1]$
        \STATE Set $\Delta s_n = s_{n+1} - s_n$ for $n=0, \dots, N-2$ 
        \STATE Sample $Y^T_0 \sim \mathcal{N}(\mathcal{E}(x^{T-1}),  \sigma_{sam}^2 I)$, $\sigma_{sam} \geq 0$ 
        \FOR{$n$ in range(0, $N-1$)}
        \STATE Sample $c \sim \mathcal{U}(2, \dots, T-1)$ 
        \STATE $y^{T-c} = \mathcal{E}(x^{T-c})$
        \STATE $Y_{n+1}^T = Y_n^T + \Delta s_n v_{s_n}^{\theta^*}(Y_n^T | Y_0^T, y^{T-c}, T-c) $
        \ENDFOR
        \STATE {\bf Return:} An estimate of $x^T$,  $\hat{x}^T = \mathcal{D}(Y_{N-1}^T)$ 
    \end{algorithmic}
\end{algorithm}

\section{Empirical Results} \label{sect_exp}

In this section, we present our main empirical results to elucidate the design choice of probability paths within the flow matching framework (comparison with other generative modeling frameworks is not our focus here).
We focus on PDE dynamics forecasting tasks here (additional details and results can be found in App. \ref{app_additional}-\ref{app_exp}). 
We test the performance of our probability path model, i.e., Eq. (\ref{eq_pttopt}) with $a_t = 1-t$, $b_t = t$ and $c_t = \sqrt{\sigma_{min}^2 + \sigma^2 t(1-t) }$ on these tasks. We pick $\sigma_{min} = 0.001$, and treat $\sigma$ and $\sigma_{sam}$ as tunable parameters.
We  compare our proposed model with four other models of probability paths:

\begin{itemize}[leftmargin=*]
\vspace{-0.2cm}
\item \textbf{RIVER}~\citep{davtyan2023efficient}: RIVER uses the OT-VF model in Table \ref{tab_choices}, i.e., $a_t = 0$, $b_t = t$, $c_t = 1-(1-\epsilon_{min}) t$, choosing  $\epsilon_{min} =  10^{-7}$.
\vspace{-0.2cm}
\item \textbf{VE-diffusion} in Table \ref{tab_choices}: We use $\sigma_t = \sigma_{min} \sqrt{ \left(\frac{\sigma_{max}}{\sigma_{min}} \right)^{2t} - 1}$ with $\sigma_{min} = 0.01$, $\sigma_{max}=0.1$, and sample $t$ uniformly from $[0, 1-\epsilon]$ with $\epsilon = 10^{-5}$, following \citep{song2020score}.
\vspace{-0.2cm}
\item \textbf{VP-diffusion} in Table \ref{tab_choices}: We use $\beta(s) = \beta_{min} + s (\beta_{max} - \beta_{min})$ where $\beta_{min}=0.1, \beta_{max}=20$ and $t$ is sampled from $\mathcal{U}[0,1-\epsilon]$ with $\epsilon=10^{-5}$, following \citep{song2020score}. Thus, $T(s) = s \beta_{min} + \frac{1}{2} s^2 (\beta_{max}-\beta_{min})$.
\vspace{-0.2cm}
\item  \textbf{The stochastic interpolant (SI) path} in Table 1: We consider the path proposed by \citep{chen2024probabilistic} and use the suggested choice of $a_t = 1-t$, $b_t = t^2$ and $c_t = \epsilon (1-t) \sqrt{t}$ (see Eq. (2) in \citep{chen2024probabilistic} and note that $Var ((1-t) W_t)= (1-t)^2 t$ for the standard Wiener process $W_t$). We also consider the choice with $b_t = t$ instead.
We choose $\epsilon = 0.01$ for both choices. This is a path that is similar to ours, but with the variance $c_t^2$ chosen such that the maximum occurs at $t = 1/\sqrt{3}$ instead of at the middle of the path at $t=1/2$. We shall see that different forms of variance can lead to vastly different performance in the considered tasks.
\end{itemize}

\textbf{Evaluation metrics.} We evaluate the models using the following metrics. 
First, we use the standard mean squared error (MSE) and  the relative Frobenius norm error (RFNE) to measure the difference between predicted and true snapshots. 
Second, we compute metrics such as the peak signal-to-noise ratio (PSNR), and the structural similarity index measure (SSIM) to further quantify the quality and similarity of the generated snapshots \citep{wang2004image}. 
Third, we use the Pearson correlation coefficient to assess the correlation between predicted and true snapshots.


\textbf{Training details.} 
We use an autoencoder (AE) to embed the training data into a low-dimensional latent space, which enables the model to capture the most relevant features of the data while reducing dimensionality \citep{azencot2020forecasting}; see App. \ref{app_exp} for further discussion of the motivation.
We then train a flow matching model in this latent space. 
Training generative models in latent space has also been considered by \citep{vahdat2021score} for score matching models and by \citep{dao2023flow} for flow matching models. 
To train the AE, we minimize a loss function that  consists of reconstruction error, in terms of MSE, between the input data and its reconstructed version from the latent space. 
The choice of the AE architecture is tailored to the complexity of the dataset (see App. \ref{app_exp} for details).


\subsection{Probabilistic Forecasting of Dynamical Systems}

We evaluate the performance of our proposed probability path model on challenging dynamical systems to demonstrate its effectiveness in forecasting complex continuous dynamics. Specifically, we consider the following tasks (for details see App. \ref{app_datasets}):
    \vspace{-0.2cm}
\begin{itemize}[leftmargin=*]
    \item \textbf{Fluid Flow Past a Cylinder (FPC):} This task involves forecasting the vorticity of a fluid flowing past a cylinder. The model conditions on the first 5 frames and predicts the subsequent 20 frames at a resolution of $64\times 64$ with 1 channel representing vorticity.
    \vspace{-0.15cm}
    \item \textbf{Shallow-Water Equation (SWE):} This dataset models the dynamics of shallow-water equations~\citep{takamoto2022pdebench}, capturing essential aspects of geophysical fluid flows. We use the first 5 frames for conditioning and predict the next 15 frames at a resolution of $128\times 128$ with 1 channel representing horizontal flow velocity.
    \vspace{-0.15cm}
    \item \textbf{Diffusion-Reaction Equation (DRE):} This dataset models the dynamics of a 2D diffusion-reaction equation~\citep{takamoto2022pdebench}.  
    We use the first 5 frames for conditioning and predict 15 future frames at a 128 $\times$ 128 with 2 channels representing velocity in the $x$ and $y$ directions.
    \vspace{-0.15cm}
    \item \textbf{Incompressible Navier-Stokes Equation (NSE):} As a more challenging benchmark, we consider forecasting the dynamics of a 2D incompressible Navier-Stokes equation~\citep{takamoto2022pdebench}. 
    We use the first 5 frames for conditioning and predict the next 20 frames at a resolution of $512\times 512$ with 2 channels representing velocity in the $x$ and $y$ directions.
\end{itemize}
    \vspace{-0.1cm}


Table \ref{tab:result_fluid} summarizes the performance of our model compared to other probability path models across all tasks. It can be seen that our probability path model achieves the lowest test MSE and RFNE across all tasks, indicating more accurate forecasts. Moreover, the higher PSNR and SSIM scores indicate that our model better preserves spatial structures in the predictions.
Despite the similarity of our proposed model with the stochastic interpolant of \citep{chen2024probabilistic}, in that both models use consecutive samples to define the path, our model outperforms the stochastic interpolant model, suggesting that choosing the maximum variance to occur at the middle of the path is a better choice.  
Importantly, our model is highly efficient during inference time since it requires very few sampling steps; this is significantly lower compared to other models.

Figure \ref{fig:result_fluid_correlation} shows the Pearson correlation coefficients of the predicted snapshots over time for all models. Our model's predictions shows a slower decay of correlation coefficients compared to other models, indicating better temporal consistency and long-term predictive capability. Correlation coefficients about $95\%$ indicate performance on par with physics-based numerical simulators.

\begin{table}[!h] 
	\caption{Results for forecasting  dynamical systems using different probability path models for flow matching. Results are averaged over 5 generations obtained with 9 sampling steps $(N=10)$ using RK4. For our model, we use $\sigma_{min} = 0.001$ and $\sigma_{sam} = 0$ for all the considered tasks.}
 \label{tab:result_fluid}
	\centering
		\begin{tabular}{l l  c cccc}
			\toprule
			Task & Model &  Test MSE ($\downarrow$) & Test RFNE ($\downarrow$) & PSNR ($\uparrow$) & SSIM ($\uparrow$) \\	
			\midrule 
			 \multirow{5}{*}{Flow past Cylinder} & RIVER & 1.86e-03 & 4.48e-02  & 44.30 & 0.99  \\
			   &   VE-diffusion  & 2.29e-01 & 4.74e-01 & 27.36  & 0.55  \\
			& VP-diffusion & 3.58e-03 & 6.09e-02 & 42.37  & 0.98  \\
   			& SI ($b_t = t^2$)  & 3.40e-03 &  6.10e-02 & 41.81 & 0.98 \\
            	& SI ($b_t = t$)  & 1.19e-02 & 9.96e-02 & 39.64 & 0.99 \\
			\rowcolor{Wheat1} &   \textbf{Ours} $(\sigma = 0.01)$ & \textbf{3.79e-04} & \textbf{2.30e-02} & \textbf{48.88} & \textbf{1.00}  \\
             \midrule
			\multirow{5}{*}{Shallow-Water}&  RIVER & 9.18e-04 & 1.49e-01 & 34.92 & 0.92 \\
			&  VE-diffusion  & 1.32e-02 & 5.66e-01 & 28.10 & 0.55   \\
			&  VP-diffusion & 1.39e-03 & 1.81e-01 & 34.07 & 0.87 &  \\
   			&  SI ($b_t = t^2$)    & 1.05e-03 & 1.53e-01 & 35.59 & 0.89  \\
            &  SI ($b_t = t$)    & 6.74e-04 & 1.29e-01 & 36.08 & 0.93 \\
            \rowcolor{Wheat1}  &  \textbf{Ours} $(\sigma = 0.1)$ & \textbf{6.60e-04} & \textbf{1.28e-01}  &\textbf{36.10} & \textbf{0.93} \\
              \midrule
			\multirow{5}{*}{Diffusion-Reaction} &  RIVER & 2.87e-03 & 2.28e-01 & 38.12  & 0.82 \\
			&  VE-diffusion  & 1.04e-01 & 1.71 & 32.98 & 0.36  \\
			&  VP-diffusion &  2.02e-02 & 7.03e-01 & 34.98 & 0.52  \\
   			&  SI ($b_t = t^2)$  & 6.17e-02 & 8.62e-01 & \textbf{45.68} &  0.76 \\
            &  SI ($b_t = t)$  & 3.72e-04 & 1.18e-01  & 34.24  & 0.89 \\
			\rowcolor{Wheat1}  &  \textbf{Ours} $(\sigma = 0)$   & \textbf{3.56e-04}  & \textbf{1.16e-01} & 34.34  & \textbf{0.90} \\
              \midrule
			\multirow{5}{*}{Navier-Stokes} & RIVER          & 2.84e-02 & 8.67e-01 & \textbf{30.75}  & 0.63 \\
			& VE-diffusion  & 1.58e-01 & 2.31 & 26.90 & 0.33  \\
			& VP-diffusion &  2.09e-01 & 2.48 & 27.96 & 0.30  \\
   			&  SI ($b_t = t^2)$  & 1.27e-03 & 2.66e-01 & 30.73 &  0.90 \\
            &  SI ($b_t = t)$  & 1.13e-03 & 2.54e-01 & 30.66 & 0.93  \\
            \rowcolor{Wheat1} &   \textbf{Ours} $(\sigma = 0.1)$ & \textbf{1.13e-03}  & \textbf{2.53e-01} & 30.66  & \textbf{0.93} \\
			\bottomrule
	\end{tabular}
\end{table}

\begin{figure}[!h] 
\centering
\vspace{0.5cm}

\begin{overpic}[width=0.48\textwidth]{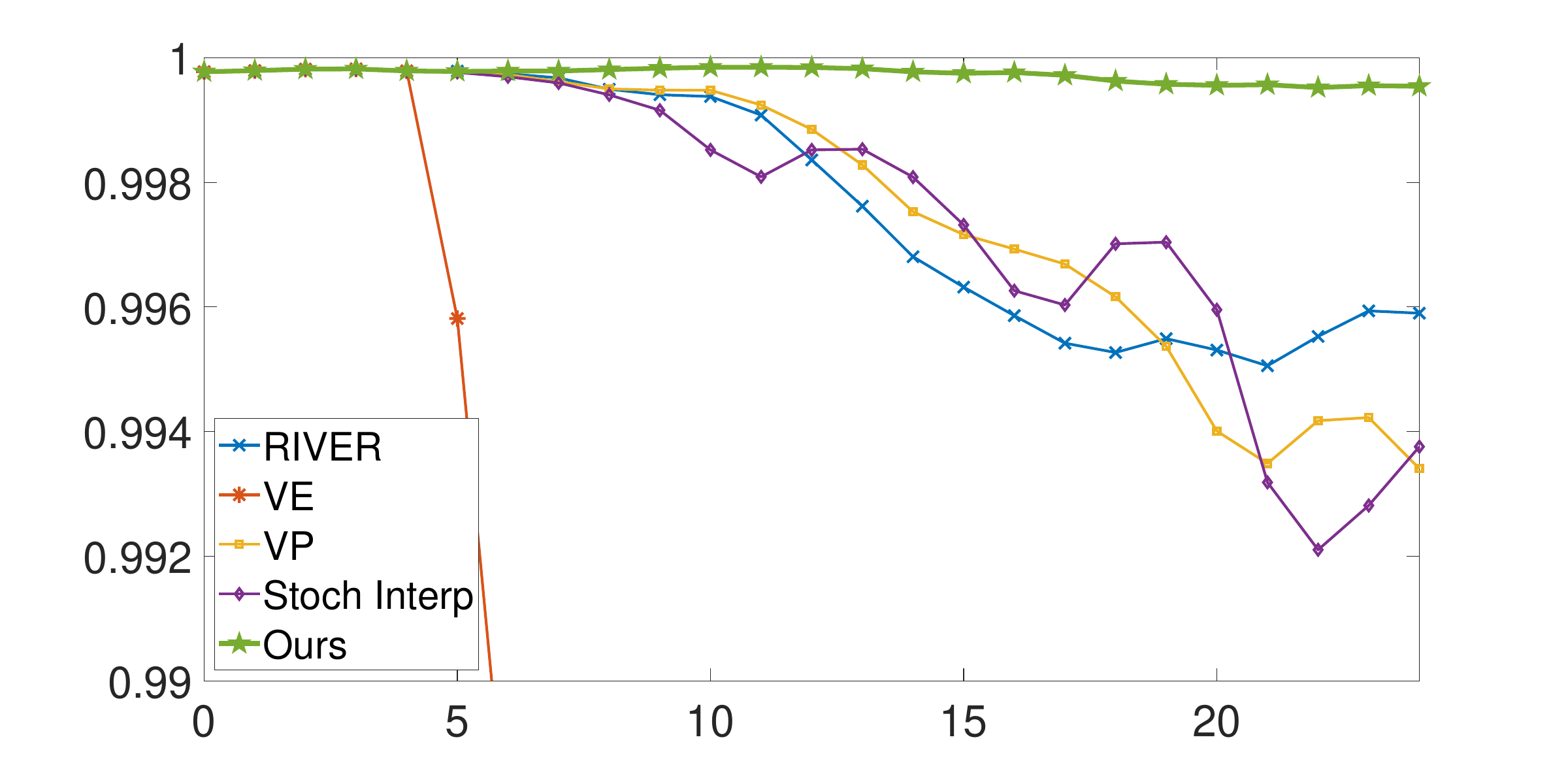}
    \put(-0,15){\rotatebox{90}{\footnotesize correlation}}
    \put(30,48){\textbf{Simple Fluid Flow}}
\end{overpic}
\begin{overpic}[width=0.48\textwidth]{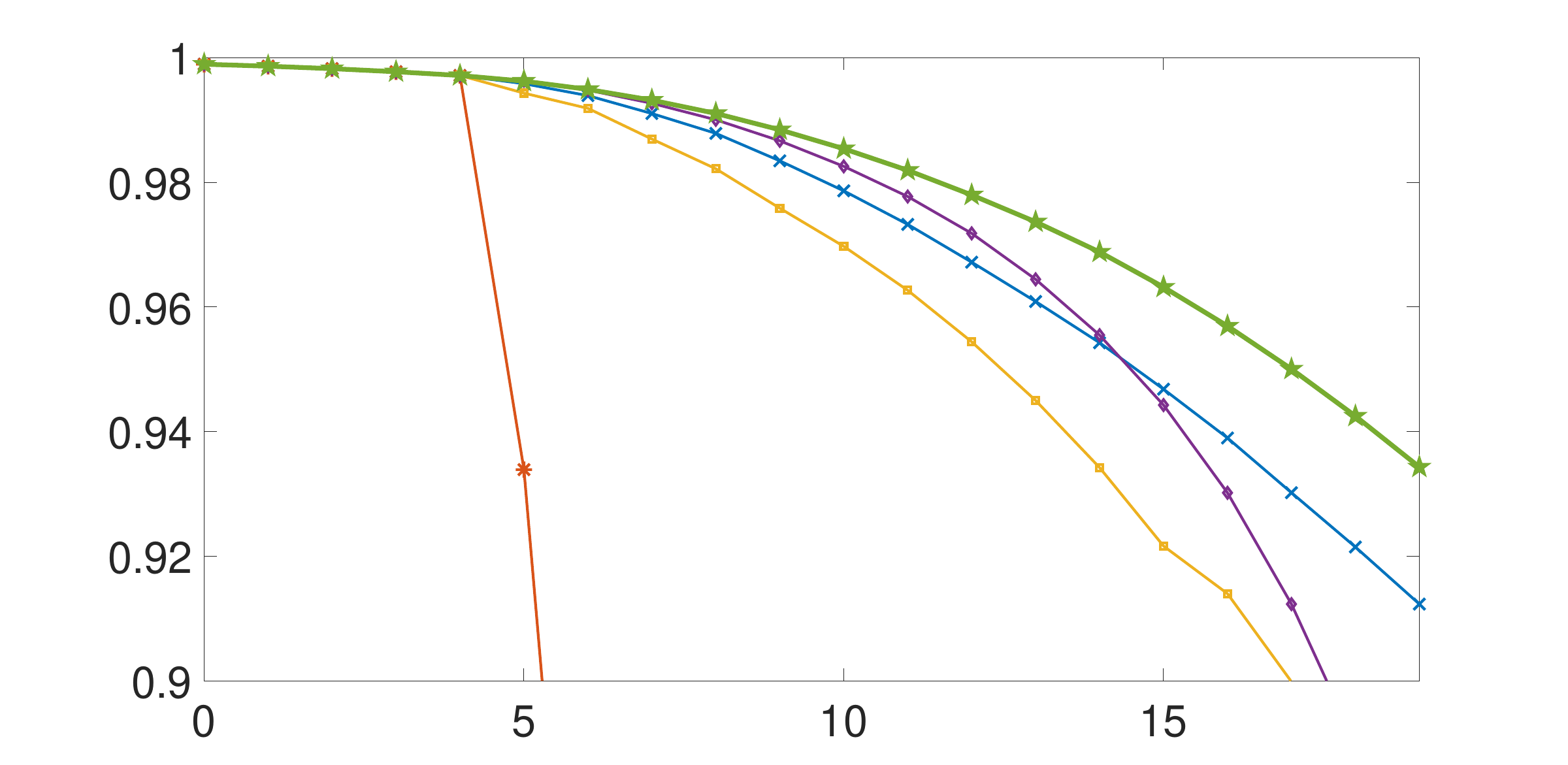}
    \put(24,48){\textbf{Shallow-Water Equation}}
\end{overpic}

\vspace{0.3cm}

\begin{overpic}[width=0.48\textwidth]{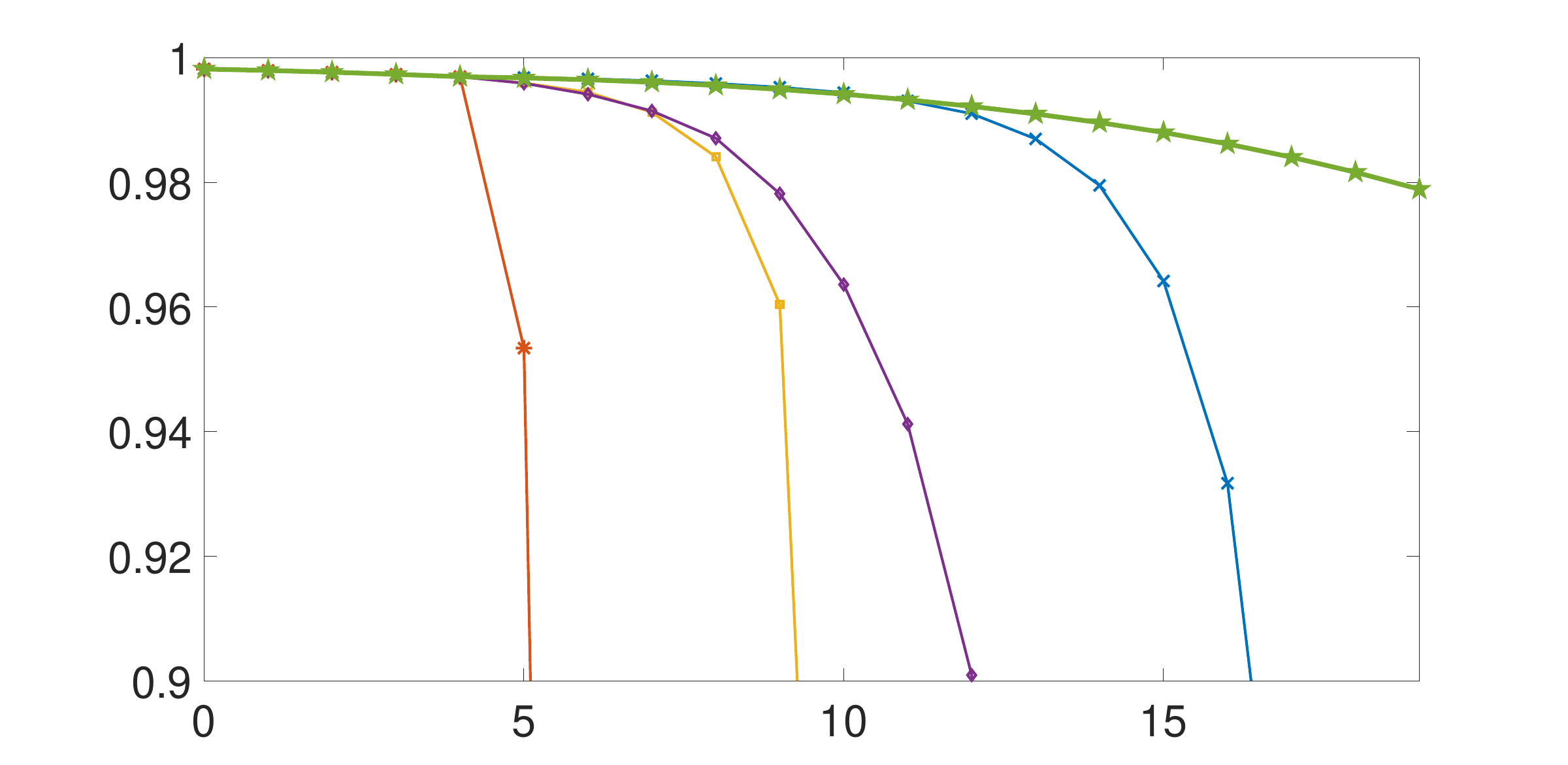}
    \put(35,-2){\footnotesize Prediction Steps}
    \put(-0,15){\rotatebox{90}{\footnotesize correlation}}
    \put(19,48){\textbf{Reaction-Diffusion Equation}}
\end{overpic}
\begin{overpic}[width=0.48\textwidth]{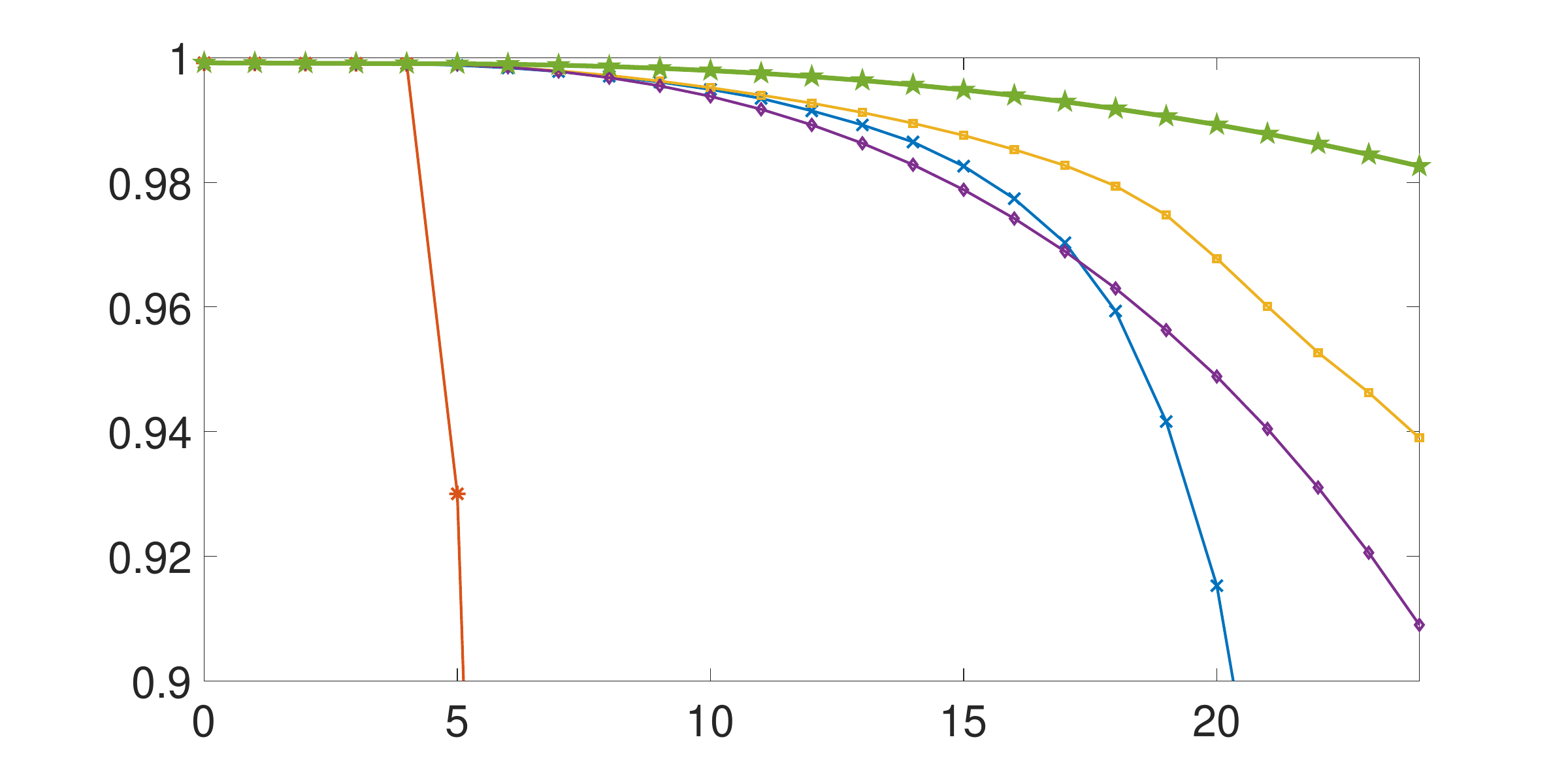}
    \put(35,-2){\footnotesize Prediction Steps}
    \put(25,48){\textbf{Navier-Stokes Equation}}
\end{overpic}

\caption{Pearson correlation coefficient to assess the correlation between predicted and true snapshots at various prediction steps for different probability path models. Our probability path model shows the best performance on all three tasks. Note that the first 5 snapshots are the conditioning snapshots.}
\label{fig:result_fluid_correlation} 
\vspace{0.1cm}
\end{figure}
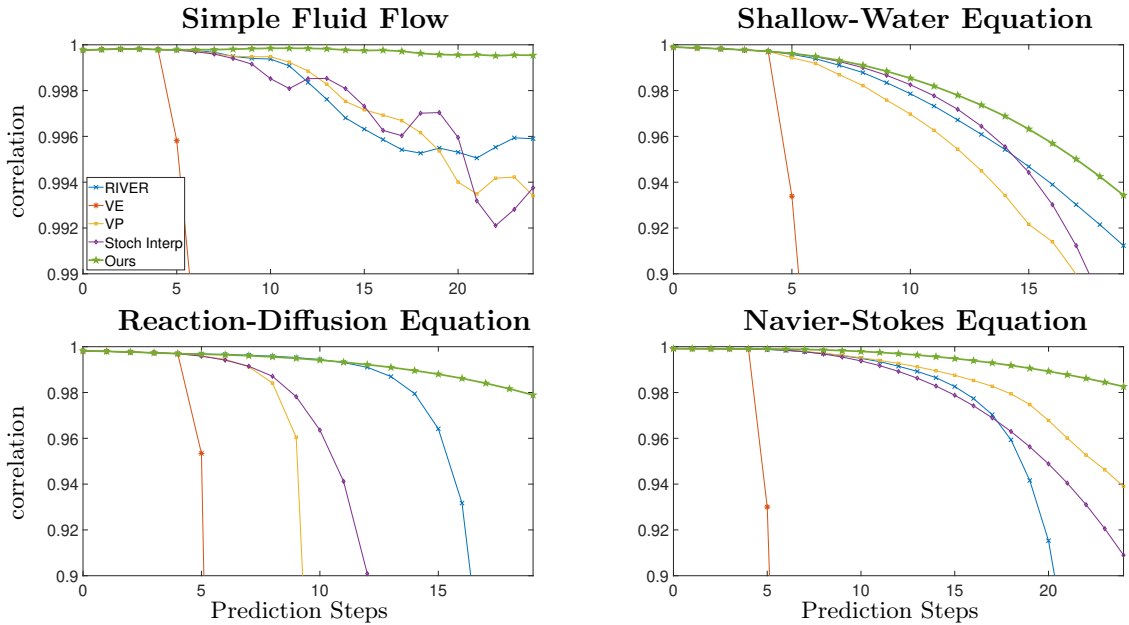

Figure \ref{fig:convergence}  compares the training loss curves of our model with others for the FPC and the SWE task. Our method leads to faster convergence during training and smoother loss curves.
This suggests that our model requires fewer iterations to generate high-quality samples when compared to other  models. We find that there is no significant difference in training time between the models when trained for the same number of epochs.

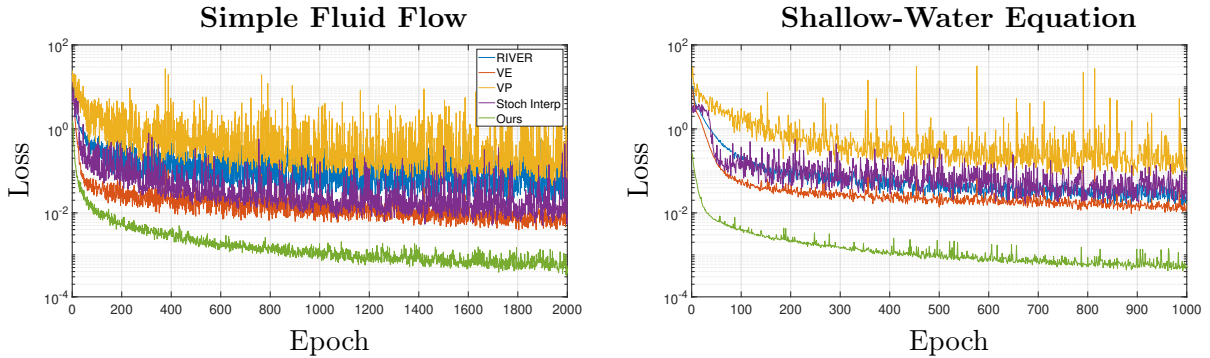
\begin{figure}[!h] 
\vspace{0.5cm}
\centering
\begin{minipage}{0.46\textwidth}
    \begin{overpic}[width = 1.1\textwidth]{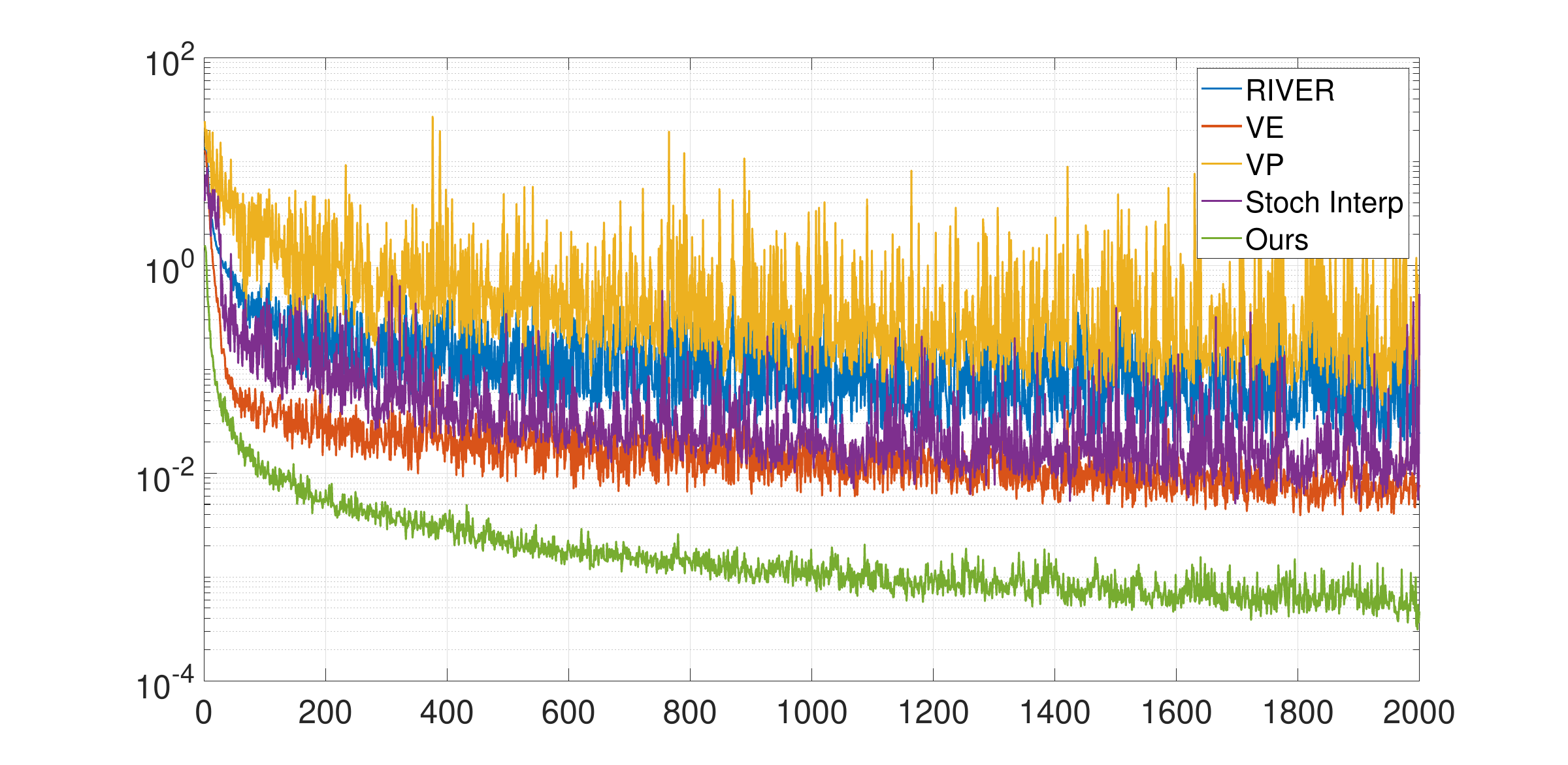}
    \put(3,21){\rotatebox{90}{Loss}}
    \put(47,-3){Epoch}
    \put(33,48){\textbf{Simple Fluid Flow}}
    \end{overpic}
\end{minipage}
\hspace{0.3cm}
\begin{minipage}{0.46\textwidth}
    \begin{overpic}[width = 1.1\textwidth]{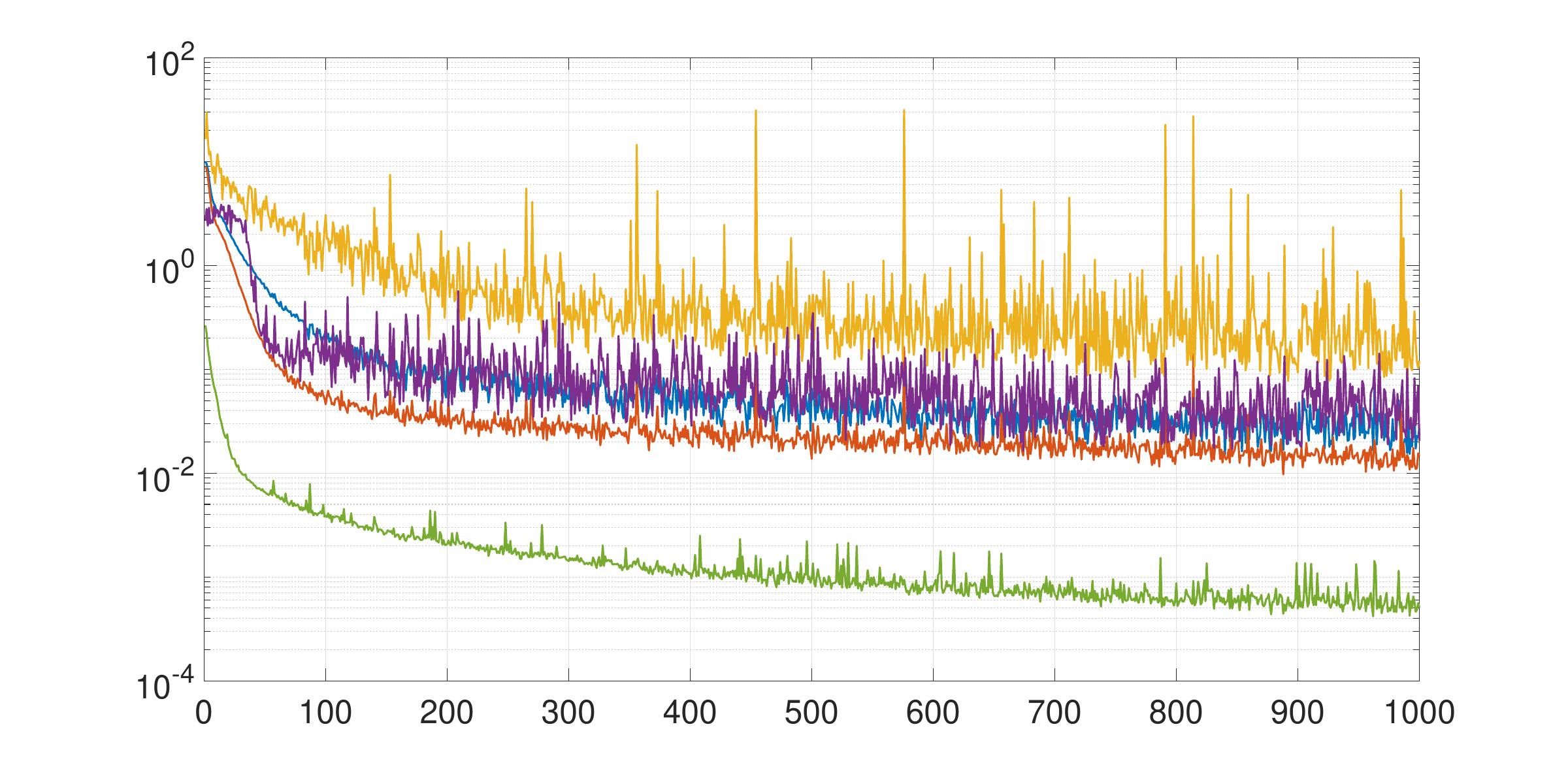}
    \put(3,21){\rotatebox{90}{Loss}}
    \put(47,-3){Epoch}
    \put(27,48){\textbf{Shallow-Water Equation}}
    \end{overpic}
\end{minipage}
\vspace{+0.25cm}
\caption{Training loss for different models of probability path for the fluid flow past a cylinder task. Our model leads to fastest convergence and smoothest loss curve among all models.
}
\vspace{+0.3cm}
\label{fig:convergence}
\end{figure}

\subsection{Discussion}

Our empirical results showed that our proposed model consistently outperformed other models across several forecasting tasks involving different PDEs. Our model shows improved training efficiency, with faster convergence reducing the computational resources and time required for model training. Moreover, our model is efficient during inference time since it only requires a few sampling steps, making it practical for real-world applications where computational efficiency is crucial. Additionally, the model maintained better temporal consistency, as indicated by a slower decay of Pearson correlation coefficients over longer prediction horizons. These findings validate the effectiveness of our approach in modeling complex dynamical systems. Our focus in this paper is on near- to mid-term forecasting (15–20 steps), 
which is a standard set-up in spatio-temporal video and PDE benchmarks. While our framework is compatible with longer rollouts in principle, scaling to significantly longer horizons remains a long-standing challenge, as for most forecasting tasks, due to compounding errors and distributional drift, especially in high-dimensional dynamics. This is an important direction for future work.

Our ablation study (see App. \ref{app_ablation}) further validated the advantages of our proposed probability path model. We found that larger $\sigma$ values not only contributed to smoother training loss curves but also enhanced the overall stability and efficiency of the model. Moreover, we saw that $\sigma_{\text{sam}} = 0$ can be fixed without compromising accuracy. We also provide studies to understand the effect of using context frames and the choice of $\sigma_{min}$ on test performance in App. \ref{app_ablation}. 
In particular, in Table \ref{tab:sigma_min_analysis} we see noticeable improvement in test performance using $\sigma_{min} > 0$, showing the advantages of going beyond the commonly used Brownian bridge model (see \citep{tong2023improving, tong2023simulation}).  
Importantly, the fact that our model achieved improved performance even with the simplest sampler (Euler scheme) and a minimal number of sampling steps (as few as one for the FPC task; see Figure \ref{fig:mse_samplingsteps}) validates its practical applicability, especially in scenarios where computational resources and time are limited. 

Lastly, we provide an expanded discussion to position the proposed probability path relative to the stochastic interpolant (SI) path. While the SI path offers simplicity and analytical tractability, it lacks principled guidance for choosing  $b_t$, which is an important design choice. Our proposed probability path is grounded in the dynamical SB framework, which,  in our formulation, models the most likely stochastic evolution between two Gaussian  distributions centered around time-adjacent samples (with a minimal variance  $\sigma_{min}^2$) under prior dynamics.  From this perspective, the path is not simply a heuristic interpolation and the choice of $b_t = t$ is justified, whereas the inclusion of a non-zero $\sigma_{min}^2$ is important in the context of probabilistic forecasting. Note that when $\sigma_{min} = 0$, our proposed model simplifies to the SI model with $b_t = t$. A key implication of this viewpoint is that the variance structure of the path plays a crucial role: it reflects the uncertainty in matching the two endpoints under the prior dynamics. Our proposed path introduces a variance profile that is symmetric and peaks at the midpoint $t = 1/2$, better reflecting the intrinsic uncertainty in interpolating between endpoints. In contrast, the preferred SI path suggested in \citep{chen2024probabilistic} (with $b_t = t^2$) has an asymmetric variance peak at $t = 1/\sqrt{3}$, which may misalign with the data geometry. Empirically, we see that our probability path model mostly outperforms the SI models  across the considered tasks (see Table \ref{tab:result_fluid}). As expected, when $\sigma_{min}$ is very close to zero, the test performance of our model is only marginally better than that of the SI model with $b_t = t$ for most tasks. As discussed earlier (see also App. \ref{app_ablation}), using bigger values of $\sigma_{min} > 0$ can not only improve test performance but also help improve training stability. 

\section{Conclusion}

In this work, we investigated the use of flow matching in latent space for probabilistic forecasting of spatio-temporal dynamics, providing a theoretical framework and efficient algorithms. We demystified the critical role of the probability path design in this setting and proposed an improved probability path model. Our model is theoretically motivated via the SB and dynamical optimal transport perspective. It leverages the inherent continuity and correlation in the spatio-temporal data, leading to more stable training and faster convergence. Our empirical evaluations on several PDE forecasting tasks demonstrated that our model performs favorably when compared to existing models. 
While we focus on the flow matching approach, we leave comprehensive comparisons with more computationally demanding approaches, such as score matching \citep{song2020score} and bridge matching \citep{de2021diffusion}, to future work.

Our findings position flow matching as a promising alternative to diffusion-based generative models in PDE forecasting. While diffusion models have shown strong performance, they are typically not simulation-free and require many iterative sampling steps. They rely on simulating reverse SDEs or ODEs, often using 25-100+ steps for generation. Therefore, they can be computationally costly and difficult to tune for high-dimensional spatio-temporal systems. In contrast, flow matching enables training wihout simulating stochastic processes,   by directly learning a continuous vector field aligned with a designed probability path and few-step sampling, offering improved inference efficiency. Our results show that careful design of the probability path, especially those tailored to continuous-time dynamics, can significantly improve training stability and forecast performance. These insights suggest that flow-based methods not only offer theoretical elegance via connections to optimal transport and Schrödinger bridges, but also hold practical advantages for scalable and controllable forecasting in complex dynamical systems.

\subsubsection*{Acknowledgments}
The computations were enabled by resources provided by the National Academic Infrastructure for Supercomputing in Sweden (NAISS), partially funded by the Swedish Research Council
through grant agreement no. 2022-06725 (NAISS 2024/5-269).

SHL would like to acknowledge support from the Wallenberg Initiative on Networks and Quantum Information (WINQ) and the Swedish Research Council (VR/2021-03648). 
AY was supported by
the SciAI Center, funded by the Office of Naval Research under Grant Number N00014-23-1-2729.
EH was supported by the U.S. Department of Energy, Office of Science, Office of Advanced Scientific Computing Research, Department of
Energy Computational Science Graduate Fellowship under Award Number DE-SC0024386.
NBE would like to acknowledge LBL’s LDRD initiative for providing partial support.
XSL was supported in part by the U.S. Department of Energy, Office of Science, Office of Advanced Scientific Computing Research's Applied Mathematics SciML program under Contract No. DE-AC02-05CH11231 at Lawrence Berkeley National Laboratory.

\bibliography{reference}
\bibliographystyle{tmlr}

\input{supplement_TMLR}

\end{document}

%% file: math_commands.tex
\usepackage{amsmath,amsfonts,bm,overpic,xcolor}

\def\eqref#1{equation~\ref{#1}}



%% file: supplement_TMLR.tex
\newpage 

\appendix
\section*{Appendix}

This appendix is organized as follows. 
In App. \ref{app_relatedwork}, we provide a detailed discussion of related work and connect them to our studies.
In App. \ref{app_connect}, we provide some remarks on the connection of flow matching models to other generative models. 
In App. \ref{app_lossparam}, we provide commonly used loss parametrizations and compare them with our flow matching loss. 
In App. \ref{app_proof}, we provide theorems and proofs to justify the discussions in Section \ref{sect_framework}-\ref{sect_model}  (see Theorem \ref{thm_1}-\ref{thm_gaussianpaths} and Theorem \ref{thm_optimal_path}-\ref{thm_variance}), as well as additional discussions. 
In App. \ref{app_additional}, we provide additional empirical results. 
In App. \ref{app_exp}, we provide the missing experimental details.

\section{Related Work} \label{app_relatedwork}
\noindent {\bf Modeling and forecasting dynamical systems.} Traditionally dynamical systems arising in scientific applications have been studied through mathematical models based on physical laws, such as ODEs and PDEs \citep{evans2022partial}. These methods often require significant domain knowledge and strong assumptions, limiting their flexibility in data-driven scenarios where governing equations are unknown. Statistical and machine learning (ML) methods have emerged as powerful alternatives to study these systems. Classical approaches include  autoregressive integrated moving average (ARMA) and state-space models such as Kalman filters \citep{box2015time}, which have long been used for time series forecasting but struggle with highly nonlinear dynamics. ML methods such as recurrent neural networks (RNNs) \citep{bengio2013advances} and variants for spatio-temporal data~\citep{shi2015convolutional,lyu2024wavecastnet} are capable of learning complex temporal dependencies but they are challenging to train. Neural ODEs \citep{chen2018neural,queiruga2021stateful} offer a structured approach to continuous-time modeling by integrating neural networks with ODE solvers. However, these models still face challenges with generalizing to new systems, and learning continuous physical dynamics \citep{krishnapriyan2023learning}. 

\noindent {\bf Generative modeling for time series forecasting.} Generative models, particularly score-based diffusion models and flow-based approaches, have gained significant attention for learning complex data distributions in high-dimensional spaces such as images \citep{song2020score, karras2022, esser2024scaling}, videos \citep{davtyan2023efficient, shrivastava2024video}, and dynamical systems \citep{pathak2024kilometer, kohl2024benchmarking, ren2024learning}. In the context of time series forecasting, diffusion models like TimeGrad \citep{rasul2021autoregressive, meijer2024rise} incorporate conditioning on previous time steps into both the forecasting process and the loss function. Building upon TimeGrad, CSDI \citep{tashiro2021csdiconditionalscorebaseddiffusion} enhances performance in imputation tasks by replacing the RNN encoder with a transformer. ScoreGrad \citep{yan2021scoregradmultivariateprobabilistictime} adapts this framework to a score-based diffusion model for multivariate probabilistic forecasting. Methods such as DSPD and CSPD \citep{bilos2023modelingtemporaldatacontinuous}, based on DDPM and SDE respectively, model dynamics as continuous functions and introduce time-correlated noise functions. Another relevant work, SSSD$^\text{S4}$ \citep{alcaraz2023diffusionbasedtimeseriesimputation}, uses state-space models (S4) to encode time series and performs diffusion over the temporal domain instead of across multivariate components. While these models demonstrate strong performance, they often require many sampling steps during generation. Related to our approach, \citep{chen2024probabilistic} proposes an SDE framework utilizing stochastic interpolants \citep{albergo2023stochastic} for probabilistic forecasting, but their sampler also requires many steps during generation. Another recent work is \citep{tamir2024conditional}, which introduces a flow matching framework for time series modeling within the data space; however, it concentrates on small ODE datasets and does not address forecasting tasks. Meanwhile, \citep{hu2024fm} proposes a model based on rectified flow with ODE-based straight line transport for efficient time series generation.

Applications of diffusion models to spatio-temporal data have primarily focused on video generation \citep{singer2023makeavideo}. For spatio-temporal forecasting on sensor grids, graph neural networks (GNNs) are widely employed in tasks such as traffic prediction \citep{qu2023urbanflowsgraph,Lv2024trajectorylearninggraphs, wen2024diffstgprobabilisticspatiotemporalgraph} and air quality forecasting \citep{wen2024diffstgprobabilisticspatiotemporalgraph}. Another recent work is DYffusion \citep{cachay2023dyffusiondynamicsinformeddiffusionmodel}, which utilizes a non-Markovian diffusion process to propagate images temporally, similar to DDIM. This method iteratively refines initial predictions at larger time steps by conditioning them on predictions at smaller time steps.

\section{Connection to SDE Based Generative Model Frameworks}\label{app_connect}
In this section, we provide some remarks on the connection between flow matching models and SDE based generative models \citep{song2020score}. 

Consider the following continuous-time Gaussian latent variable model \citep{tomczak2021latent}: 
$Z_t = \mathcal{E}(X_t)$,
\begin{align}
Z_t &= A_t Z_0 + B_t Z_1 + L_t \epsilon, \ \ t \in [0,1], \label{eq_latent} 
\\
X_t &= \mathcal{D}(Z_t),
\end{align}
where $t$ is the continuous variable, $X_0 \in \RR^d$ represent data samples, $Z \in \RR^{p}$ is the latent variable,  $\epsilon \sim \mathcal{N}(0, I)$ is independent of the random variables $X_0, Z_0, Z_1$. Here $A_t, B_t$ and $L_t \geq 0$ are pre-specified coefficients which are possibly matrix-valued and time-dependent, $\mathcal{D}$ and $\mathcal{E}$ denote the decoder and encoder map respectively, and  $\mathcal{D} \circ \mathcal{E} = I$. Note that $Z_0$
and $Z_1$ are initial and terminal point of the path $(Z_t)_{t 
\in [0,1]}$ in the latent space.

The above latent variable model can be identified (up to equivalence in law for each $t$) with the linear SDE of the form:
\begin{equation}
    d\hat{Z}_t = F_t \hat{Z}_t dt + H_t Z_1 dt + G_t dW_t,  \ \ \hat{Z}_0 = Z_0,  \ \ t \in [0,1), \label{eq_sde}
\end{equation}
where $(W_t)_{t \in [0,1]}$ is the standard Wiener process. 
By matching the moments, we obtain 
\begin{align}
    F_t &= \dot{A}_t A_t^{-1}, \label{app_id1} \\
    H_t &= \dot{B}_t - \dot{A}_t A_t^{-1} B_t, \\
    G_t G_t^{T} &= \dot{L}_t L_t^{T} + L_t \dot{L}_t^T - \dot{A}_t A_t^{-1} L_t L_t^T - L_t L_t^T A_t^{-T} (\dot{A}_t)^T,  \label{app_id2}
\end{align}
where the overdot denotes derivative with respect to $t$ and $A^T$ denotes the transpose of $A$.

Under the above formulation, various existing generative models such as DDPM \citep{ho2020denoising}, VP-SDE and VE-SDE of 
\citep{song2020score, song2020denoising}, the critically damped SDE of  \citep{dockhorn2021score}, the flow matching models in 
\citep{lipman2022flow, tong2023improving, liu2022flow} and the stochastic interpolants of \citep{albergo2023stochastic} can be recovered, and new models can be derived. 

The following proposition establishes 
the connection between flow matching using our proposed probability path model, the Gaussian latent variable model (\ref{eq_latent}) and the linear SDE model (\ref{eq_sde}).

\begin{proposition} \label{prop_id}
For every $t 
\in [0,1]$, the  $Z_t$ defined in Eq. (\ref{eq_pttopt}) can be identified, up to equivalence in law, with the $Z_t$ generated by the latent variable model (\ref{eq_latent}) with $A_t = (1-t) I$, $B_t = t I$, $L_t = \sqrt{\sigma_{min}^2 + \sigma^2 t(1-t)} I$. For $t \in [0,1)$, it can also be identified with the solution $
\hat{Z}_t$ of the linear SDE (\ref{eq_sde}) with $F_t = -I/(1-t)$, $H_t = (1 + \frac{t}{1-t})I$ and $G_t = \sqrt{\sigma^2 +\frac{2  \sigma_{min}^2 }{1-t}} I$. Moreover, $\lim_{t \to 1} \hat{Z}_t =^d Z_1 + \sigma_{min} 
\epsilon$, where $\epsilon \sim \mathcal{N}(0,I)$ and $=^d$ denotes equivalence in distribution. 
\end{proposition}

\begin{proof}
The identification follows from matching the moments of $Z_t$ and $\hat{Z}_t$, i.e., applying Eq. (\ref{app_id1})-(\ref{app_id2}). 

To prove that $\lim_{t \to 1} \hat{Z}_t =^d Z_1 + \sigma_{min} \epsilon$, 
we use the explicit solution of the SDE:
$$\hat{Z}_t = \Phi(t,0)Z_0 + \int_0^t \Phi(t,s)H_s Z_1 ds + \int_0^t \Phi(t,s)G_s dW_s,$$
where $\Phi(t,s)$ is the fundamental solution of the homogeneous equation $d\Phi(t,s) = F_t \Phi(t,s)dt$ with $\Phi(s,s) = I$.
For our $F_t = -I/(1-t)$, we have $\Phi(t,s) = \exp(-\int_s^t \frac{1}{1-u}du)I = (\frac{1-t}{1-s})I$.
Substituting this and the formula for $H_t$ into the solution, we obtain
$\hat{Z}_t = (1-t)Z_0 + \int_0^t (\frac{1-t}{1-s})(1 + \frac{s}{1-s})Z_1 ds + \int_0^t (\frac{1-t}{1-s})\sqrt{\sigma^2 + \frac{2\sigma_{min}^2}{1-s}}I dW_s$.

Now, let us examine each term as $t \to 1$.
First, 
$(1-t)Z_0 \to 0$ as $t \to 1$ and
$\int_0^t (\frac{1-t}{1-s})(1 + \frac{s}{1-s})Z_1 ds = tZ_1 \to Z_1$ as $t \to 1$.
It remains to deal with the stochastic integral term
$M_t := \int_0^t (\frac{1-t}{1-s})\sqrt{\sigma^2 + \frac{2\sigma_{min}^2}{1-s}} dW_s$. Note that $M_t$ is an It\^o integral that has zero mean, i.e. $\mathbb{E} M_t = 0$, and using It\^o's formula \citep{sarkka2019applied},
\begin{align}
\mathbb{E} M_t^2 &= \int_0^t ((1-t)/(1-s))^2 \left(\sigma^2 + \frac{2\sigma_{min}^2}{1-s} \right) ds \\
&= (1-t)^2 \left[\sigma^2 \int_0^t \frac{1}{(1-s)^2} ds + 2\sigma_{min}^2 \int_0^t \frac{1}{(1-s)^3} ds \right] \\
&= (1-t)^2 \left[\sigma^2 \left(\frac{t}{1-t} \right) + \sigma_{min}^2 \left(\frac{1}{(1-t)^2} - 1 \right) \right] \\
&= (1-t)t\sigma^2 + \sigma_{min}^2 t(2-t), 
\end{align}
which tends to $
\sigma_{min}^2$ as $t \to 1$. Combining the above results, $\lim_{t \to 1} \hat{Z}_t \sim \mathcal{N}(Z_1, \sigma_{min}^2 I)$.

\end{proof}


\section{On Different Loss Parametrizations} \label{app_lossparam}

In this section, we list popular choices of loss parametrization considered in the literature and connect them to our flow matching loss. We refer to \citep{kingma2023understanding} for a more comprehensive discussion. Recall that the Gaussian path that we consider is: $Z_t = a_t Z_0 + b_t Z_1 + c_t \xi$, where $\xi \sim \mathcal{N}(0,I)$.  In general, these loss parametrizations take the form of:
\begin{equation} \label{conditional_opt_app}
    \mathcal{L}(\theta) :=  \mathbb{E}_{t, p_t(Z|\tilde{Z}), q(\tilde{Z}) } \ \omega(t) \| m^\theta_t(Z) - m_t(Z, \tilde{Z}) \|^2, 
\end{equation}
where $\omega(t) > 0$ is a weighting function, $m_t(Z, \tilde{Z})$ is the object (conditioned on $\tilde{Z}$)  to be learnt and $m_t^\theta$ is a neural network model used to learn the object of interest. Depending on which object one would like to learn/match, we have different loss parametrizations.

\noindent {\bf Flow matching.} The flow matching  loss that we focus in this paper is:
\begin{equation} \label{conditional_opt_app}
    \mathcal{L}_{cfm}(\theta) :=  \mathbb{E}_{t, p_t(Z|\tilde{Z}), q(\tilde{Z}) } \ \omega(t) \| v^\theta_t(Z) - u_t(Z | \tilde{Z}) \|^2, 
\end{equation}
where one aims to learn the flow generating vector field: 
\begin{equation}
    u_t(Z|\tilde{Z}) = \frac{c_t'}{c_t} (Z- (a_t Z_0 + b_t Z_1)) + a_t' Z_0 + b_t' Z_1,
\end{equation}

\noindent {\bf Score matching.} The score matching loss is: 
\begin{equation} \label{smloss_app}
    \mathcal{L}_{sm}(\theta) :=  \mathbb{E}_{t, p_t(Z|\tilde{Z}), q(\tilde{Z}) } \ \lambda(t) \| s^\theta_t(Z) - \nabla \log p_t (Z | \tilde{Z}) \|^2, 
\end{equation}
where $\lambda(t) > 0$ is a weighting function and one aims to learn the score function:
\begin{equation}
    \nabla \log p_t (Z | \tilde{Z}) = \frac{a_t Z_0 + b_t Z_1 - Z}{c_t^2}. 
\end{equation}
If $\lambda(t) = c_t^2$, then this reduces to the original score matching loss \citep{song2019generative}, whereas if $\lambda(t) = \beta(1-t)$, this becomes the score flow loss \citep{song2021maximum}.

\noindent {\bf Noise matching.}
The noise matching loss is:
\begin{equation} \label{smloss_app}
    \mathcal{L}_{nm}(\theta) :=  \mathbb{E}_{t, p_t(Z|\tilde{Z}), q(\tilde{Z}) }  \| \epsilon^\theta_t(Z) - \epsilon_t(Z | \tilde{Z}) \|^2, 
\end{equation}
where one aims to learn the noise:
\begin{equation}
    \epsilon_t(Z|\tilde{Z}) = \frac{Z - (a_t Z_0 + b_t Z_1)}{c_t}.
\end{equation}

\section{Theoretical Results and Proofs} \label{app_proof}

In this section, we provide theorems and proofs to justify the discussions in Section \ref{sect_framework} and Section \ref{sect_model}.

\subsection{Connecting Flow Matching with Conditional Flow Matching}

The following theorem justifies the claim that minimizing $\mathcal{L}_{fm}$ is equivalent to minimizing $\mathcal{L}_{cfm}$.

\begin{theorem} \label{thm_1}
    If the conditional vector field $u_t(Z| \tilde{Z})$ generates the conditional probability path $p_t(Z| \tilde{Z})$, then the marginal vector field $u_t$  in Eq. (\ref{eq_marginalvf}) generates the marginal probability path $p_t$ in Eq.  (\ref{eq_marginalpath}). Moreover, if $p_t(Z) > 0$ for all $t, Z$, then $\mathcal{L}_{fm}$ and $\mathcal{L}_{cfm}$ are equal up to a constant independent of $\theta$.
\end{theorem}

\begin{proof}
The proof is a straightforward extension of the proofs  of Theorem 1-2 in \citep{lipman2022flow} from conditioning on data samples to conditioning on latent samples and allowing an arbitrary weighting function $\omega(t)$. 

Suppose that the conditional vector field $u_t(Z| \tilde{Z})$ generates the conditional probability path $p_t(Z| \tilde{Z})$, we would like to show that the marginal vector field $u_t$ in Eq. (\ref{eq_marginalvf}) generates the marginal probability path $p_t$ in Eq. (\ref{eq_marginalpath}). To show this, it suffices to verify that $p_t$ and $u_t$ satisfy the continuity equation: 
\begin{equation}
    \frac{d}{dt} p_t(Z) + div(p_t(Z) u_t(Z)) = 0,
\end{equation}
where the divergence operator, $div$, is defined with respect to the latent variable $Z = (Z^1, \dots, Z^d)$, i.e., $div = \sum_{i=1}^d \frac{\partial}{\partial Z^i}$. 

We begin by taking derivative of $p_t(Z)$ with respect to time:
\begin{align}
    \frac{d}{dt} p_t(Z) &= \frac{d}{dt} \int p_t(Z | \tilde{Z}) q(\tilde{Z}) d\tilde{Z} \\
    &=  \int \frac{d}{dt} p_t(Z | \tilde{Z}) q(\tilde{Z}) d\tilde{Z} \\
    &= - \int div (u_t(Z | \tilde{Z}) p_t(Z | \tilde{Z})) q(\tilde{Z}) d\tilde{Z} \\
    &= - div \left( \int u_t(Z | \tilde{Z}) p_t(Z | \tilde{Z}) q(\tilde{Z}) d\tilde{Z} \right) \\
    &= - div(u_t(Z) p_t(Z)).
\end{align}
In the third line, we use the fact that $u_t(\cdot |\tilde{Z} )$ generates $p_t(\cdot | \tilde{Z})$. In the last line, we use Eq. (\ref{eq_marginalvf}). In the second and forth line above, the exchange of integration and differentiation can be justified by assuming that the integrands satisfy the regularity conditions of the Leibniz rule.

Next, we would like to show that if $p_t(Z) > 0$ for all $t, Z$, then $\mathcal{L}_{fm}$ and $\mathcal{L}_{cfm}$ are equal up to a constant independent of $\theta$. We follow \citep{lipman2022flow} and assume that $q(Z)$ and $p_t(Z|\tilde{Z})$ are decreasing to zero sufficiently fast as $\|Z\| \to 0$, that $u_t, v_t, \nabla_\theta v_t$ are bounded, so that all integrals exist and exchange of integration order is justified via Fubini's theorem.

Using the bilinearity of the 2-norm, we decompose the squared losses as:
\begin{align}
    \|v_t^\theta(Z) - u_t(Z)\|^2 &= \|v_t^\theta(Z)\|^2 - 2 \langle v^\theta_t(Z), u_t(Z) \rangle + \|u_t(Z)\|^2, \\
      \|v^\theta_t(Z) - u_t(Z|\tilde{Z})\|^2 &= \|v^\theta_t(Z)\|^2 - 2 \langle v^\theta_t(Z), u_t(Z|\tilde{Z}) \rangle + \|u_t(Z|\tilde{Z})\|^2.
\end{align}

Now, 
\begin{align}
    \mathbb{E}_{p_t(Z)} \|v^\theta_t(Z)\|^2 &= \int \|v^\theta_t(Z)\|^2 p_t(Z) dZ \\ 
    &= \int \int \|v^\theta_t(Z)\|^2 p_t(Z| \tilde{Z}) q(\tilde{Z}) d\tilde{Z} dZ \\
    &= \mathbb{E}_{q(\tilde{Z}), p_t(Z|\tilde{Z})} \|v^\theta_t(Z)\|^2,
\end{align}
where we use Eq. (\ref{eq_marginalpath}) in the second equality above and exchange the order of integration in the third equality. 

Next, we compute:
\begin{align}
     \mathbb{E}_{p_t(Z)} \langle v^\theta_t(Z), u_t(Z) \rangle &= \int \bigg\langle v^\theta_t(Z), \frac{\int u_t(Z|\tilde{Z}) p_t(Z|\tilde{Z}) q(\tilde{Z}) d\tilde{Z} }{p_t(Z)} \bigg\rangle p_t(Z) dZ \\
     &= \int \bigg\langle v^\theta_t(Z), \int u_t(Z|\tilde{Z}) p_t(Z|\tilde{Z}) q(\tilde{Z}) d\tilde{Z} \bigg\rangle dZ \\
     &= \int \int \langle v^\theta_t(Z), u_t(Z|\tilde{Z}) \rangle p_t(Z|\tilde{Z}) q(\tilde{Z}) d\tilde{Z} dZ \\
     &= \mathbb{E}_{q(\tilde{Z}), p_t(Z|\tilde{Z})} \langle v^\theta_t(Z), u_t(Z | \tilde{Z}),
\end{align}
where we first plug in Eq. (\ref{eq_marginalvf}) and then exchange the order the integration in order to arrive at the last equality.

Finally, noting that $u_t$ are $
\omega(t)$ independent of $
\theta$ (and are thus irrelevant for computing the loss gradients), we have proved the desired result.
\end{proof}

\subsection{Identifying the Vector Field that Generates the Gaussian Paths}
Similar to Theorem 3 in \citep{lipman2022flow}, we have the following result, which identifies the unique vector field that generates the Gaussian probability path.

\begin{theorem} \label{thm_gaussianpaths}
    Let $p_t(Z|\tilde{Z})$ be the Gaussian probability path defined in Eq. (\ref{eq_generalclass}) and consider the flow map $\psi_t$ defined as $\psi_t(Z) = a_t Z_0 + b_t Z_1 + c_t Z$ with $c_t >0$. Then the unique vector field that defines $\psi_t$ is:
    \begin{equation}
    u_t(Z|\tilde{Z}) = \frac{c_t'}{c_t} (Z- (a_t Z_0 + b_t Z_1)) + a_t' Z_0 + b_t' Z_1,
    \end{equation}
    where prime denotes derivative with respect to $t$, and $u_t(Z|\tilde{Z})$ generates the Gaussian path $p_t(Z|\tilde{Z})$.
\end{theorem}

\begin{proof}
Let $p_t(Z|\tilde{Z})$ be the Gaussian probability path defined in Eq. (\ref{eq_generalclass}) and consider the flow map $\psi_t$ defined as $\psi_t(Z) = a_t Z_0 + b_t Z_1 + c_t Z$. We would like to show that the unique vector field that defines $\psi_t$ is:
    \begin{equation}
    u_t(Z|\tilde{Z}) = \frac{c_t'}{c_t} (Z- (a_t Z_0 + b_t Z_1)) + a_t' Z_0 + b_t' Z_1,
    \end{equation}
    and $u_t(Z|\tilde{Z})$ generates the Gaussian path $p_t(Z|\tilde{Z})$.

We denote $w_t = u_t(Z|\tilde{Z})$ for notational simplicity. Then, 
\begin{equation}
    \frac{d}{dt} \psi_t(Z) = w_t(\psi_t(Z)). 
\end{equation}
Since $\psi_t$ is invertible (as $c_t > 0$), we let $Z = \psi^{-1}(Y)$ and obtain 
\begin{equation} \label{eq_proof}
\psi_t'(\psi^{-1}(Y)) = w_t(Y), 
\end{equation}
where the prime denotes derivative with respect to $t$ and we have used the apostrosphe notation for the derivative to indicate that $\psi_t'$ is evaluated at $\psi^{-1}(Y)$.

Inverting $\psi_t(Z)$ gives:
\begin{equation}
    \psi_t^{-1}(Y) = \frac{Y - \mu_t(\tilde{Z})}{c_t},
\end{equation}
where $\mu_t(\tilde{Z}) := a_t Z_0 + b_t Z_1$. 

Differentiating $\psi_t$ with respect to $t$ gives $\psi_t'(Z) = c_t' Z + 
\mu_t'(\tilde{Z})$. 

Plugging the last two equations into  Eq. (\ref{eq_proof}), we obtain:
\begin{equation}
    w_t(Y) = \frac{c_t'}{c_t} ( Y - \mu_t(\tilde{Z})) + \mu_t'(\tilde{Z})
\end{equation}
which is the result that we wanted to show. 

\end{proof}

\subsection{Solution to the Gaussian Schrodinger Bridge Problem} \label{app_GSB}

We prove Theorem \ref{thm_optimal_path} in this subsection. First, we recall Theorem 3 from \citep{bunne2023Schrodinger}. Theorem \ref{thm_optimal_path} then follows from an application of the theorem.

Let $\xi_0 \sim \mathcal{N}(\mu_0, \Sigma_0)$ and $\xi_1 \sim \mathcal{N}(\mu_1, \Sigma_1)$ be two given Gaussian random variables. From now on, by abusing the notation we will also denote the measures of these Gaussians by $\xi_0$ and $\xi_1$ respectively. 

Let $\sigma \geq 0$ and set:
\begin{align}
    D_\sigma &:= (4 \Sigma_0^{1/2}
     \Sigma_1 \Sigma_0^{1/2} + \sigma^4 I)^{1/2} \\
    C_\sigma &:= \frac{1}{2}(\Sigma_0^{1/2} D_\sigma \Sigma_0^{-1/2} - \sigma^2 I).  
\end{align}

Consider the following  Schrodinger bridges with Gaussian marginal constraints:
\begin{equation}
    \min_{\mathbb{P}_0 = \xi_0, \mathbb{P}_1 = \xi_1} D_{KL}(\mathbb{P} \| \mathbb{Q}),
\end{equation}
where $\mathbb{Q}$ is the path measure of the linear SDE: 
\begin{equation}
    dY_t = (c_t Y_t + \alpha_t) dt + g_t dW_t := f_t dt + g_t dW_t.
\end{equation}
Here, $c_t: \RR^+ \to \RR$, $\alpha_t: \RR^+ \to \RR^d$ and $g_t: \RR^+ \to \RR^+$ are smooth functions, and $W_t$ is the standard Wiener process. 

The following theorem is a simplified version of Theorem 3 in \citep{bunne2023Schrodinger}.

\begin{theorem}[Theorem 3 in \citep{bunne2023Schrodinger}] \label{app_thmSB}
    The solution $\mathbb{P}^*$  to the Gaussian Schrodinger bridges is (the path measure of) a Markov Gaussian process whose marginal variable $X_t \sim \mathcal{N}(\mu_t, \Sigma_t)$, 
    where 
    \begin{align}
        \mu_t &= (\tau_t - r_t \tau_1) \mu_0 + r_t \mu_1 + \zeta(t) - r_t \zeta(1), \\
        \Sigma_t &= (\tau_t - r_t \tau_1)^2 \Sigma_0 + r_t^2 \Sigma_1 + r_t (\tau_t - r_t \tau_1) (C_{\sigma^*} + C_{\sigma^*}^T) + \kappa(t,t) (1-\rho_t) I.
    \end{align}
    In the above,
    \begin{align*}
    \tau_t &= \exp\left(\int_0^t c_s ds \right), \ \ \  \kappa(t,t') = \tau_t \tau_{t'} \int_0^t \tau_s^{-2} g_s^2 ds, \ \ \ r_t = \kappa(t,1)/\kappa(1,1), \\
    \sigma^* &= \sqrt{\tau_1^{-1} \kappa(1,1)}, \ \ \  \zeta(t) =  \tau_t \int_0^t \tau_s^{-1} \alpha_s ds, \ \ \ \rho_t = \frac{\int_0^t \tau_s^{-2} g_s^2 ds }{\int_0^1 \tau_s^{-2} g_s^2 ds }. 
    \end{align*}
\end{theorem}

Now we prove Theorem \ref{thm_optimal_path}. 
\begin{proof}[Proof of Theorem \ref{thm_optimal_path}] We shall apply Theorem \ref{app_thmSB} to the Gaussian SB (\ref{Gaussian_SB}) by setting $\mu_0 := Z_0$, $\mu_1 = Z_1$, $\Sigma_0 = \Sigma_1 := \sigma_{min}^2 I$, $\alpha_t := \dot{a}_t Z_0 + \dot{b}_t Z_1$, $c_t := 0$, and $g_t := \omega$.
Then, we have $\tau_t = 1$, $\kappa(t,t') = \omega^2 t$, $r_t = t$, $\sigma^* = \omega$, $\zeta(t) = \int_0^t \alpha_s ds = \left(\int_0^t \dot{a}_s ds \right) Z_0 + \left(\int_0^t \dot{b}_s ds \right) Z_1 = (a_t-a_0) Z_0 + (b_t - b_0) Z_1$, and $\rho_t = t$. Also, $D_{\sigma^*} = (4 \sigma_{min}^4  + \omega^4)^{1/2} I  $ and  $C_{\sigma^*} = C_{\sigma^*}^T = \frac{1}{2}( (4 \sigma_{min}^4  + \omega^4)^{1/2} - \omega^2) I$.

Thus, 
\begin{equation}
\mu_t = (1-t) Z_0 + t Z_1 + (a_t-a_0) Z_0 + (b_t - b_0) Z_1 - t [(a_1-a_0) Z_0 + (b_1 - b_0) Z_1]    
\end{equation}
and 
\begin{align}
    \Sigma_t &= [((1-t)^2 + t^2) \sigma_{min}^2  + t(1-t) ((4 \sigma_{min}^4  + \omega^4)^{1/2} - \omega^2) + \omega^2 t (1-t) ] I \\
    &= \sigma_{min}^2 + \left(\sqrt{4 \sigma_{min}^4 + \omega^4} - 2 \sigma_{min}^2 \right) \cdot t(1-t).
\end{align}
Theorem \ref{thm_optimal_path} then follows from Theorem \ref{app_thmSB} with the above formula for $\mu_t$ and $\Sigma_t$. 
\end{proof}

\subsection{Comparing the Variance of Different Vector Field Models} \label{app_explain}

We begin with providing intuitions for understanding our proposed probability path model. We expect that our model improves upon the other considered models, as it takes advantage of the inherent continuity and correlation in the spatio-temporal data. 
Intuitively, for time series samples whose underlying dynamics are continuous and obey a physical law, a Gaussian sample is typically further from the time series samples, so the distance between a Gaussian sample and a time series sample should generally be larger than the distance between consecutive time series samples (which can be highly correlated). 
Therefore, using a probability path that connects consecutive time series samples could lead to faster convergence and more stable training, when compared to using a path that simply connects a time series sample to a Gaussian sample, since the resulting flow model in the former case can better capture the true system dynamics  with less effort. 

Moreover, if the consecutive samples are sufficiently correlated, then the variance of the VF corresponding to our proposed probability path model can be lower than the variance of the VF corresponding to the other choices of probability paths. Precise statements capturing this are the contents of  Theorem \ref{thm_variance}, which focuses on comparison between our proposed model and the optimal transport VF model of \citep{lipman2022flow}. 


Before proving Theorem \ref{thm_variance}, we start with the following lemma.

\begin{lemma} \label{lemma_1}
    Let $A, B, C, D$ be random vectors  where $C$ and $D$ are independent, both $A$ and $B$ are independent of $C$ and $D$ (but $A$ and $B$ could be dependent). If $Cov(A, B) \geq (Var(C) - Var(D) + Var(B))/2$, then 
    \begin{equation}
        Var(A + D) \geq Var(A-B+C) = Var(B-A + C),
    \end{equation}
where $A \geq B$ means that $A-B$ is positive semidefinite. 
\end{lemma}

\begin{proof}
We compute:
\begin{align}
    Var(A + D) &= Var(A - B + C + B + D - C) \\ 
    &= Var(A-B+C) + Var(B + D) + Var(C) + 2 Cov(A-B+C, B+D) \nonumber \\ 
    &\ \ \ \ - 2 Cov(A-B+C, C)  - 2 Cov(B+D, C) \\
    &= Var(A-B+C) + Var(B + D) - Var(C) + 2 Cov(A-B + C, B) \nonumber \\
    &\ \ \ \ + 2 Cov(A-B + C, D) \\ 
    &= Var(A-B+C) + Var(B+D) - Var(C) + 2 Cov(A, B) - 2 Var(B) \\ 
    &= Var(A-B+C) - Var(C) + Var(D) + 2 Cov(A, B) - Var(B),
\end{align}
where we have simply rearranged the terms in the first equality, used the formula $Var(A+B+C) = Var(A) + Var(B) + Var(C) + 2Cov(A,B) + 2 Cov(A,C) + 2 Cov(B,C)$, bilinearity of covariance, the facts that $Cov(A,A) = Var(A)$ and $Var(cA) = c^2 A$ for a scalar $c$, as well as the assumption that both  $A$ and $B$ are independent of $C$, $D$  in the last four equalities. 

Therefore, if $- Var(C) + Var(D) + 2 Cov(A, B) - Var(B) \geq 0$, then we have $Var(A + D) \geq Var(A-B+C)$.

\end{proof}



    

Now we prove Theorem \ref{thm_variance}.
\begin{proof}
    Note that $c_t' = \frac{\sigma^2 (1-2t)}{2 \sqrt{\sigma_{min}^2 + \sigma^2 t(1-t)}}$, $Var(c_t' \xi) = (c_t')^2 I$ and $Var(-\eta) = I$. Therefore, using these and applying Lemma \ref{lemma_1} with $A := z^{\tau-1}$, $B := z^\tau$,  $C := c_t' \xi$ and $D := - \eta$, allow us to establish the claim that $Var(\tilde{u}_t(Z_t|z^{\tau-1})) \geq Var(u_t(Z_t|z^{\tau-1}, z^\tau))$ if
    \begin{equation}
        Cov(z^{\tau-1}, z^\tau) \geq \frac{1}{2} \left( \frac{\sigma^4(1-2t)^2 I}{4 (\sigma_{min}^2 + \sigma^2 t (1-t))}  + Var(z^\tau) - I \right).
    \end{equation}
    Since  the function $f(t) 
 := \frac{\sigma^4(1-2t)^2}{4 (\sigma_{min}^2 + \sigma^2 t (1-t))} $ is maximized at the endpoints $t=0,1$ with the maximum value of $\sigma^4/4\sigma_{min}^2$, the desired result stated in the theorem follows.
\end{proof}


Lastly, we provide some discussions following Theorem \ref{thm_variance} (using the notations there).

\noindent {\bf Discussions.}
Let us denote $v_t^\theta(Z) := v^\theta_t(z^\tau, z^{\tau-1}, \xi)$ to show the explicit dependence of the  vector field neural net on the random samples $t, z^\tau$, $z^{\tau-1}$ and $\xi \sim \mathcal{N}(0,I)$ drawn during each update of gradient descent during training. During each gradient descent update, our model involves computation of  
\begin{align}
\nabla_\theta \mathcal{L}_{cfm}(\theta; t, \xi, z^\tau, z^{\tau-1}) &= 2 \nabla_\theta v_t^\theta(z^\tau, z^{\tau-1}, \xi)^T \cdot (v_t^\theta(z^\tau, z^{\tau-1}, \xi) - u_t(Z_t| z^\tau, z^{\tau-1} ) )  \label{eq_our_grad} \\
&=: 2 \nabla_\theta v_t^\theta(z^\tau, z^{\tau-1}, \xi)^T \cdot \Delta_t^\theta(  z^{\tau-1}, z^\tau, \xi), 
\end{align}
with $t \sim \mathcal{U}[0,1]$, $\xi \sim \mathcal{N}(0,I)$ and the latent samples $z^\tau, z^{\tau-1}$ drawn randomly.

Similarly, for the rectified flow model let us denote $\tilde{v}_t^\theta(Z) := v^\theta_t(z^{\tau-1}, \eta)$ to show the explicit dependence of the  vector field neural net on the random samples $t$, $z^{\tau-1}$ and $\eta \sim \mathcal{N}(0,I)$ drawn during each update of gradient descent during training. 
Each update of gradient descent using the rectified flow model involves computation of
\begin{align}
\nabla_\theta \tilde{\mathcal{L}}_{cfm}(\theta; t, \eta, z^{\tau-1}) &= 2 \nabla_\theta \tilde{v}_t^\theta( z^{\tau-1}, \eta  )^T \cdot (\tilde{v}_t^\theta( z^{\tau-1}, \eta ) - \tilde{u}_t(\tilde{Z}_t| z^{\tau-1} ) )  \label{eq_river_grad} \\ 
&=: 2 \nabla_\theta \tilde{v}_t^\theta( z^{\tau-1}, \eta )^T \cdot \tilde{\Delta}_t^\theta( z^{\tau-1}, \eta),    
\end{align}
with $t \sim \mathcal{U}[0,1]$, $\eta \sim \mathcal{N}(0,I)$ and the latent sample $z^{\tau-1}$ drawn randomly.

Lower gradient variance results in smoother training loss curve and potentially faster convergence, so it is useful to compare the variances of the loss gradient for the two models.
However, the variances are highly dependent on $\nabla_\theta v_t^\theta$, $\nabla_\theta \tilde{v}_t^\theta$ and their covariance with the other random vectors appearing in Eq. (\ref{eq_our_grad}) and Eq. (\ref{eq_river_grad}), making such comparison challenging without strong assumptions. Heuristically, the difference in the variances of the loss gradient during each update for the two models is primarily  determined by the difference between  $Var(\Delta_t^\theta(z^{\tau-1}, z^\tau, \xi))$ and $Var(\tilde{\Delta}_t^\theta( z^{\tau-1}, \eta ))$ if $\nabla_\theta v_t^\theta$ and $\nabla_\theta \tilde{v}_t^\theta$ are relatively stable. In this case, we have $Var(\Delta_t^\theta(z^{\tau-1}, z^\tau, \xi)) \leq Var(\tilde{\Delta}_t^\theta( z^{\tau-1}, \eta ))$ if we suppose the assumptions in Theorem \ref{thm_variance}, $Var(\tilde{v}_t^\theta) \geq Var(v_t^\theta)$ and $Cov(v_t^\theta, u_t) \geq Cov(\tilde{v}_t^\theta, \tilde{u}_t)$.

The implications of Theorem \ref{thm_variance}  together with the heuristics above could partially explain why our probability path model leads to smoother loss curve and faster convergence (see Figure \ref{fig:convergence}) compared to the RIVER method of \citep{davtyan2023efficient}. 
On the other hand,  the dependence of the lower bound in the theorem on $\sigma$ and $\sigma_{min}$ suggests that using values of $\sigma$ that is relatively large enough might not keep the variance of the vector field  low, which could partially explain the phenomenon displayed in Figure \ref{fig:convergence}, where using $\sigma = 0.1$ and $\sigma_{min} = 0.001$ leads to large loss fluctuations.

\section{Additional Empirical Results} \label{app_additional}
In this section, we provide additional experimental results.

\subsection{Visualization of Flow Patterns and Dynamics} \label{app_visual}

Figure \ref{fig:result_visual} provides visual results of the predicted snapshots by our model for each task. The visualizations highlight our model's ability to capture complex flow patterns and dynamics.

\begin{figure}[!h]
\centering
\vspace{+1.0cm}

\begin{overpic}[width=0.95\textwidth]{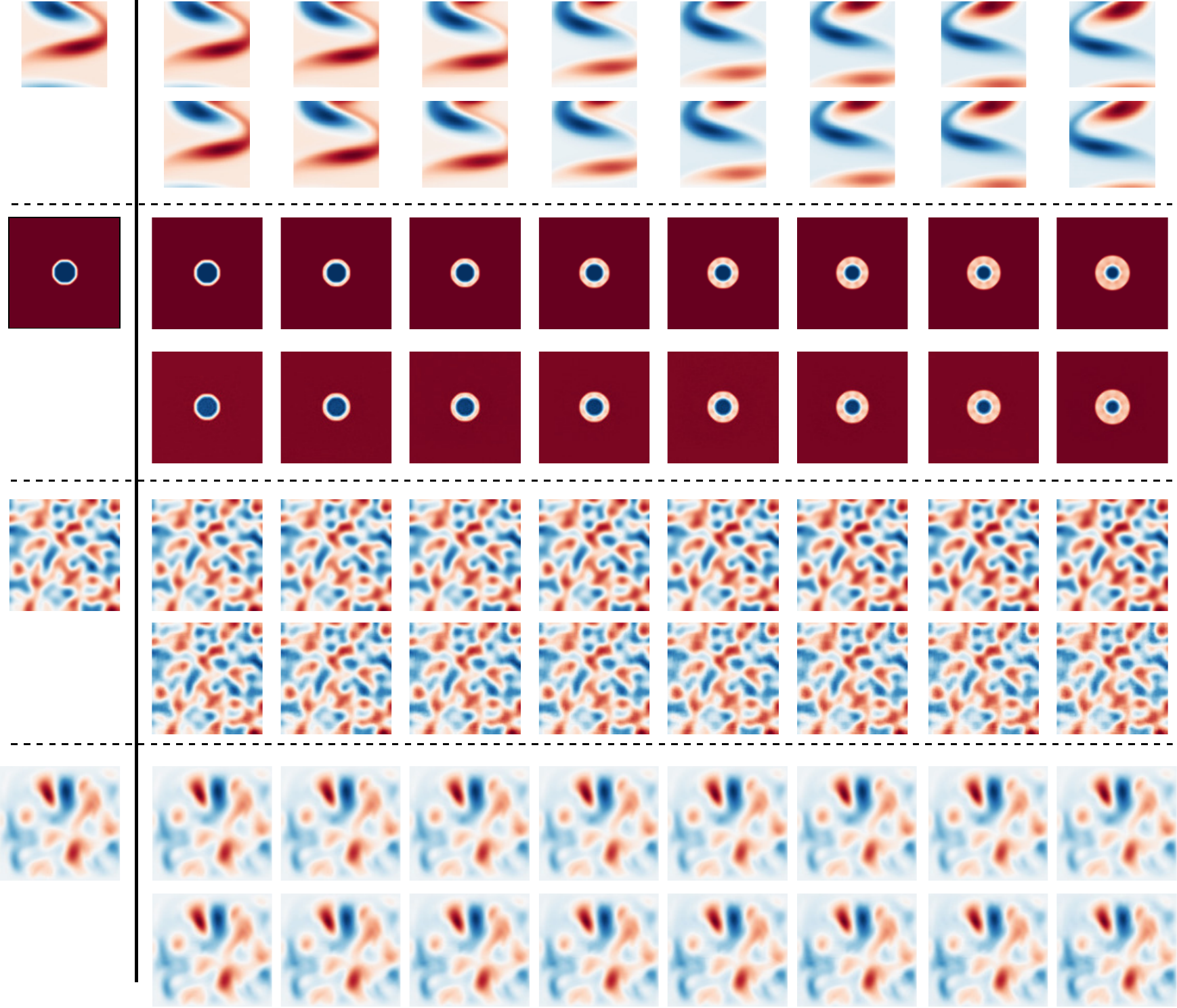}
\put(-2,88){\footnotesize last con. frame}
\put(30,88){\footnotesize time $\rightarrow$}

\put(-2,80){\rotatebox{90}{\footnotesize GT}}
\put(8,72){\rotatebox{90}{\footnotesize pred.}}

\put(-2,61){\rotatebox{90}{\footnotesize GT}}
\put(8,49){\rotatebox{90}{\footnotesize pred.}}

\put(-2,37){\rotatebox{90}{\footnotesize GT}}
\put(8,27){\rotatebox{90}{\footnotesize pred.}}

\put(-2,14){\rotatebox{90}{\footnotesize GT}}
\put(8,3){\rotatebox{90}{\footnotesize pred.}}

\end{overpic}
\caption{Visualization of predicted frames using our model of probability path for the considered tasks. From top to bottom: fluid flow past a cylinder, shallow-water equation, diffusion-reaction equation, and incompressible Navier-Stokes equation. In each case, GT indicates the ground truth frames and pred. indicates the predicted frames.}
\label{fig:result_visual}
\end{figure}

\subsection{Ablation Studies}  \label{app_ablation}

To further assess our model, we conducted a detailed ablation study focusing on the impact of various hyperparameters for the fluid flow (FPC) task. Specifically, we study the impact of the values of $\sigma$, the choice of sampler, and the number of sampling steps during inference. For small $\sigma_{sam}$, we find that this parameter has negligible impact on test performance, so we fixed $\sigma_{\text{sam}} = 0$ for all experiments in this section. 

\textbf{Impact of $\sigma$ on training stability.} Figure \ref{loss_fluid_comparesigma} illustrates the effect of different $\sigma$ values on the training loss curve for our method on the fluid flow past a cylinder task. We observed that larger values of $\sigma$ (e.g., $\sigma = 0.1$) resulted in smoother loss curves and more stable convergence during training.

\textbf{Effect of $\sigma$, sampler choice, and sampling steps on test performance.} Table \ref{tab:simpleflow_ablation} investigates how different values of $\sigma$, the choice of sampler (Euler or RK4), and the number of sampling steps affect test performance. It can be seen, that even with as few as one sampling step $(N=2)$ using the Euler scheme, our model perform reasonably well. However, increasing the number of sampling steps or employing the more computationally intensive RK4 sampler can help to lead to better results. From Figure \ref{fig:mse_samplingsteps}, we see that our model leads to the smallest test MSE for both samplers at all sampling steps. In particular, lowest MSEs can be achieved using as few as one sampling step for both sampling schemes. Moreover, while using the RK4 sampler can lead to a lower test MSE for all models, the performance gap is much smaller for our model, showing robustness of our model to sampling choice.  

\begin{figure}[!h] 
\centering
    \begin{overpic}[width = 0.9\textwidth]{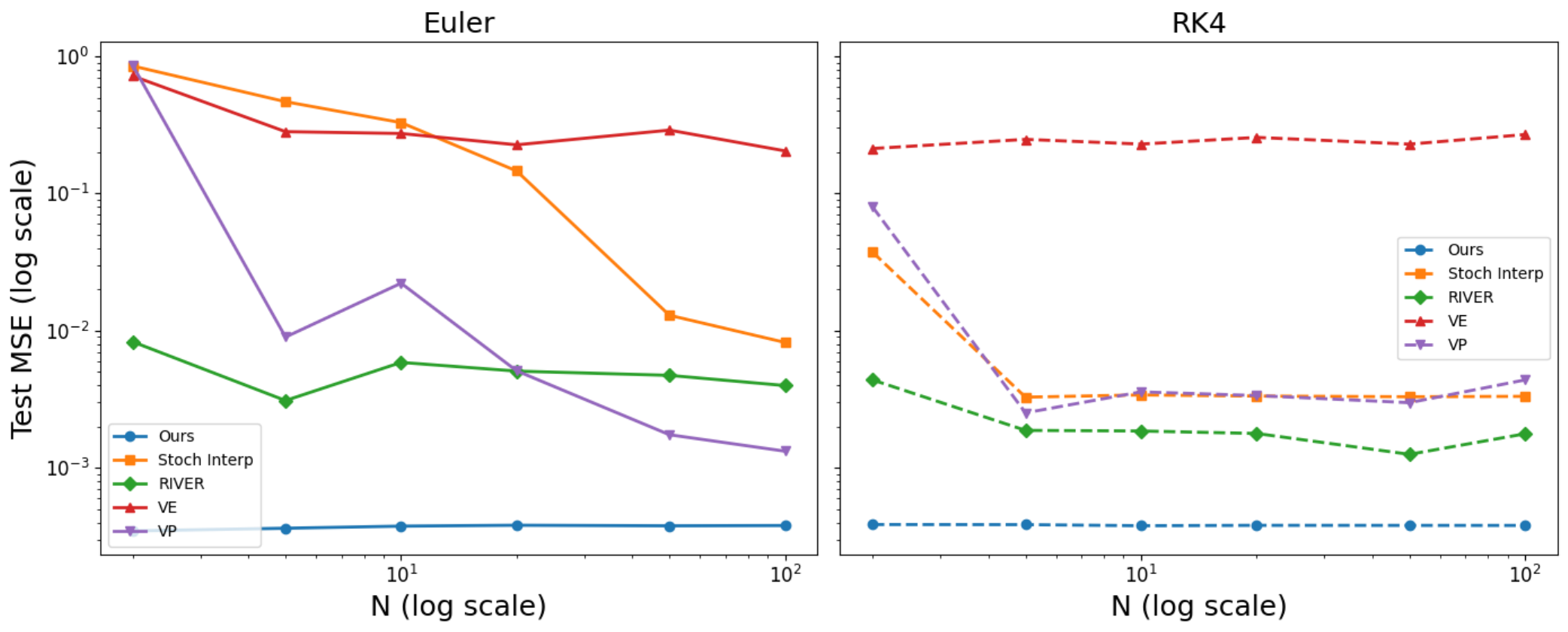}
    \end{overpic}
\caption{Test MSE vs. number of sampling steps $N-1$ (we consider $N=2,5,10,20,50,100$) for the five models with the Euler sampler (left) and the RK4 sampler (right).}
\label{fig:mse_samplingsteps}
\end{figure}

\textbf{Effect of context frames on test performance.} Table \ref{tab:ablation_contextframes} shows the test performance in terms of the considered evaluation metrics in the case when the neural networks are trained without using the context frames, in which case $C$ is null in  (\ref{conditional_opt_practice}) and we are minimizing the loss in (\ref{conditional_opt}).  We see that without using the proposed context frames, the test performance of all models degrades across the evaluation metrics, showing the benefits of using the additional context information. The degradation is particularly significant for RIVER, VE-diffusion and VP-diffusion.
\begin{table}[!h]
	\caption{Ablation study to assess the impact of context frames on test performance of the considered models, given that the same pre-trained autoencoder is used. Results are averaged over 5 generations.}
	\label{tab:ablation_contextframes}
	\centering
	\scalebox{0.9}{
		\begin{tabular}{l c  c ccccccc}
			\toprule
			Model &  Test MSE ($\downarrow$) & Test RFNE ($\downarrow$) & PSNR ($\uparrow$) & SSIM ($\uparrow$) \\	
			\midrule 
			RIVER & 1.21  & 1.13 & 20.84  &  0.25  \\
			VE-diffusion  & 1.01  & 1.07  & 23.21 & 0.23   \\
			VP-diffusion & 1.04    & 1.06  & 21.18  & 0.24   \\
   			Stochastic interpolant  &  5.37e-03 & 7.04e-02 & 41.36 & 0.97   \\
      \midrule 
      Ours ($\sigma  = 0.01$, $\sigma_{sam} = 0$)   & {\bf 8.42e-04} & {\bf 3.18e-02}  & {\bf 46.75} & {\bf 0.99}    \\
			\bottomrule
	\end{tabular}}
\end{table}

We further assess the impact of the random conditioning frames $z^c$, considering the case when the neural networks are trained with the reference frames but without using random conditioning frames, in which case $C = z^{\tau-1}$ in  (\ref{conditional_opt_practice}).  Table \ref{tab:ablation_randomframes} shows the results. We see that in this case RIVER and  VE-diffusion achieve lower test performance across all evaluation metrics, which shows that these models indeed benefit from the use of conditional frames. Interestingly, our model, the stochastic interpolant model and the VP-diffusion model remain relatively robust, with comparable or even improved test performance. Despite this, our model outperforms the other models across all evaluation metrics regardless of whether the conditioning frames are involved. For consistency and fair comparison with RIVER, we retain the conditioning scheme by default. 
\begin{table}[!h]
	\caption{Ablation study to assess the impact of the random conditioning frames on test performance of the considered models, given that the same pre-trained autoencoder is used. Results are averaged over 5 generations.}
	\label{tab:ablation_randomframes}
	\centering
	\scalebox{0.9}{
		\begin{tabular}{l c  c ccccccc}
			\toprule
			Model &  Test MSE ($\downarrow$) & Test RFNE ($\downarrow$) & PSNR ($\uparrow$) & SSIM ($\uparrow$) \\	
			\midrule 			RIVER & 4.68e-03  & 6.19e-02 & 43.26  &  0.97  \\
            VE-diffusion  &  3.96e-01 & 6.04e-01  & 26.40 & 0.53   \\
			VP-diffusion & 1.25e-03   & 3.67e-02  & 45.73 &  0.99  \\
   			Stochastic interpolant  &  2.33e-03 & 4.90e-02 & 43.99 & 0.98   \\
      \midrule 
      Ours ($\sigma  = 0.01$, $\sigma_{sam} = 0$)   & {\bf 3.21e-04} & {\bf 2.16e-02}  & {\bf 49.23} & {\bf 1.00}    \\
			\bottomrule
	\end{tabular}}
\end{table}

\textbf{Sensitivity to $\sigma_{min}$.} Table \ref{tab:sigma_min_analysis} shows that using $\sigma_{min} > 0$ (while fixing the other hyperparameters) leads to noticeable improvement in test performance of our model across the evaluation metrics. This is in line with the observation that using $\sigma_{min} > 0$ improves training stability, since this alleviates the singularity of the target VF (\ref{eq_targetVF}) at $t=0, 1$. In fact,  we see that using $\sigma_{min} > 0.001$ (the default value that we use) leads to better test results for the FPC task. 
\begin{table}[!h]
	\caption{Sensitivity analysis of our model to $\sigma_{min}$. Results are averaged over 5 generations.}
	\label{tab:sigma_min_analysis}
	\centering
	\scalebox{0.9}{
		\begin{tabular}{l c  c ccccccc}
			\toprule
			$\sigma_{min}$ &  Test MSE ($\downarrow$) &  Test RFNE ($\downarrow$) &  PSNR ($\uparrow$) & SSIM ($\uparrow$) \\	
			\midrule 
			0.0 & 1.98e-03 & 4.41e-02 & 44.74 & 0.99 \\
			0.001 & 3.79e-04 & 2.30e-02 & 48.88 & 1.00 \\
   			0.01  & 3.78e-04 & 2.26e-02 & 49.11 & 1.00 \\
            	0.1 &  3.17e-04 & 2.12e-02 & 49.52 & 1.00 \\ 
                1.0 & 3.46e-04 & 2.17e-02 & 49.44 & 1.00 \\
			\bottomrule
	\end{tabular}}
\end{table} 

\begin{figure}[!t] 
\vspace{0.5cm}
\begin{minipage}{0.46\textwidth}
    \begin{overpic}[width = 1.1\textwidth]{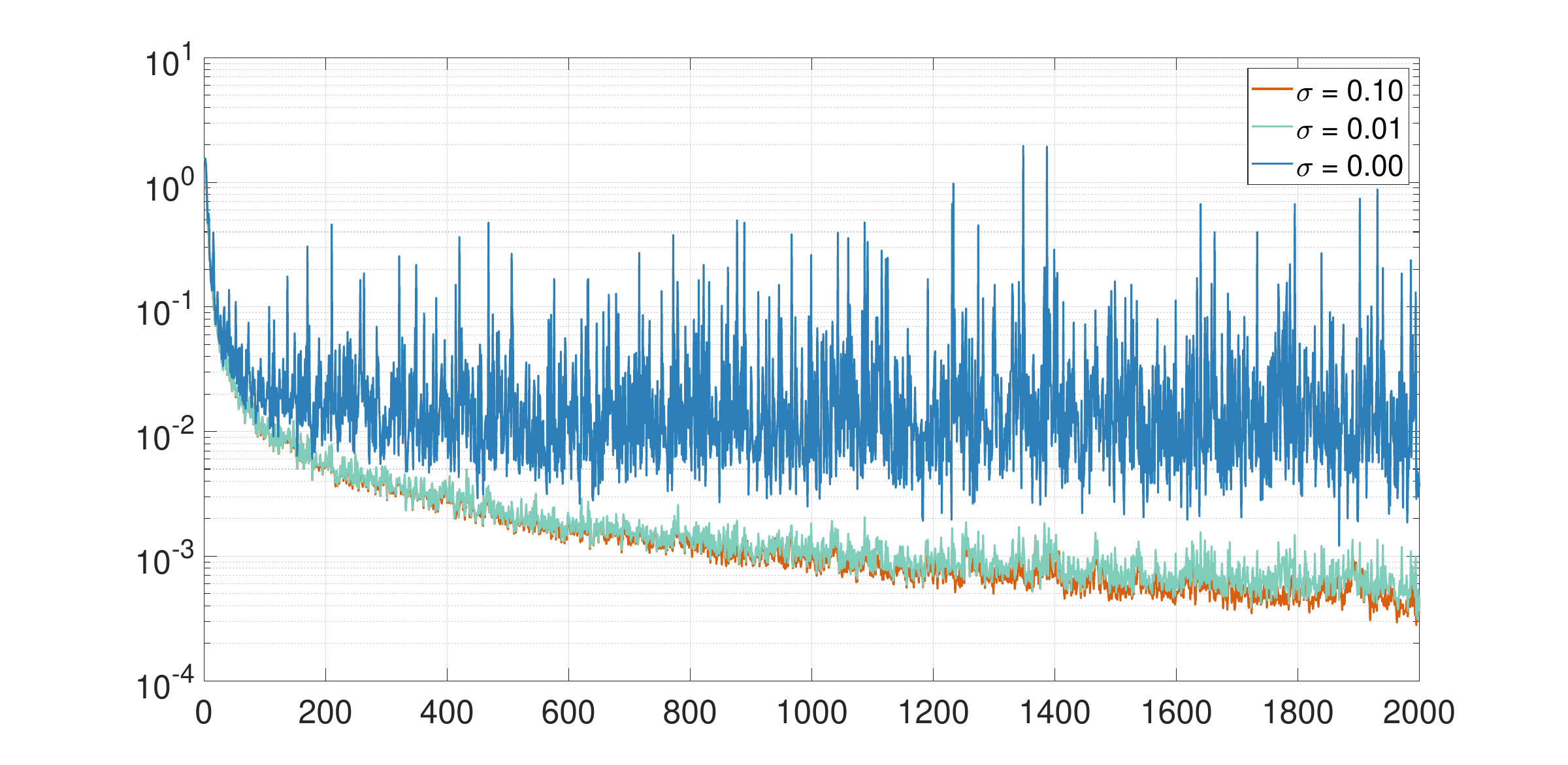}
    \put(3,21){\rotatebox{90}{Loss}}
    \put(45,-3){Epoch}
    \put(30,48){\textbf{Euler Discretization}}
    \end{overpic}
\end{minipage}
\hspace{0.5cm}
\begin{minipage}{0.46\textwidth}
    \begin{overpic}[width = 1.1\textwidth]{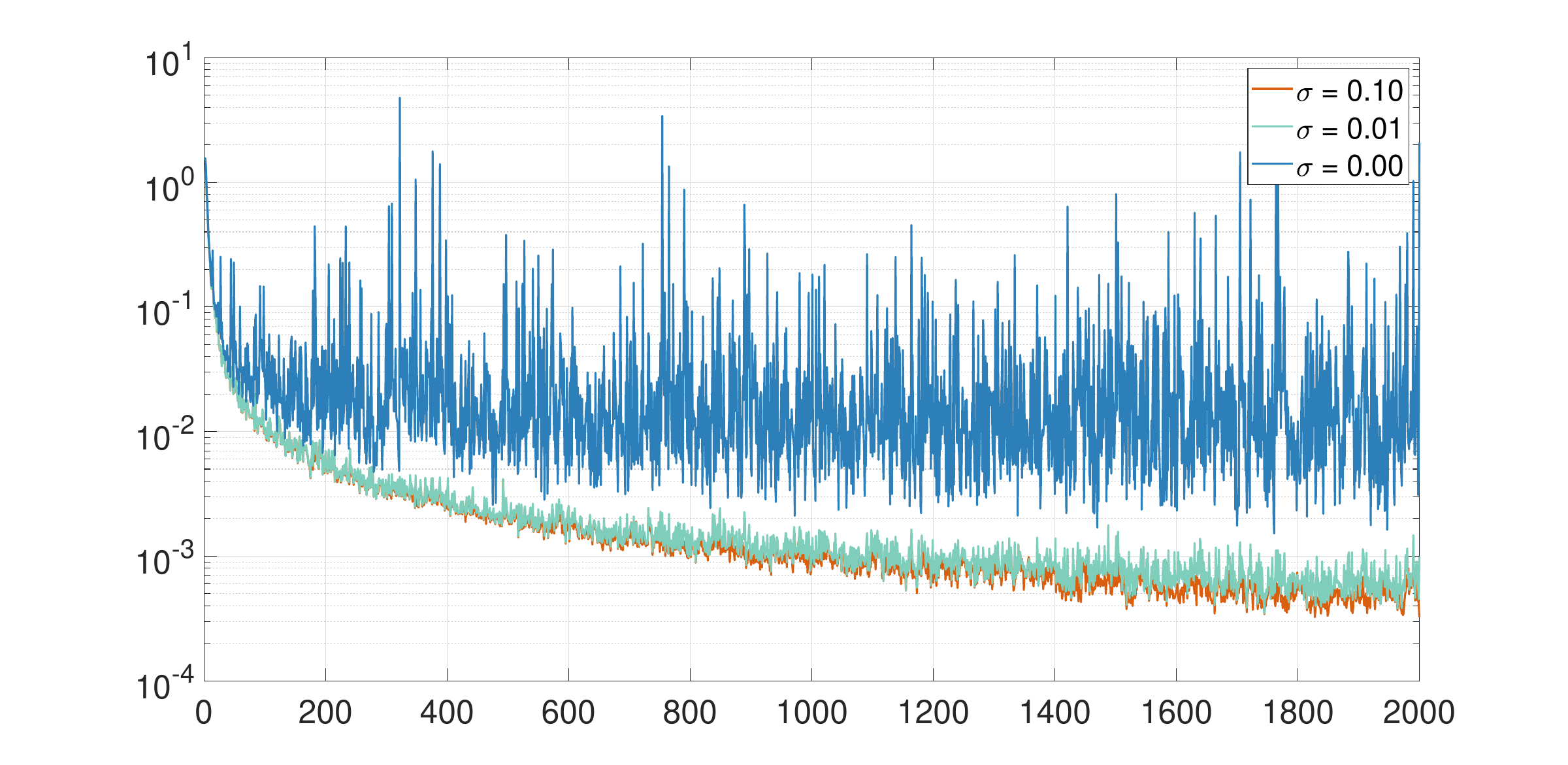}
    \put(3,21){\rotatebox{90}{Loss}}
    \put(45,-3){Epoch}
    \put(31,48){\textbf{RK4 Discretization}}
    \end{overpic}
\end{minipage}
\vspace{+0.25cm}
\caption{Training loss for different values of $\sigma$ using our probability path model. The left subplot shows results for the Euler sampler, and the right subplot for the RK4 sampler. We see that the loss curve is sensitive to the choice of $\sigma$, with larger values of $\sigma$ giving smoother loss curves.}
\label{loss_fluid_comparesigma}
\end{figure}

\begin{table}[!h]
\vspace{0.3cm}
	\caption{Ablation study for the fluid flow past a cylinder task. Results are averaged over 5 generations.}
	\label{tab:simpleflow_ablation}
	\centering
	\scalebox{0.85}{
		\begin{tabular}{l c c | c  c ccccccc}
			\toprule
			$
   \sigma$ & sampler & $N$ & Test MSE ($\downarrow$) & Test RFNE ($\downarrow$) & PSNR ($\uparrow$) & SSIM ($\uparrow$) \\		       \midrule 
    0.0 & Euler & 5 &   7.02e-04 & 2.96e-02  & 47.15 & 0.99  \\
			0.01  & Euler & 5 & 3.63e-04 & 4.14e-03  & 48.89 & 1.00 \\
			0.1  & Euler & 5 & 2.97e-03  & 4.11e-02 & 43.99  & 0.98 \\
			\midrule 
			 0.0 & Euler & 10 & 7.31e-04 & 3.01e-02 & 47.02  & 0.99 \\
			0.01  & Euler & 10 & 3.76e-04 & 2.31e-02 & 48.76 & 1.00  \\
			0.1  & Euler & 10 & 2.85e-03  & 5.04e-02 & 44.12 & 0.98 \\
   			\midrule 
			0.0 & RK4 & 10 & 3.89e-04 & 2.26e-02 & 49.19 & 1.00  \\
			0.01  & RK4 & 10 & 3.79e-04  & 2.30e-02 & 48.88  & 1.00   \\
			0.1 & RK4 & 10 & 6.56e-03  & 7.83e-02 &  40.68 & 0.97  \\
   			\midrule 
                0.0 & Euler & 20 & 7.40e-04  &  3.03e-02 & 46.98  &  0.99 \\
			0.01  & Euler &  20  & 3.82e-04  & 2.33e-02 & 48.71 & 1.00 \\
			0.1  & Euler & 20 & 2.78e-03  & 4.98e-02  & 44.19 & 0.98  \\
   			\midrule 
			0.0 & RK4 & 20 & 3.88e-04  & 2.26e-02  &  49.19 & 1.00
  \\
			0.01  & RK4 & 20 & 6.53e-04  & 2.80e-02 &  47.63 & 0.99 \\
			0.1  & RK4 & 20 & 6.51e-04 & 7.80e-02  & 40.70 & 0.97 \\
   			\bottomrule \\
	\end{tabular}}
\end{table}

\section{Experimental Details} \label{app_exp}

In this section, we provide the experimental details for the tasks considered in Section \ref{sect_exp}.

\subsection{On the Choice of Gaussian Reference Processes}
Our choice of a Gaussian reference process is not directly derived from the characteristics of a particular spatio-temporal dataset, but rather follows a modeling convention that is widely adopted, particularly when physical or statistical smoothness assumptions are in play.
Gaussian processes are a standard modeling choice in these domains for several reasons: 

\begin{itemize}
    \item They arise as solutions to linear stochastic partial differential equations (e.g., the heat equation with additive white noise), which are common in spatio-temporal physical systems.
    \item They offer desirable analytical properties such as closed-form marginals, smooth sample paths, and tractable likelihoods.
    \item In many practical settings (e.g., climate modeling, diffusion, spatial statistics), Gaussian priors are used because they act as regularizing priors on functions or fields that evolve over space and time.
\end{itemize}

\subsection{Details on the Datasets} \label{app_datasets}
\noindent {\bf Fluid flow past a cylinder (FPC).} 
We use the fluid flow past a stationary cylinder at a Reynolds number of 100 as a simple test problem. This fluid flow is a canonical problem in fluid dynamics characterized by a periodically shedding wake structure~\citep{erichson2020shallow,erichson2019randomized}. The flow dynamics are governed by the two-dimensional incompressible Navier–Stokes equations:

\[
\frac{\partial \mathbf{u}}{\partial t} + (\mathbf{u} \cdot \nabla)\mathbf{u} = -\frac{1}{\rho} \nabla p + \nu \nabla^2 \mathbf{u},
\]

\[
\nabla \cdot \mathbf{u} = 0,
\]

where \(\mathbf{u} = (u, v)\) is the velocity field, \(p\) is the pressure, \(\rho\) is the fluid density, and \(\nu\) is the kinematic viscosity. The vorticity field \(\omega\) is obtained from the velocity field via:
\[
\omega = \nabla \times \mathbf{u},
\]
providing insights into the rotational characteristics of the flow.

For simulating the data, the Immersed Boundary Projection Method (IBPM) has been used~\citep{colonius2008fast}. The flow tensor has dimensions $199 \times 449 \times 151$, representing 151 temporal snapshots on a $449 \times 199$ spatial grid. We crop and spatially subsample the data which results in a $64\times64$ spatial field.

\noindent {\bf Shallow-water equation (SWE).}
The shallow-water equations, derived from the compressible Navier-Stokes equations, can be used for modeling free-surface flow problems. We consider the 2D equation used in \citep{takamoto2022pdebench}, which is the following  system of hyperbolic PDEs:
\begin{equation}
    \partial_t h + \nabla h {\bf u} = 0, \ \  \partial_t h {\bf u} + \nabla \left({\bf u}^2 h + \frac{1}{2} g_r h^2 \right) = - g_r h \nabla b,
\end{equation}
where ${\bf u} = u, v$ being the velocities in the horizontal and vertical direction respectively, $h$ describes the water depth, and $b$ describes a spatially varying bathymetry. $h {\bf u}$ can be interpreted as the directional
momentum components and $g_r$ describes the gravitational acceleration. The mass and momentum conservation properties can hold across shocks in the solution and thus challenging datasets can be
generated. This equation finds application in modeling tsunamis and flooding events.

We use the dataset generated and provided by PDEBench \citep{takamoto2022pdebench}. The data file ({\url{2D_rdb_NA_NA.h5}}) can be downloaded from \url{https://github.com/pdebench/PDEBench/tree/main/pdebench/data_download}. The data sample is a series of 101 frames at a 128 $\times$ 128 pixel resolution and come with 1 channel. 
The simulation considered in \citep{takamoto2022pdebench} is a 2D radial dam break scenario. On a square domain $\Omega = [-2.5, 2.5]^2$, the
water height is initialized as a circular bump in the center of the domain:
\[
h(t=0, x, y) = \begin{cases}
2, \text{ for } r < \sqrt{x^2 + y^2}, \\
1, \text{ for } r \geq  \sqrt{x^2 + y^2},
\end{cases}
\]
with the radius $r$ randomly sampled from $\mathcal{U}(0.3, 0.7)$. The dataset is simulated with a finite volume solver using the {\it PyClaw} package. 
We apply standardization and then normalization to the range of $[-1,1]$ to preprocess the simulated data.

\noindent {\bf Incompressible Navier-Stokes equation (NSE).}
The Navier-Stokes equation is the incompressible version of the compressible Navier-Stokes equation, and it can be used to
model hydromechanical systems,
turbulent dynamics and weather. We use the inhomogeenous version of the equation (which includes a vector field forcing term ${\bf u}$) considered by \citep{takamoto2022pdebench}: 
\begin{equation}
    \nabla \cdot {\bf v} = 0, \ \ \
    \rho (\partial_t {\bf v} + {\bf v} \cdot \nabla {\bf v} ) = - \nabla p + \eta \Delta {\bf v} + {\bf u},
\end{equation}
where $\rho$ is the mass density, ${\bf v}$ is the fluid velocity, $p$ is the gas pressure and $\eta$ is shear viscosity. 
The  initial conditions ${\bf v}_0$ and inhomogeneous forcing parameters ${\bf u}$ are each drawn from isotropic Gaussian random fields with truncated power-law decay $\tau$ of the power spectral density and scale $\sigma$, where
$\tau_{v_0} = -3$, $\sigma_{v_0} = 0.15$, $\tau_u = -1$, $\sigma_u = 0.4$.  The domain is taken to be the unit square $\Omega  = [0, 1]^2$ and the viscosity $\eta = 0.01$.
The equation is numerically simulated using {\it Phiflow}. Boundary conditions are taken to be Dirichlet to clamp the field
velocity to zero at the perimeter.

We use the dataset generated and provided by PDEBench \citep{takamoto2022pdebench}. The data file (\url{ns_incom_inhom_2d_512-0.h5}) can be downloaded from 
{\url{https://github.com/pdebench/PDEBench/tree/main/pdebench/data_download}}. The data sample is a series of 1000 frames at a 512 $\times$ 512 pixel resolution and come with 2 channels.
We do not apply any data preprocessing procedure here.

\noindent {\bf Diffusion-reaction equation (DRE).}
We use the 2D extension of diffusion-reaction equation of \citep{takamoto2022pdebench} which describes  two non-linearly coupled variables, namely
the activator $u = u(t, x, y)$ and the inhibitor $v = v(t, x, y)$. The equation is given by:
\begin{align}
    \partial_t u &= D_u \partial_{xx} u + D u \partial_{yy} u + R_u, \\
    \partial_t v &= D_v \partial_{xx} v + D_v \partial_{yy} v + R_v, 
\end{align}
where $D_u$ and $D_v$ are the diffusion coefficient for the activator and inhibitor respectively, $R_u =
R_u(u, v)$ and $R_v = R_v(u, v)$ are the activator and inhibitor reaction function respectively. The
domain of the simulation includes $x \in (-1, 1)$, $y \in (-1, 1)$, $t \in (0, 5]$. This equation can be used for modeling biological pattern formation.

The reaction functions for the activator and inhibitor are defined by the Fitzhugh-Nagumo equation as:
$R_u(u, v) = u - u^3 - k - v$, $R_v(u, v) = u - v$, 
where $k = 5 \times 10^{-3}$, and the diffusion coefficients for the activator and inhibitor are $D_u = 1 \times 10^{-3}$ and $D_v = 5 \times 10^{-3}$ respectively. The initial condition is generated as standard Gaussian
noise $u(0, x, y) \sim \mathcal{N}(0, 1.0)$ for $x \in (-1, 1)$ and $y \in (-1, 1)$. We take a no-flow Neumann boundary condition: $D_u \partial_x u = 0$, $D_v \partial_x v = 0$, $D_u \partial_y u = 0$, and $D_v \partial_y v = 0$ for $x, y \in (-1, 1)^2$.

We use a downsampled version of the dataset provided by PDEBench \citep{takamoto2022pdebench}. The data file ({\url{2D_diff-react_NA_NA.h5}) can be downloaded from {\url{https://github.com/pdebench/PDEBench/tree/main/pdebench/data_download}}. The data sample is a series of 101 frames at a 128 $\times$ 128 pixel resolution and come with 2 channels. The sample frames are generated using the finite volume method for spatial discretization, and the time integration is performed using the built-in fourth order Runge-Kutta method in the {\it scipy} package. 
We do not apply any data preprocessing procedure here.

\subsection{Details on Pre-Training the Autoencoder}

We provide details on pre-training the autoencoder here. 
The choice of first pre-training an autoencoder is motivated by the computational challenges of working directly with the high-dimensional spatial resolution of PDE datasets. Training directly in the ambient space requires substantial GPU memory and computational resources, making it impractical for large-scale or high-resolution datasets. By leveraging a latent-space representation, we achieve significant dimensionality reduction while preserving the essential structure of the data, enabling efficient training and inference with standard hardware configurations. For these datasets, latent-space modeling provides a critical balance between computational efficiency, scalability, and performance. 

We use the same architecture for the encoder and decoder for all the tasks, with the architecture parameters chosen based on the complexity of the task.

\noindent {\bf The encoder.} The encoder first applies a 2D convolution (\verb|conv_in|) to the input frame, which reduces the number of channels from \verb|in_channels| to \verb|mid_channels|, and processes the spatial dimensions. Then, a series of ResidualBlock layers, which progressively process and downsample the feature map, making it smaller in spatial dimensions but more enriched in terms of features, are applied. After the residual blocks, the feature map undergoes an attention process via a multi-head attention layer. This layer helps the encoder focus on important parts of the input, learning relationships between spatial positions in the image. For the post-attention step, the feature map is further processed by residual blocks and normalized, preparing it for the final convolution. The output of the encoder is obtained by applying a final 2D convolution (\verb|out_conv|), which maps the processed feature map to the desired number of output channels (\verb|out_channels|).

\noindent {\bf The decoder.} The decoder takes the encoded feature map and transforms it back into an output with similar spatial dimensions as the input. Similar to the encoder, the decoder starts with a convolution that adjusts the number of channels from \verb|in_channels| to \verb|mid_channels|. Then, an attention mechanism (similar to the encoder) is applied to focus on important aspects of the encoded features. Next, a series of ResidualBlock layers, combined with UpBlock layers, are used to progressively increase the spatial dimensions of the feature map (upsampling), undoing the compression applied by the encoder.  After the upsampling, the output is normalized and passed through a final convolution (\verb|out_conv|), mapping the internal feature representation to the desired number of output channels (\verb|out_channels|).

Table \ref{tab:arch_params_ae} summarizes the architecture parameters used for the considered tasks.
\begin{table}[!h]
	\caption{Parameters chosen for the encoder (decoder) architecture.} \label{tab:arch_params_ae}
	\centering
	\scalebox{1.0}{
		\begin{tabular}{l c c c c}
			\toprule
			Task & Fluid flow  & Shallow-water eq.   & Navier-Stokes eq. & Diffusion-reaction eq.  \\	
			\midrule 
			\hspace{-0.11cm}\small \texttt{in\_channels} & 1 (1) & 1 (1) & 2 (2) & 2 (2) \\ 
   			\hspace{-0.11cm}\small \texttt{out\_channels} & 1 (1) & 1 (1) & 2 (2) & 2 (2) \\ 
			\hspace{-0.11cm}\small \texttt{mid\_channels} & 64 (128) & 128 (256) & 128 (256) & 128 (256)  \\ 

			\bottomrule
	\end{tabular}}
\end{table}

\noindent {\bf Training details.} 
We train the autoencoder  using AdamW with batch size of 32,  no weight decay and $\beta=(0.9,0.999)$. We use the cosine learning rate scheduler with warmup. For the fluid flow past a cylinder task, we train for 2000 epochs and use learning rate of 0.001. For the Navier-Stokes task, we train for 500 epochs and use learning rate of 0.0001.  For the other two tasks we train for 5000 epochs and use learning rate of 0.0005. 
Our implementation is in PyTorch, and all experiments are run on an NVIDIA A100-SXM4 GPU
with 40 GB VRAM belonging to an internal SLURM cluster.

\noindent {\bf Impact of the autoencoder (AE) size on test performance.} We consider using pre-trained AEs with increasing number of middle channels (\texttt{mid\_channels}) in the encoder and decoder, resulting in AEs with 374,605, 1,485,589,  5,916,709 and 23,615,557 trainable parameters respectively.  Table \ref{tab:AE_ablation} shows the effect of the size of AE on the test performance for all models. We see that there is an optimal size of AE that leads to the best test performance for RIVER, the stochastic interpolant model and our model. For VE-diffusion, the best test performance is achieved with the smallest AE, suggesting that the model might overfit when AE increases in size or suffer from optimization issues in bigger latent spaces. For VP-diffusion, the test results are relatively insensitive to the size of the AE used.
\begin{table}[!h]
\vspace{0.3cm}
	\caption{Ablation study to assess the impact of the size of the pre-trained autoencoder (AE)  on test performance using the fluid flow (FPC) task. The number of channels in the decoder is in parenthesis. Results are averaged over 5 generations.   }
	\label{tab:AE_ablation}
	\centering
	\scalebox{0.85}{
		\begin{tabular}{l  c  | c  c ccccccc}
			\toprule Model &
			 \texttt{mid\_channels}  & Test MSE ($\downarrow$) & Test RFNE ($\downarrow$) & PSNR ($\uparrow$) & SSIM ($\uparrow$) \\		       \midrule 
             RIVER & 16 (32)	 &   2.54e-03  &    5.43e-02 &  42.32 & 0.98 \\
	 & 32 (64)	 & 2.56e-03 & 4.89e-02 & {\bf 44.55} & {\bf 0.99}  \\
		& 64 (128)  &  {\bf 1.86e-03} & {\bf 4.48e-02}  & 44.30 & {\bf 0.99}  \\ 
         & 128 (256) & 2.69e-03 & 5.52e-02 & 42.48  & 0.98      \\
         \midrule 
          VE-diffusion & 16 (32)	 & {\bf 1.91e-01}    & {\bf  3.99e-01} & {\bf 29.45}  & {\bf 0.68} \\
	 & 32 (64)	& 4.28e-01	& 6.19e-01 & 27.20  & 0.49      \\
		& 64 (128)  & 2.29e-01 & 4.74e-01 & 27.36  & 0.55 \\ 
         & 128 (256)   & 5.12e-01 & 7.22e-01 & 26.65 & 0.42  \\
         \midrule 
          VP-diffusion & 16 (32)	 & {\bf 2.02e-03}    &  5.11e-02 & 42.32 & {\bf 0.98} \\
	 & 32 (64)	& 2.84e-03 &  5.11e-02	& {\bf 44.14} & {\bf 0.98}   \\
		& 64 (128)  & 3.58e-03 & 6.09e-02 & 42.37  & {\bf 0.98}  \\ 
         & 128 (256)  & 2.03e-03 & {\bf 4.97e-02} & 42.97 & {\bf 0.98}   \\
         \midrule 
          Stoch. interp. & 16 (32)	 & 1.44e-02    &   1.22e-01 & 36.35 & 0.92 \\
	 & 32 (64)		  &  1.09e-02   & 9.01e-02  & 41.15 & 0.95  \\
		& 64 (128) &  {\bf 3.40e-03} & {\bf  6.10e-02 } & {\bf 41.81} &  {\bf 0.98}  \\ 
         & 128 (256) & 2.41e-02 & 1.36e-01 & 38.02 & 0.90   \\
         \midrule 
          Ours & 16 (32)	 &  2.49e-03   & 5.70e-02  & 41.30  & 0.98 \\
	  ($\sigma = 0.01, \sigma_{sam} = 0$) & 32 (64)  & {\bf \underline{2.87e-04}}    & {\bf \underline{1.93e-02}}  & {\bf \underline{50.62}} & {\bf \underline{1.00}}  \\
		& 64 (128) & 3.79e-04 & 2.30e-02 & 48.88 & {\bf \underline{1.00}}  \\ 
         & 128 (256)    & 3.53e-04  & 2.24e-02 & 48.96 & {\bf \underline{1.00}} \\
   			\bottomrule \\
	\end{tabular}}
\end{table}

\subsection{Details on Training the Flow Matching Models}

\noindent {\bf Architecture.}
The vector field regressor is a transformer-based model designed to process latent vector fields and predict refined outputs with spatial and temporal dependencies. It uses key parameters like \verb|depth| and \verb|mid_depth|, which control the number of transformer encoder layers in the input, middle, and output stages. The \verb|state_size| and \verb|state_res| parameters define the number of channels and spatial resolution of the input data, while \verb|inner_dim| sets the embedding dimension for processing. The model uses learned positional encodings and a time projection to incorporate spatial and temporal context into the input, which can include \verb|input_latents|, \verb|reference_latents|, and \verb|conditioning_latents|. The input is projected into the inner dimension and passed through a series of transformer layers, with intermediate outputs from the input blocks concatenated with the output layers to refine predictions. Finally, the model projects the processed data back to the original spatial resolution and channel size using BatchNorm, producing the final vector field output.

Table \ref{tab:arch_params} summarizes the architecture parameters used for the considered tasks.

\begin{table}[!h]
	\caption{Parameters chosen for the vector field  neural network.} \label{tab:arch_params}
	\label{arch_param}
	\centering
	\scalebox{1.0}{
		\begin{tabular}{l c c c c}
			\toprule
			Parameter & Fluid flow  & Shallow-water eq.   & Navier-Stokes eq. &  Diffusion-reaction eq.  \\	
			\midrule 
			\texttt{state\_size} & 4 &  4 & 8 & 4 \\
   			\texttt{state\_res} & [8,8] &  [16, 16] & [64, 64] & [16, 16] \\
        \texttt{inner\_dim} & 512 &  512 & 512 & 512 \\
            \texttt{depth} & 4 &  4 &  4 & 4 \\
                \texttt{mid\_depth} & 5 & 5 &  5 & 5 \\
			\bottomrule \\
	\end{tabular}}
\end{table}  

\noindent {\bf Training details.}
For all the considered tasks, we train the regressor using AdamW with batch size of 32, learning rate of 0.00005, no weight decay and $\beta=(0.9,0.999)$. We use the cosine learning rate scheduler with warmup. For the fluid flow past cylinder, we train for  2000 epochs,  for the shallow-water equation and diffusion-reaction task we train for 1000 epochs, and for the Navier-Stokes task we train for 100 epochs.  Our implementation is in PyTorch, and all experiments are run on a single NVIDIA A100-SXM4 GPU
with 40 GB VRAM belonging to an internal SLURM cluster. For the Navier-Stokes task, due to GPU memory constraints in the experiments, we use gradient accumulation to simulate the desired batch size while training with a smaller per-step batch.

\subsection{Details on the Evaluation Metrics}
In addition to the standard mean squared error and relative Frobenius norm error (RFNE), we use
the Pearson correlation coefficient to measure the linear relationship between the forecasted frames and the target frames. The range of this coefficient is $[-1,1]$,  with zero implying no correlation. Correlations of $-1$ or $+1$ imply an exact linear relationship. Positive correlations imply that as $x$ increases, so does $y$. Negative correlations imply that as $x$ increases, $y$ decreases. 
In addition,  we use peak signal-to-noise ratio (PSNR) to evaluate the quality of signal representation against corrupting noise, and structural similarity index measure (SSIM) \citep{wang2004image} to assess perceptual results. The presented results are computed by averaging over batch size and number of sample generations.

\subsection{Standard Deviations for the Presented Results}

Table \ref{tab:simpleflow_stdev}-\ref{tab:ns_stdev} provide the standard deviation of the results presented in the main paper.

\begin{table}[!h]
	\caption{Standard deviation results for  the fluid flow past a cylinder task using different choices of probability paths for flow matching. Results are averaged over 5 generations.}
	\label{tab:simpleflow_stdev}
	\centering
	\scalebox{0.9}{
		\begin{tabular}{l c  c ccccccc}
			\toprule
			Model &  Test MSE ($\downarrow$) & Test RFNE ($\downarrow$) & PSNR ($\uparrow$) & SSIM ($\uparrow$) \\	
			\midrule 
			RIVER & 1.04e-03  & 2.70e-02 & 1.46 &  4.92e-03  \\
			VE-diffusion  & 5.22e-02 & 3.34e-01 & 6.02e-01 & 3.14e-02 \\
			VP-diffusion & 7.09e-03 & 3.96e-02 & 4.26e-01 & 3.30e-03   \\
   			SI $(b_t = t^2)$  & 6.35e-05 & 3.53e-02 & 5.39e-02 &  4.03e-04  \\
            SI $(b_t = t)$  & 1.48e-04 & 8.70e-02 & 2.97e-02 & 5.30e-04 \\
      \midrule 
      Ours ($\sigma  = 0.01$, $\sigma_{sam}$ = 0, RK4) & 4.26e-06  & 5.29e-03 & 3.61e-02 & 3.63e-05 \\
			\bottomrule
	\end{tabular}}
\end{table}

\begin{table}[!h]
	\caption{Standard deviation for the results of ablation study for the fluid flow past a cylinder task. Results are averaged over 5 generations.}
	\label{tab:simpleflow_ablation_stdev}
	\centering
	\scalebox{0.9}{
		\begin{tabular}{l c c | c  c ccccccc}
			\toprule
			$
   \sigma$ & sampler & $N$ & Test MSE ($\downarrow$) & Test RFNE ($\downarrow$) &  PSNR ($\uparrow$) & SSIM ($\uparrow$) \\		       \midrule 
    0.0 & Euler & 5 &  1.00e-05  & 1.25e-02  & 4.37e-02 & 9.76e-05  \\
			0.01  & Euler & 5 & 6.71e-06 & 4.14e-03 & 5.17e-02 & 4.81e-05 \\
			0.1  & Euler & 5 & 4.88e-05  & 4.11e-02 & 5.39e-02 & 3.91e-04 \\
			\midrule 
			 0.0 & Euler & 10 & 1.24e-05 & 1.29e-02 & 4.28e-02  & 8.27e-05 \\
			0.01  & Euler & 10 & 8.10e-06  & 4.37e-03 & 6.42e-02 & 6.21e-05 \\
			0.1  & Euler & 10 & 4.32e-05  & 4.02e-02 & 5.13e-02 & 3.06e-04 \\
   			\midrule 
			0.0 & RK4 & 10 & 3.45e-06 & 7.56e-03  & 2.57e-02 & 2.46e-05  \\
			0.01  & RK4 & 10 & 4.26e-06  & 5.29e-03 & 3.61e-02  &  3.63e-05  \\
			0.1 & RK4 & 10 & 8.62e-06   & 5.88e-02 & 2.69e-02 & 2.61e-04 \\
   			\midrule 
                0.0 & Euler & 20 & 8.14e-06  & 1.31e-02  & 3.11e-02 & 5.02e-05 \\
			0.01  & Euler &  20  & 4.25e-06 & 4.51e-03 & 3.66e-02 & 2.77e-05  \\
			0.1  & Euler & 20 & 2.12e-05  & 3.96e-02 &  2.21e-02 & 1.03e-04 \\
   			\midrule 
			0.0 & RK4 & 20 &  1.54e-06  & 7.49e-03 & 1.84e-02 & 1.26e-05 \\
			0.01  & RK4 & 20 &  3.99e-06   & 1.29e-02 & 1.74e-02 & 2.95e-05\\
			0.1  & RK4 & 20 &  2.84e-05 & 5.85e-02 & 8.10e-03 & 8.33e-05 \\
   			\bottomrule
	\end{tabular}}
\end{table}

\begin{table}[!h]
	\caption{Standard deviation results for the shallow-water equation task using different choices of probability paths for flow matching. Results are averaged over 5 generations.}
	\label{tab:swe_stdev}
	\centering
	\scalebox{0.9}{
		\begin{tabular}{l c  c ccccccc}
			\toprule
			Model &  Test MSE ($\downarrow$) &  Test RFNE ($\downarrow$) &  PSNR ($\uparrow$) & SSIM ($\uparrow$) \\	
			\midrule 
			RIVER & 2.28e-05 & 8.67e-02 & 8.39e-02 & 1.23e-03 \\
			VE-diffusion  & 7.29e-04 & 3.29e-01 & 2.17e-01 & 1.68e-02  \\
			VP-diffusion & 2.55e-04 & 1.10e-01 & 3.62e-01 & 1.56e-02 \\
   			SI $(b_t = t^2)$  & 2.97e-06  & 1.03e-01 & 1.86e-02 & 2.03e-04 \\
            	SI $(b_t = t)$  & 1.37e-06 & 7.18e-02 & 4.11e-03 & 2.12e-04 \\
			\midrule 
Ours ($\sigma$ = 0.1, $\sigma_{sam}$ = 0, RK4)
   & 1.55e-06 & 7.06e-02 & 5.98e-03 & 2.11e-04 \\
			\bottomrule
	\end{tabular}}
\end{table}

\begin{table}[!h]
	\caption{Standard deviation results for  the diffusion-reaction equation task using different choices of probability paths for flow matching. Results are averaged over 5 generations.}
	\label{tab:reacdiff_stdev}
	\centering
	\scalebox{0.9}{
		\begin{tabular}{l c  c ccccccc}
			\toprule
			Model &  Test MSE ($\downarrow$) &  Test RFNE ($\downarrow$) &  PSNR ($\uparrow$) & SSIM ($\uparrow$) \\	
			\midrule 
			RIVER &  5.56e-04 & 2.54e-01 & 7.88e-01 & 1.37e-02\\
			VE-diffusion  & 1.51e-02 & 1.11 & 5.89e-01 &  1.22e-02 \\
			VP-diffusion & 1.55e-03 & 6.15e-01 & 1.06 & 7.08e-03  \\
   			SI $(b_t = t^2)$ & 6.44e-04  & 1.15 & 8.60e-02 & 6.98e-04     \\
                SI $(b_t = t)$  & 8.15e-07  & 5.01e-02 & 5.67e-03 & 8.25e-05  \\
			\midrule
			Ours ($\sigma$ = 0, $\sigma_{sam} = $ 0,  RK4) & 8.47e-07 & 4.77e-02 &  6.22e-03 & 9.53e-05 \\ 
			\bottomrule
	\end{tabular}}
\end{table}

\begin{table}[!h]
	\caption{Standard deviation results for the Navier-Stokes equation task using different choices of probability paths for flow matching. Results are averaged over 5 generations. }
	\label{tab:ns_stdev}
	\centering
	\scalebox{0.9}{
		\begin{tabular}{l c  c ccccccc}
			\toprule
			Model &  Test MSE ($\downarrow$) &  Test RFNE ($\downarrow$) &  PSNR ($\uparrow$) & SSIM ($\uparrow$) \\	
			\midrule 
			RIVER & 7.52e-04 & 8.31e-01 & 9.34e-02 & 3.80e-03 \\
			VE-diffusion  & 6.53e-04 & 1.21 & 1.19e-01 & 1.61e-03 \\
			VP-diffusion &  3.88e-03 & 1.62 & 1.64e-01 & 5.74e-03 \\
   			SI  $(b_t = t^2)$   & 5.11e-06 & 3.27e-02 & 2.37e-02 & 3.34e-04  \\
             SI $(b_t = t)$  & 4.87e-07 & 1.09e-02 & 1.98e-03 & 6.22e-05 \\
			\midrule 
Ours ($\sigma$ = 0.1, $\sigma_{sam}$ = 0, RK4)
   & 4.81e-07 & 1.06e-02 & 2.30e-03 &  6.01e-05 \\
			\bottomrule
	\end{tabular}}
\end{table}

%% file: main.bbl
\begin{thebibliography}{75}
\providecommand{\natexlab}[1]{#1}
\providecommand{\url}[1]{\texttt{#1}}
\expandafter\ifx\csname urlstyle\endcsname\relax
  \providecommand{\doi}[1]{doi: #1}\else
  \providecommand{\doi}{doi: \begingroup \urlstyle{rm}\Url}\fi

\bibitem[Albergo et~al.(2023)Albergo, Boffi, and
  Vanden-Eijnden]{albergo2023stochastic}
Michael~S Albergo, Nicholas~M Boffi, and Eric Vanden-Eijnden.
\newblock Stochastic interpolants: A unifying framework for flows and
  diffusions.
\newblock \emph{arXiv preprint arXiv:2303.08797}, 2023.

\bibitem[Alcaraz \& Strodthoff(2023)Alcaraz and
  Strodthoff]{alcaraz2023diffusionbasedtimeseriesimputation}
Juan Miguel~Lopez Alcaraz and Nils Strodthoff.
\newblock Diffusion-based time series imputation and forecasting with
  structured state space models, 2023.

\bibitem[Anonymous(2025)]{anonymous2025diffusionmodelsand}
Anonymous.
\newblock Diffusion models and {G}aussian flow matching: Two sides of the same
  coin.
\newblock In \emph{ICLR Blogposts 2025}, 2025.

\bibitem[Azencot et~al.(2020)Azencot, Erichson, Lin, and
  Mahoney]{azencot2020forecasting}
Omri Azencot, N~Benjamin Erichson, Vanessa Lin, and Michael~W Mahoney.
\newblock Forecasting sequential data using consistent {K}oopman autoencoders.
\newblock In \emph{International Conference on Machine Learning}, pp.\
  475--485. PMLR, 2020.

\bibitem[Battaglia et~al.(2016)Battaglia, Pascanu, Lai, Jimenez~Rezende,
  et~al.]{battaglia2016interaction}
Peter Battaglia, Razvan Pascanu, Matthew Lai, Danilo Jimenez~Rezende, et~al.
\newblock Interaction networks for learning about objects, relations and
  physics.
\newblock \emph{Advances in Neural Information Processing Systems}, 29, 2016.

\bibitem[Bengio et~al.(2013)Bengio, Boulanger-Lewandowski, and
  Pascanu]{bengio2013advances}
Yoshua Bengio, Nicolas Boulanger-Lewandowski, and Razvan Pascanu.
\newblock Advances in optimizing recurrent networks.
\newblock In \emph{2013 IEEE International Conference on Acoustics, Speech and
  Signal Processing}, pp.\  8624--8628. IEEE, 2013.

\bibitem[Biloš et~al.(2023)Biloš, Rasul, Schneider, Nevmyvaka, and
  Günnemann]{bilos2023modelingtemporaldatacontinuous}
Marin Biloš, Kashif Rasul, Anderson Schneider, Yuriy Nevmyvaka, and Stephan
  Günnemann.
\newblock Modeling temporal data as continuous functions with stochastic
  process diffusion, 2023.

\bibitem[Blattmann et~al.(2023)Blattmann, Dockhorn, Kulal, Mendelevitch,
  Kilian, Lorenz, Levi, English, Voleti, Letts, et~al.]{blattmann2023stable}
Andreas Blattmann, Tim Dockhorn, Sumith Kulal, Daniel Mendelevitch, Maciej
  Kilian, Dominik Lorenz, Yam Levi, Zion English, Vikram Voleti, Adam Letts,
  et~al.
\newblock Stable video diffusion: Scaling latent video diffusion models to
  large datasets.
\newblock \emph{arXiv preprint arXiv:2311.15127}, 2023.

\bibitem[Box et~al.(2015)Box, Jenkins, Reinsel, and Ljung]{box2015time}
George~EP Box, Gwilym~M Jenkins, Gregory~C Reinsel, and Greta~M Ljung.
\newblock \emph{Time Series Analysis: Forecasting and Control}.
\newblock John Wiley \& Sons, 2015.

\bibitem[Bunne et~al.(2023)Bunne, Hsieh, Cuturi, and
  Krause]{bunne2023Schrodinger}
Charlotte Bunne, Ya-Ping Hsieh, Marco Cuturi, and Andreas Krause.
\newblock The {S}chr{\"o}dinger bridge between {G}aussian measures has a closed
  form.
\newblock In \emph{International Conference on Artificial Intelligence and
  Statistics}, pp.\  5802--5833. PMLR, 2023.

\bibitem[Cachay et~al.(2023)Cachay, Zhao, Joren, and
  Yu]{cachay2023dyffusiondynamicsinformeddiffusionmodel}
Salva~Rühling Cachay, Bo~Zhao, Hailey Joren, and Rose Yu.
\newblock Dyffusion: A dynamics-informed diffusion model for spatiotemporal
  forecasting, 2023.

\bibitem[Chen et~al.(2018)Chen, Rubanova, Bettencourt, and
  Duvenaud]{chen2018neural}
Ricky~TQ Chen, Yulia Rubanova, Jesse Bettencourt, and David~K Duvenaud.
\newblock Neural ordinary differential equations.
\newblock \emph{Advances in Neural Information Processing Systems}, 31, 2018.

\bibitem[Chen et~al.(2021{\natexlab{a}})Chen, Liu, and
  Theodorou]{chen2021likelihood}
Tianrong Chen, Guan-Horng Liu, and Evangelos~A Theodorou.
\newblock Likelihood training of {S}chr{\"o}dinger bridge using
  forward-backward sdes theory.
\newblock \emph{arXiv preprint arXiv:2110.11291}, 2021{\natexlab{a}}.

\bibitem[Chen et~al.(2024)Chen, Goldstein, Hua, Albergo, Boffi, and
  Vanden-Eijnden]{chen2024probabilistic}
Yifan Chen, Mark Goldstein, Mengjian Hua, Michael~S Albergo, Nicholas~M Boffi,
  and Eric Vanden-Eijnden.
\newblock Probabilistic forecasting with stochastic interpolants and
  {F}\"{o}llmer processes.
\newblock \emph{arXiv preprint arXiv:2403.13724}, 2024.

\bibitem[Chen et~al.(2021{\natexlab{b}})Chen, Georgiou, and
  Pavon]{chen2021stochastic}
Yongxin Chen, Tryphon~T Georgiou, and Michele Pavon.
\newblock Stochastic control liaisons: Richard {S}inkhorn meets {G}aspard
  {M}onge on a {S}chrodinger bridge.
\newblock \emph{SIAM Review}, 63\penalty0 (2):\penalty0 249--313,
  2021{\natexlab{b}}.

\bibitem[Chewi et~al.(2024)Chewi, Niles-Weed, and
  Rigollet]{chewi2024statistical}
Sinho Chewi, Jonathan Niles-Weed, and Philippe Rigollet.
\newblock Statistical optimal transport.
\newblock \emph{arXiv preprint arXiv:2407.18163}, 2024.

\bibitem[Colonius \& Taira(2008)Colonius and Taira]{colonius2008fast}
Tim Colonius and Kunihiko Taira.
\newblock A fast immersed boundary method using a nullspace approach and
  multi-domain far-field boundary conditions.
\newblock \emph{Computer Methods in Applied Mechanics and Engineering},
  197\penalty0 (25-28):\penalty0 2131--2146, 2008.

\bibitem[Dao et~al.(2023)Dao, Phung, Nguyen, and Tran]{dao2023flow}
Quan Dao, Hao Phung, Binh Nguyen, and Anh Tran.
\newblock Flow matching in latent space.
\newblock \emph{arXiv preprint arXiv:2307.08698}, 2023.

\bibitem[Davtyan et~al.(2023)Davtyan, Sameni, and Favaro]{davtyan2023efficient}
Aram Davtyan, Sepehr Sameni, and Paolo Favaro.
\newblock Efficient video prediction via sparsely conditioned flow matching.
\newblock In \emph{Proceedings of the IEEE/CVF International Conference on
  Computer Vision}, pp.\  23263--23274, 2023.

\bibitem[De~Bortoli et~al.(2021)De~Bortoli, Thornton, Heng, and
  Doucet]{de2021diffusion}
Valentin De~Bortoli, James Thornton, Jeremy Heng, and Arnaud Doucet.
\newblock Diffusion {S}chr{\"o}dinger bridge with applications to score-based
  generative modeling.
\newblock \emph{Advances in Neural Information Processing Systems},
  34:\penalty0 17695--17709, 2021.

\bibitem[Dockhorn et~al.(2021)Dockhorn, Vahdat, and Kreis]{dockhorn2021score}
Tim Dockhorn, Arash Vahdat, and Karsten Kreis.
\newblock Score-based generative modeling with critically-damped {L}angevin
  diffusion.
\newblock \emph{arXiv preprint arXiv:2112.07068}, 2021.

\bibitem[Dueben \& Bauer(2018)Dueben and Bauer]{dueben2018challenges}
Peter~D Dueben and Peter Bauer.
\newblock Challenges and design choices for global weather and climate models
  based on machine learning.
\newblock \emph{Geoscientific Model Development}, 11\penalty0 (10):\penalty0
  3999--4009, 2018.

\bibitem[Erichson et~al.(2019)Erichson, Mathelin, Kutz, and
  Brunton]{erichson2019randomized}
N~Benjamin Erichson, Lionel Mathelin, J~Nathan Kutz, and Steven~L Brunton.
\newblock Randomized dynamic mode decomposition.
\newblock \emph{SIAM Journal on Applied Dynamical Systems}, 18\penalty0
  (4):\penalty0 1867--1891, 2019.

\bibitem[Erichson et~al.(2020)Erichson, Mathelin, Yao, Brunton, Mahoney, and
  Kutz]{erichson2020shallow}
N~Benjamin Erichson, Lionel Mathelin, Zhewei Yao, Steven~L Brunton, Michael~W
  Mahoney, and J~Nathan Kutz.
\newblock Shallow neural networks for fluid flow reconstruction with limited
  sensors.
\newblock \emph{Proceedings of the Royal Society A}, 476\penalty0
  (2238):\penalty0 20200097, 2020.

\bibitem[Esser et~al.(2024)Esser, Kulal, Blattmann, Entezari, M{\"u}ller,
  Saini, Levi, Lorenz, Sauer, Boesel, et~al.]{esser2024scaling}
Patrick Esser, Sumith Kulal, Andreas Blattmann, Rahim Entezari, Jonas
  M{\"u}ller, Harry Saini, Yam Levi, Dominik Lorenz, Axel Sauer, Frederic
  Boesel, et~al.
\newblock Scaling rectified flow transformers for high-resolution image
  synthesis.
\newblock In \emph{Forty-first International Conference on Machine Learning},
  2024.

\bibitem[Evans(2022)]{evans2022partial}
Lawrence~C Evans.
\newblock \emph{Partial Differential Equations}, volume~19.
\newblock American Mathematical Society, 2022.

\bibitem[Gasbarra et~al.(2007)Gasbarra, Sottinen, and
  Valkeila]{gasbarra2007gaussian}
Dario Gasbarra, Tommi Sottinen, and Esko Valkeila.
\newblock Gaussian bridges.
\newblock In \emph{Stochastic Analysis and Applications: The Abel Symposium
  2005}, pp.\  361--382. Springer, 2007.

\bibitem[Gat et~al.(2024)Gat, Remez, Shaul, Kreuk, Chen, Synnaeve, Adi, and
  Lipman]{gat2024discrete}
Itai Gat, Tal Remez, Neta Shaul, Felix Kreuk, Ricky~TQ Chen, Gabriel Synnaeve,
  Yossi Adi, and Yaron Lipman.
\newblock Discrete flow matching.
\newblock \emph{arXiv preprint arXiv:2407.15595}, 2024.

\bibitem[Gottwald \& Reich(2024)Gottwald and Reich]{gottwald2024localized}
Georg~A Gottwald and Sebastian Reich.
\newblock Localized {S}chr{\"o}dinger bridge sampler.
\newblock \emph{arXiv preprint arXiv:2409.07968}, 2024.

\bibitem[Gupta et~al.(2023)Gupta, Yu, Sohn, Gu, Hahn, Fei-Fei, Essa, Jiang, and
  Lezama]{gupta2023photorealistic}
Agrim Gupta, Lijun Yu, Kihyuk Sohn, Xiuye Gu, Meera Hahn, Li~Fei-Fei, Irfan
  Essa, Lu~Jiang, and Jos{\'e} Lezama.
\newblock Photorealistic video generation with diffusion models.
\newblock \emph{arXiv preprint arXiv:2312.06662}, 2023.

\bibitem[Ho et~al.(2020)Ho, Jain, and Abbeel]{ho2020denoising}
Jonathan Ho, Ajay Jain, and Pieter Abbeel.
\newblock Denoising diffusion probabilistic models.
\newblock \emph{Advances in Neural Information Processing Systems},
  33:\penalty0 6840--6851, 2020.

\bibitem[Ho et~al.(2022)Ho, Salimans, Gritsenko, Chan, Norouzi, and
  Fleet]{ho2022video}
Jonathan Ho, Tim Salimans, Alexey Gritsenko, William Chan, Mohammad Norouzi,
  and David~J Fleet.
\newblock Video diffusion models.
\newblock \emph{Advances in Neural Information Processing Systems},
  35:\penalty0 8633--8646, 2022.

\bibitem[Hu et~al.(2024)Hu, Wang, Wu, Zhang, Li, Wang, and Chen]{hu2024fm}
Yang Hu, Xiao Wang, Lirong Wu, Huatian Zhang, Stan~Z Li, Sheng Wang, and
  Tianlong Chen.
\newblock F{M}-{T}{S}: Flow matching for time series generation.
\newblock \emph{arXiv preprint arXiv:2411.07506}, 2024.

\bibitem[Karras et~al.(2022)Karras, Aittala, Aila, and Laine]{karras2022}
Tero Karras, Miika Aittala, Timo Aila, and Samuli Laine.
\newblock Elucidating the design space of diffusion-based generative models.
\newblock \emph{Advances in Neural Information Processing Systems},
  35:\penalty0 26565--26577, 2022.

\bibitem[Kingma \& Gao(2023)Kingma and Gao]{kingma2023understanding}
Diederik~P Kingma and Ruiqi Gao.
\newblock Understanding the diffusion objective as a weighted integral of
  {E}{L}{B}{O}s.
\newblock \emph{arXiv preprint arXiv:2303.00848}, 2023.

\bibitem[Kohl et~al.(2024)Kohl, Chen, and Thuerey]{kohl2024benchmarking}
Georg Kohl, Liwei Chen, and Nils Thuerey.
\newblock Benchmarking autoregressive conditional diffusion models for
  turbulent flow simulation.
\newblock In \emph{ICML 2024 AI for Science Workshop}, 2024.

\bibitem[Krishnapriyan et~al.(2023)Krishnapriyan, Queiruga, Erichson, and
  Mahoney]{krishnapriyan2023learning}
Aditi~S Krishnapriyan, Alejandro~F Queiruga, N~Benjamin Erichson, and Michael~W
  Mahoney.
\newblock Learning continuous models for continuous physics.
\newblock \emph{Communications Physics}, 6\penalty0 (1):\penalty0 319, 2023.

\bibitem[L{\'e}onard(2013)]{leonard2013survey}
Christian L{\'e}onard.
\newblock A survey of the {S}chr{\"o}dinger problem and some of its connections
  with optimal transport.
\newblock \emph{arXiv preprint arXiv:1308.0215}, 2013.

\bibitem[Lipman et~al.(2022)Lipman, Chen, Ben-Hamu, Nickel, and
  Le]{lipman2022flow}
Yaron Lipman, Ricky~TQ Chen, Heli Ben-Hamu, Maximilian Nickel, and Matt Le.
\newblock Flow matching for generative modeling.
\newblock \emph{arXiv preprint arXiv:2210.02747}, 2022.

\bibitem[Lipman et~al.(2024)Lipman, Havasi, Holderrieth, Shaul, Le, Karrer,
  Chen, Lopez-Paz, Ben-Hamu, and Gat]{lipman2024flow}
Yaron Lipman, Marton Havasi, Peter Holderrieth, Neta Shaul, Matt Le, Brian
  Karrer, Ricky~TQ Chen, David Lopez-Paz, Heli Ben-Hamu, and Itai Gat.
\newblock Flow matching guide and code.
\newblock \emph{arXiv preprint arXiv:2412.06264}, 2024.

\bibitem[Liu et~al.(2023)Liu, Lipman, Nickel, Karrer, Theodorou, and
  Chen]{liu2023generalized}
Guan-Horng Liu, Yaron Lipman, Maximilian Nickel, Brian Karrer, Evangelos~A
  Theodorou, and Ricky~TQ Chen.
\newblock Generalized {S}chr{\"o}dinger bridge matching.
\newblock \emph{arXiv preprint arXiv:2310.02233}, 2023.

\bibitem[Liu et~al.(2022)Liu, Gong, and Liu]{liu2022flow}
Xingchao Liu, Chengyue Gong, and Qiang Liu.
\newblock Flow straight and fast: Learning to generate and transfer data with
  rectified flow.
\newblock \emph{arXiv preprint arXiv:2209.03003}, 2022.

\bibitem[Lv et~al.(2024)Lv, Yuan, and Ni]{Lv2024trajectorylearninggraphs}
Kai Lv, Liang Yuan, and Xiaoyu Ni.
\newblock Learning autoencoder diffusion models of pedestrian group
  relationships for multimodal trajectory prediction.
\newblock \emph{IEEE Transactions on Instrumentation and Measurement},
  73:\penalty0 1--12, 2024.
\newblock \doi{10.1109/TIM.2024.3375973}.

\bibitem[Lyu et~al.(2024)Lyu, Nakata, Ren, Mahoney, Pitarka, Nakata, and
  Erichson]{lyu2024wavecastnet}
Dongwei Lyu, Rie Nakata, Pu~Ren, Michael~W Mahoney, Arben Pitarka, Nori Nakata,
  and N~Benjamin Erichson.
\newblock Wave{C}ast{N}et: An {A}{I}-enabled wavefield forecasting framework
  for earthquake early warning.
\newblock \emph{arXiv preprint arXiv:2405.20516}, 2024.

\bibitem[Meijer \& Chen(2024)Meijer and Chen]{meijer2024rise}
Caspar Meijer and Lydia~Y Chen.
\newblock The rise of diffusion models in time-series forecasting.
\newblock \emph{arXiv preprint arXiv:2401.03006}, 2024.

\bibitem[Pathak et~al.(2024)Pathak, Cohen, Garg, Harrington, Brenowitz, Durran,
  Mardani, Vahdat, Xu, Kashinath, et~al.]{pathak2024kilometer}
Jaideep Pathak, Yair Cohen, Piyush Garg, Peter Harrington, Noah Brenowitz, Dale
  Durran, Morteza Mardani, Arash Vahdat, Shaoming Xu, Karthik Kashinath, et~al.
\newblock Kilometer-scale convection allowing model emulation using generative
  diffusion modeling.
\newblock \emph{arXiv preprint arXiv:2408.10958}, 2024.

\bibitem[Peluchetti(2023)]{peluchetti2023diffusion}
Stefano Peluchetti.
\newblock Diffusion bridge mixture transports, {S}chr{\"o}dinger bridge
  problems and generative modeling.
\newblock \emph{Journal of Machine Learning Research}, 24\penalty0
  (374):\penalty0 1--51, 2023.

\bibitem[Pooladian \& Niles-Weed(2024)Pooladian and
  Niles-Weed]{pooladian2024plug}
Aram-Alexandre Pooladian and Jonathan Niles-Weed.
\newblock Plug-in estimation of {S}chr{\"o}dinger bridges.
\newblock \emph{arXiv preprint arXiv:2408.11686}, 2024.

\bibitem[Pooladian et~al.(2023)Pooladian, Ben-Hamu, Domingo-Enrich, Amos,
  Lipman, and Chen]{pooladian2023multisample}
Aram-Alexandre Pooladian, Heli Ben-Hamu, Carles Domingo-Enrich, Brandon Amos,
  Yaron Lipman, and Ricky~TQ Chen.
\newblock Multisample flow matching: Straightening flows with minibatch
  couplings.
\newblock \emph{arXiv preprint arXiv:2304.14772}, 2023.

\bibitem[Qu et~al.(2023)Qu, Gong, Chen, Zhang, Zheng, and
  Yin]{qu2023urbanflowsgraph}
Hao Qu, Yongshun Gong, Meng Chen, Junbo Zhang, Yu~Zheng, and Yilong Yin.
\newblock Forecasting fine-grained urban flows via spatio-temporal contrastive
  self-supervision.
\newblock \emph{IEEE Transactions on Knowledge and Data Engineering},
  35\penalty0 (8):\penalty0 8008--8023, 2023.
\newblock \doi{10.1109/TKDE.2022.3200734}.

\bibitem[Queiruga et~al.(2021)Queiruga, Erichson, Hodgkinson, and
  Mahoney]{queiruga2021stateful}
Alejandro Queiruga, N~Benjamin Erichson, Liam Hodgkinson, and Michael~W
  Mahoney.
\newblock Stateful {ODE}-nets using basis function expansions.
\newblock \emph{Advances in Neural Information Processing Systems},
  34:\penalty0 21770--21781, 2021.

\bibitem[Rasul et~al.(2021)Rasul, Seward, Schuster, and
  Vollgraf]{rasul2021autoregressive}
Kashif Rasul, Calvin Seward, Ingmar Schuster, and Roland Vollgraf.
\newblock Autoregressive denoising diffusion models for multivariate
  probabilistic time series forecasting.
\newblock In \emph{International Conference on Machine Learning}, pp.\
  8857--8868. PMLR, 2021.

\bibitem[Ren et~al.(2024)Ren, Nakata, Lacour, Naiman, Nakata, Song, Bi, Malik,
  Morozov, Azencot, et~al.]{ren2024learning}
Pu~Ren, Rie Nakata, Maxime Lacour, Ilan Naiman, Nori Nakata, Jialin Song,
  Zhengfa Bi, Osman~Asif Malik, Dmitriy Morozov, Omri Azencot, et~al.
\newblock Learning physics for unveiling hidden earthquake ground motions via
  conditional generative modeling.
\newblock \emph{arXiv preprint arXiv:2407.15089}, 2024.

\bibitem[S{\"a}rkk{\"a} \& Solin(2019)S{\"a}rkk{\"a} and
  Solin]{sarkka2019applied}
Simo S{\"a}rkk{\"a} and Arno Solin.
\newblock \emph{Applied Stochastic Differential Equations}, volume~10.
\newblock Cambridge University Press, 2019.

\bibitem[Schultz et~al.(2021)Schultz, Betancourt, Gong, Kleinert, Langguth,
  Leufen, Mozaffari, and Stadtler]{schultz2021can}
Martin~G Schultz, Clara Betancourt, Bing Gong, Felix Kleinert, Michael
  Langguth, Lukas~Hubert Leufen, Amirpasha Mozaffari, and Scarlet Stadtler.
\newblock Can deep learning beat numerical weather prediction?
\newblock \emph{Philosophical Transactions of the Royal Society A},
  379\penalty0 (2194):\penalty0 20200097, 2021.

\bibitem[Shi et~al.(2015)Shi, Chen, Wang, Yeung, Wong, and
  Woo]{shi2015convolutional}
Xingjian Shi, Zhourong Chen, Hao Wang, Dit-Yan Yeung, Wai-Kin Wong, and
  Wang-chun Woo.
\newblock Convolutional {L}{S}{T}{M} network: A machine learning approach for
  precipitation nowcasting.
\newblock \emph{Advances in Neural Information Processing Systems}, 28, 2015.

\bibitem[Shi et~al.(2024)Shi, De~Bortoli, Campbell, and
  Doucet]{shi2024diffusion}
Yuyang Shi, Valentin De~Bortoli, Andrew Campbell, and Arnaud Doucet.
\newblock Diffusion {S}chr{\"o}dinger bridge matching.
\newblock \emph{Advances in Neural Information Processing Systems}, 36, 2024.

\bibitem[Shrivastava \& Shrivastava(2024)Shrivastava and
  Shrivastava]{shrivastava2024video}
Gaurav Shrivastava and Abhinav Shrivastava.
\newblock Video prediction by modeling videos as continuous multi-dimensional
  processes.
\newblock In \emph{Proceedings of the IEEE/CVF Conference on Computer Vision
  and Pattern Recognition}, pp.\  7236--7245, 2024.

\bibitem[Singer et~al.(2023)Singer, Polyak, Hayes, Yin, An, Zhang, Hu, Yang,
  Ashual, Gafni, Parikh, Gupta, and Taigman]{singer2023makeavideo}
Uriel Singer, Adam Polyak, Thomas Hayes, Xi~Yin, Jie An, Songyang Zhang, Qiyuan
  Hu, Harry Yang, Oron Ashual, Oran Gafni, Devi Parikh, Sonal Gupta, and Yaniv
  Taigman.
\newblock Make-a-video: Text-to-video generation without text-video data.
\newblock In \emph{The Eleventh International Conference on Learning
  Representations}, 2023.

\bibitem[Song et~al.(2020{\natexlab{a}})Song, Meng, and
  Ermon]{song2020denoising}
Jiaming Song, Chenlin Meng, and Stefano Ermon.
\newblock Denoising diffusion implicit models.
\newblock \emph{arXiv preprint arXiv:2010.02502}, 2020{\natexlab{a}}.

\bibitem[Song \& Ermon(2019)Song and Ermon]{song2019generative}
Yang Song and Stefano Ermon.
\newblock Generative modeling by estimating gradients of the data distribution.
\newblock \emph{Advances in Neural Information Processing Systems}, 32, 2019.

\bibitem[Song et~al.(2020{\natexlab{b}})Song, Sohl-Dickstein, Kingma, Kumar,
  Ermon, and Poole]{song2020score}
Yang Song, Jascha Sohl-Dickstein, Diederik~P Kingma, Abhishek Kumar, Stefano
  Ermon, and Ben Poole.
\newblock Score-based generative modeling through stochastic differential
  equations.
\newblock \emph{arXiv preprint arXiv:2011.13456}, 2020{\natexlab{b}}.

\bibitem[Song et~al.(2021)Song, Durkan, Murray, and Ermon]{song2021maximum}
Yang Song, Conor Durkan, Iain Murray, and Stefano Ermon.
\newblock Maximum likelihood training of score-based diffusion models.
\newblock \emph{Advances in Neural Information Processing Systems},
  34:\penalty0 1415--1428, 2021.

\bibitem[Takamoto et~al.(2022)Takamoto, Praditia, Leiteritz, MacKinlay,
  Alesiani, Pfl{\"u}ger, and Niepert]{takamoto2022pdebench}
Makoto Takamoto, Timothy Praditia, Raphael Leiteritz, Daniel MacKinlay,
  Francesco Alesiani, Dirk Pfl{\"u}ger, and Mathias Niepert.
\newblock P{D}{E}{B}ench: An extensive benchmark for scientific machine
  learning.
\newblock \emph{Advances in Neural Information Processing Systems},
  35:\penalty0 1596--1611, 2022.

\bibitem[Tamir et~al.(2024)Tamir, Laabid, Heinonen, Garg, and
  Solin]{tamir2024conditional}
Ella Tamir, Najwa Laabid, Markus Heinonen, Vikas Garg, and Arno Solin.
\newblock Conditional flow matching for time series modelling.
\newblock In \emph{ICML 2024 Workshop on Structured Probabilistic Inference \&
  Generative Modeling}, 2024.

\bibitem[Tashiro et~al.(2021)Tashiro, Song, Song, and
  Ermon]{tashiro2021csdiconditionalscorebaseddiffusion}
Yusuke Tashiro, Jiaming Song, Yang Song, and Stefano Ermon.
\newblock {CSDI}: Conditional score-based diffusion models for probabilistic
  time series imputation.
\newblock 2021.

\bibitem[Tomczak(2021)]{tomczak2021latent}
Jakub~M Tomczak.
\newblock Latent variable models.
\newblock In \emph{Deep Generative Modeling}, pp.\  57--127. Springer, 2021.

\bibitem[Tong et~al.(2023{\natexlab{a}})Tong, Malkin, Fatras, Atanackovic,
  Zhang, Huguet, Wolf, and Bengio]{tong2023simulation}
Alexander Tong, Nikolay Malkin, Kilian Fatras, Lazar Atanackovic, Yanlei Zhang,
  Guillaume Huguet, Guy Wolf, and Yoshua Bengio.
\newblock Simulation-free {S}chr{\"o}dinger bridges via score and flow
  matching.
\newblock \emph{arXiv preprint arXiv:2307.03672}, 2023{\natexlab{a}}.

\bibitem[Tong et~al.(2023{\natexlab{b}})Tong, Malkin, Huguet, Zhang,
  Rector-Brooks, Fatras, Wolf, and Bengio]{tong2023improving}
Alexander Tong, Nikolay Malkin, Guillaume Huguet, Yanlei Zhang, Jarrid
  Rector-Brooks, Kilian Fatras, Guy Wolf, and Yoshua Bengio.
\newblock Improving and generalizing flow-based generative models with
  minibatch optimal transport.
\newblock \emph{arXiv preprint arXiv:2302.00482}, 2023{\natexlab{b}}.

\bibitem[Vahdat et~al.(2021)Vahdat, Kreis, and Kautz]{vahdat2021score}
Arash Vahdat, Karsten Kreis, and Jan Kautz.
\newblock Score-based generative modeling in latent space.
\newblock \emph{Advances in Neural Information Processing Systems},
  34:\penalty0 11287--11302, 2021.

\bibitem[Wang et~al.(2021)Wang, Jiao, Xu, Wang, and Yang]{wang2021deep}
Gefei Wang, Yuling Jiao, Qian Xu, Yang Wang, and Can Yang.
\newblock Deep generative learning via {S}chr{\"o}dinger bridge.
\newblock In \emph{International Conference on Machine Learning}, pp.\
  10794--10804. PMLR, 2021.

\bibitem[Wang et~al.(2004)Wang, Bovik, Sheikh, and Simoncelli]{wang2004image}
Zhou Wang, Alan~C Bovik, Hamid~R Sheikh, and Eero~P Simoncelli.
\newblock Image quality assessment: from error visibility to structural
  similarity.
\newblock \emph{IEEE Transactions on Image Processing}, 13\penalty0
  (4):\penalty0 600--612, 2004.

\bibitem[Wen et~al.(2024)Wen, Lin, Xia, Wan, Wen, Zimmermann, and
  Liang]{wen2024diffstgprobabilisticspatiotemporalgraph}
Haomin Wen, Youfang Lin, Yutong Xia, Huaiyu Wan, Qingsong Wen, Roger
  Zimmermann, and Yuxuan Liang.
\newblock Diff{STG}: Probabilistic spatio-temporal graph forecasting with
  denoising diffusion models, 2024.

\bibitem[Yan et~al.(2021)Yan, Zhang, Zhou, Zhan, and
  Xia]{yan2021scoregradmultivariateprobabilistictime}
Tijin Yan, Hongwei Zhang, Tong Zhou, Yufeng Zhan, and Yuanqing Xia.
\newblock Score{G}rad: Multivariate probabilistic time series forecasting with
  continuous energy-based generative models, 2021.

\bibitem[Zhang \& Chen(2022)Zhang and Chen]{zhang2022fast}
Qinsheng Zhang and Yongxin Chen.
\newblock Fast sampling of diffusion models with exponential integrator.
\newblock \emph{arXiv preprint arXiv:2204.13902}, 2022.

\end{thebibliography}
